\documentclass{article}

\usepackage[numbers, sort&compress]{natbib}
\usepackage{float}
\usepackage{titletoc}

\usepackage[main, final]{neurips_2026}

\usepackage{multirow}
\usepackage{enumitem}
\usepackage[utf8]{inputenc} 
\usepackage[T1]{fontenc}    
\usepackage{hyperref}       
\usepackage{nicefrac}       
\usepackage{xcolor}         

\usepackage{amsmath,amsfonts}
\usepackage[ruled,vlined]{algorithm2e}
\SetKwComment{Comment}{$\triangleright$\ }{}   
\usepackage{array}
\usepackage{textcomp}
\usepackage{stfloats}
\usepackage{url}
\usepackage{verbatim}
\usepackage{graphicx}
\usepackage{natbib}
\usepackage{amsmath}
\usepackage{amssymb}
\usepackage{mathtools}
\usepackage{amsthm}
\usepackage{microtype}
\usepackage{graphicx}
\usepackage{subcaption}
\usepackage{booktabs} 
\usepackage[capitalize,noabbrev]{cleveref}
\usepackage{wrapfig}
\theoremstyle{plain}
\newtheorem{theorem}{Theorem}[section]
\newtheorem{proposition}[theorem]{Proposition}
\newtheorem{lemma}[theorem]{Lemma}

\theoremstyle{definition}

\newtheorem{assumption}[theorem]{Assumption}

\theoremstyle{remark}

\title{Federated Learning by Utility-Constrained Stochastic Aggregation for
Improving Rational Participation}

\author{%
\makebox[\textwidth][c]{%
\begin{tabular}{@{}ccccc@{}}
M Yashwanth$^{1}$ &
Arunabh Singh$^{1}$ &
Ashok Nayak$^{2}$ &
Saikiran Bulusu$^{3}$ &
Anirban Chakraborty$^{1}$ \\[1.5mm]
\multicolumn{5}{c}{\small $^{1}$Indian Institute of Science, Bangalore \quad
$^{2}$Indian Institute of Technology Bombay \quad
$^{3}$IIIT Hyderabad} \\[1mm]
\multicolumn{5}{c}{\scriptsize
\texttt{yashwanthm@iisc.ac.in} \quad
\texttt{arunabhsingh25@gmail.com} \quad
\texttt{22b0455@iitb.ac.in}} \\[-0.5mm]
\multicolumn{5}{c}{\scriptsize
\texttt{saikiran.bulusu@iiit.ac.in} \quad
\texttt{anirban@iisc.ac.in}}
\end{tabular}%
}
}
\begin{document}
\maketitle
\begin{abstract}
Federated Learning (FL) algorithms implicitly assume that clients passively comply with server-side orchestration, sharing local model updates upon the server's request. However, this overlooks an important aspect of real-world cross-silo environments: clients are often rational agents who may prioritize their own interests, such as local model performance, over that of the global model. In settings with significant statistical heterogeneity, rational clients may opt out of the federation if the perceived benefits of collaboration fail to meet their local utility thresholds. Such attrition degrades global model performance and can lead to the collapse of the federated training process. To avoid such a scenario, in this work, we introduce \textit{FedUCA (\underline{Fed}erated Learning by \underline{U}tility-\underline{C}onstrained Stochastic \underline{A}ggregation for Improving Rational Participation)}, a non-monetary framework that establishes a non-payment channel through which clients communicate their willingness to participate via their utilities. The server acts as an optimizer, identifying aggregation parameters that satisfy clients' individual rationality (IR) constraints to sustain the global model performance. FedUCA achieves this through a stochastic aggregation mechanism over Dirichlet-sampled mixture strategies, which induces an implicit utility surplus that helps satisfy participation constraints. Supported by theory and extensive experiments\footnote{The implementation is provided \href{https://github.com/vcl-iisc/feduca/tree/main}{here}.}, FedUCA’s focus on participation feasibility improves client retention and global model performance. 

\end{abstract}
\section{Introduction}
\label{sec:introduction}
Federated Learning (FL) enables multiple data holders to collaboratively train models without sharing raw data~\citep{mcmahan2017communication}. In the standard paradigm, a central server coordinates iterative rounds, aggregating local updates to learn a global model across decentralized data silos. This framework is essential in settings constrained by privacy, data sovereignty, and institutional restrictions. Most prior work focuses on improving statistical efficiency, communication, and robustness~\citep{acar2021federated,karimireddy2020scaffold, hard2019federatedlearningmobilekeyboard}, while implicitly assuming passive client participation. This assumption is reasonable in cross-device FL, but breaks down in cross-silo settings, where participation incurs non-trivial operational costs.

In cross-silo FL, participants such as hospitals, research labs, and enterprises face costs related to computation, networking, compliance, and infrastructure. As a result, participation becomes a strategic decision based on whether collaboration improves local performance. Empirical deployments reflect this behavior: clinical FL systems prioritize intermediate utility and deployment readiness~\citep{dayan2021federated}, while platforms such as Fed-BioMed emphasize governance and integration constraints~\citep{cremonesi2025fed}. Similar considerations arise in industrial settings, including energy and infrastructure costs~\citep{cao2023towards}.

This cost-driven behavior motivates viewing cross-silo FL through a mechanism-design lens. In settings where external payments are infeasible or undesirable, the primary basis for sustained collaboration is the value of the model itself. We therefore model each client as a rational agent that participates when the perceived utility of the global model exceeds its local baseline. Sustaining the federation then requires satisfying individual rationality (IR) constraints across clients.
Thus, each client evaluates a simple, performance-based proposition: they will participate if and only if their perceived utility from the global model strictly eclipses their local baseline. However, under the severe statistical heterogeneity typical of cross-silo data, a single deterministic aggregation vector often cannot simultaneously satisfy the diverse IR thresholds of all participants, making attrition likely. 

Breaking this deadlock requires the right model of client preferences. As detailed in Section~\ref{est_client_ut}, we model client utility as a concave function of its assigned aggregation weight. Concavity is the standard primitive in mechanism design~\citep{pratt1964risk}, capturing risk aversion and diminishing marginal returns: a client benefits from increased influence up to an optimal informational threshold, beyond which over-representation degrades the model's generalization. We adopt this primitive directly and verify empirically that it holds in the federated learning setting. 


This inherent concavity allows us to introduce a stochastic aggregation framework that exploits \emph{Jensen’s gap}, the difference between the utility of the expected aggregation and the expected utility across a distribution of weights. By optimizing a set of candidate aggregation strategies evaluated under a Dirichlet-sampled mixture, the server uses this gap to generate a utility surplus that acts as a safety margin. The magnitude of this induced surplus is structurally regulated via the Dirichlet concentration parameter $\delta$, sustaining participation across heterogeneous clients.

Our key contributions are 
\begin{itemize} 

\item We introduce \textbf{FedUCA}, a federated aggregation framework that establishes a non-payment channel through which clients communicate their desire to participate via their utilities. The server acts as an optimizer, tasked with identifying aggregation parameters that satisfy the individual rationality (IR) demands communicated through this channel.

\item We introduce a stochastic aggregation mechanism with Dirichlet-sampled mixture weights over multiple aggregation strategies, and show that optimizing the resulting Jensen lower-bound induces an implicit utility surplus under concave-modeled client utilities. This surplus acts as a buffer that helps satisfy participation constraints under heterogeneity, and is regulated by the Dirichlet concentration parameter $\delta$. We validate our theoretical findings through empirical evaluations on standard datasets.

\item We provide a convergence characterization for this endogenous system where participation sets and aggregation weights co-evolve. The mechanism-induced aggregation discrepancy is the mathematical footprint of rational participation management and is directly controlled by the Dirichlet concentration parameter $\delta$, connecting our Server Optimization design and Convergence.

\end{itemize}
\section{Related Work}
Federated Learning (FL) was introduced via FedAvg~\citep{mcmahan2017communication}, which assumes fully compliant clients. Subsequent work improves robustness to data heterogeneity, including FedProx~\citep{li2020federated}, SCAFFOLD~\citep{karimireddy2020scaffold}, and FedDyn~\citep{acar2021federated}.
While these methods address statistical drift, they implicitly assume passive participation. Fairness and contribution-aware approaches have also been studied. q-FFL~\citep{li2020fair} reweights client losses to improve equity, while Data Shapley~\citep{ghorbani2019data} and its extensions to FL~\citep{wang2020principled} quantify data contribution via cooperative game theory. A separate line of work models rational participation. IncFL~\citep{cho2022to} incentivizes clients based on performance gains, but tends to favor clients whose local models align with the global model, potentially degrading performance under heterogeneity.

Other approaches incorporate explicit incentives. Model-quality rewards~\citep{kong2022incentivizing} and sequential decision-based aggregation~\citep{hahn2024pursuingoverallwelfarefederated} aim to improve participation and fairness. Economic formulations include Stackelberg frameworks~\citep{sarikaya2019motivatingworkersfederatedlearning}, pricing-based methods~\citep{luo2023incentive}, and contract-theoretic approaches~\citep{kang2019incentive}. Broader surveys are provided in~\citep{zeng2021comprehensivesurveyincentivemechanism}. Auction-based methods~\citep{10591978} and hierarchical game formulations~\citep{10323190} further extend this paradigm. Recent work also considers dynamic and learning-based incentives. DaringFed~\citep{xin2025daringfed} uses Bayesian persuasion with monetary rewards, while accuracy-shaping mechanisms~\citep{karimireddy2022mechanisms} penalize free-riders through degraded model quality. Across these approaches, randomness is primarily used for client selection rather than as a mechanism for shaping client utilities. 

\begin{table}[htp]
\vspace{-10pt}
\centering
\caption{Comparison of incentive mechanisms in federated learning.}
\label{tab:comparison}
\resizebox{\columnwidth}{!}{
\begin{tabular}{lcccc}
\toprule
Method & Uses Monetary Incentives & Requires Equilibrium Modeling & Handles Rational Dropout & Direct Aggregation Optimization \\
\midrule
Contract-based~\citep{kang2019incentive} & \checkmark & \checkmark & \checkmark & Indirect \\
Pricing-based~\citep{luo2023incentive} & \checkmark & \checkmark & \checkmark & Indirect \\
DaringFed~\citep{xin2025daringfed} & \checkmark & \checkmark & \checkmark & Pricing-driven \\
IncFL~\citep{cho2022to} & $\times$ & $\times$ & \checkmark & Similarity-based \\
FedUCA (Ours) & $\times$ & $\times$ & \checkmark & \checkmark \\
\bottomrule
\end{tabular}
}
\vspace{-10pt}
\end{table}

Table~\ref{tab:comparison} summarizes key differences between FedUCA and existing approaches. FedUCA adopts a non-monetary, performance-driven formulation and does not rely on equilibrium modeling. Instead, it encodes rationality as performance-based constraints and directly optimizes aggregation weights to satisfy participation requirements. Unlike pricing and contract-based approaches that depend on external rewards, FedUCA sustains participation purely through model utility. Accordingly, we focus on baselines within the same non-monetary setting, while noting that monetary and aggregation-based approaches are complementary.

\section{Method}
FedUCA uses stochastic aggregation to exploit concavity-induced Jensen surplus in client utilities, enabling participation beyond deterministic aggregation. We first show how aggregation weights define client utility and rational participation (Sections~\ref{sec:impact_agg_weight}, \ref{est_client_ut}), then formulate the resulting constrained optimization problem (Section~\ref{frame_opt}), describe the client-server protocol (Section~\ref{algo_protocol}), and provide theoretical guarantees. We consider a standard FL setting with $N$ rational clients and a global objective $f(\theta)=\sum_{k=1}^N \beta_k f_k(\theta)$, where $f_k$ is the local objective at client $k$ and $\{\beta_k\}_{k=1}^{N}$ is a simplex. Training proceeds in rounds with server-side aggregation of client updates.

\subsection{Rational Participation}
\label{sec:impact_agg_weight}
In standard FL, the server broadcasts the aggregated global model $\boldsymbol{\theta}^t = \sum_{k=1}^{N} \alpha_k^{t-1}\boldsymbol{\theta}_k^{t-1}.$ A rational client evaluates this model against its current local model using a validation-based rationality check:
\begin{equation}
{acc}(\boldsymbol{\theta}^{t},\mathcal{D}_k^{\text{val}}) >
{acc}(\boldsymbol{\theta}_{k}^{t-1},\mathcal{D}_k^{\text{val}}),
\label{RC}
\end{equation}
where ${acc}(\cdot)$ denotes accuracy and $\mathcal{D}_k^{\text{val}}$ is the held-out validation set of client $k$ (one can use loss as well, instead of accuracy). If Eq.~\ref{RC} holds, the client accepts $\boldsymbol{\theta}^t$, trains from it, and shares an update in the next round; otherwise, it rejects the global model and does not participate. However, this binary accept/reject mechanism is a "black box" to the server. It has no way to control this check as it is agnostic to the clients' data and may lead to federation collapse. Since the server aggregates client models using simplex weights, FedUCA makes this rationality check weight-aware by sharing the aggregation weight $\alpha_k^{t-1}$ with client $k$. The client uses this information to estimate a utility function $u_k(r;\zeta_k)$ over possible aggregation weights, and reports the fitted utility parameters $\zeta_k$ together with a minimum acceptable threshold $\gamma_k$. The threshold $\gamma_k$ captures the client’s individual rationality requirement: client $k$ participates only if
\begin{equation}
u_k(r_k;\zeta_k)  > \gamma_k .
\label{eq:ir_constraint}
\end{equation}
By collecting these utilities and thresholds, the server observes heterogeneous participation requirements and can optimize aggregation weights to satisfy individual rationality constraints. 

\subsection{Estimation of Clients' Utility Functions}
\label{est_client_ut}

The assumption that client utility is concave in its aggregation weight is standard in the mechanism design and contract theory literature~\cite{pratt1964risk}. Concavity captures the economic property of diminishing marginal returns: a client benefits from increased weight up to its optimal contribution level, beyond which additional weight reduces the informational benefit drawn from the rest of the federation. This is the natural analogue of risk aversion in expected utility theory~\cite{pratt1964risk} and is the standard primitive in incentive-compatible mechanism design. We adopt this assumption throughout, and verify empirically in Appendix~\ref{app:emp_verify} that it holds in the federated deep learning setting studied here.

If the client $k$  knows the information of the aggregation weight $\alpha_k^{t-1}$ used in aggregating the client models, it can exploit this information to design the utility as a function of the aggregation weight $r$. We now estimate the utility function of the client based on the global performance on its held-out dataset. The server sends the value $\alpha_k^{t-1}$ for the client $k$ based on the distribution $\mathbf{s}$ along with the global model {$\boldsymbol{\theta}^t$}. 
Client $k$ estimates its impact on the aggregated global model by computing the model that would result if its own update were excluded, denoted as $\boldsymbol{\theta}^t_{-k}$. Since the client knows the exact weight, it uses it to remove its own contribution. 
\begin{equation}
\boldsymbol{\theta}_{-k}^t = \frac{ \boldsymbol{\theta}^t - \alpha_k^{t-1} \boldsymbol{\theta}^{t-1}_k } {1 -  \alpha_k^{t-1} } 
\end{equation}
where $\boldsymbol{\theta}_k^{t-1}$ is the client model in round $t-1$. 
\footnote{\scriptsize {This leave-one-out construction follows the broader counterfactual contribution-analysis perspective used in influence functions and data valuation methods~\citep{koh2017understanding,ghorbani2019data,wang2020principled}. Unlike these works, which quantify the contribution or value of data sources, we use the leave-one-out model as an anchor for estimating how client $k$'s utility varies with its aggregation weight.}} The client now creates a model $\tilde{\boldsymbol{\theta}}_k(r)$ with an aggregation parameter $r$ as $\tilde{\boldsymbol{\theta}}_k(r) = (1-r) \boldsymbol{\theta}^t_{-k} + r \boldsymbol{\theta}_k^{t}$, $\tilde{\boldsymbol{\theta}}_k(r)$ can be interpreted as the global model when the aggregation weight given to client $k$ is $r$. The client uses $\tilde{\boldsymbol{\theta}}_k(r)$ to assess its utility from the server model as a function of its aggregation weight. For a specific value of $r$, this is computed by evaluating the validation loss/accuracy on the client held-out data, $acc(.)$ denotes the accuracy. 
\begin{equation}
\label{eq:10}
\hat{u}_k(r) = \underset{(\mathbf{x},y) \sim \mathcal{D}_k^{val}}{\mathbb{E}}[acc(h(\mathbf{x};\tilde{\boldsymbol{\theta}}_k(r)),y)]
\end{equation} 
$h(x;\tilde{\boldsymbol{\theta}}_k(r))$ denotes the prediction made by client $k$ on data point $\mathbf{x}$ with the model $\tilde{\boldsymbol{\theta}}_k(r)$.
We observe that the clients attain best accuracies when $r \in (0,1)$. When $r$ is $0$, the accuracy will be lower as the client model is not utilized; on the other hand, when $r$ is $1$, the server model is ignored. This motivates us to parameterize the utility curve 
by $\zeta_k$ for each individual client $k$. 

The empirical utility $\hat u_k(r)$ is evaluated along the interpolation path between the leave-one-out global model $\boldsymbol{\theta}^t_{-k}$ and the locally updated model $\boldsymbol{\theta}^t_k$. For convex models such as linear logistic regression, linear SVMs with convex surrogate losses, or last-layer fine-tuning with a convex loss, the validation loss along this path is convex in $r$; hence a loss-based utility, e.g., negative validation loss, is concave. For deep models and accuracy-based utilities, exact concavity is not guaranteed. However, parameter-space interpolations in deep networks are often observed to traverse smooth low-loss regions~\citep{garipov2018loss,li2018visualizing}, and we empirically observe smooth, unimodal utility profiles across clients, rounds, and datasets (Appendix~\ref{app:emp_verify}). 

We model client utility as a concave function of the aggregation weight $r$, parameterized by $\zeta_k$. The parameters $\zeta_k$ are estimated per round from held-out validation data:
\vspace{-0.02in}
\begin{equation}
\zeta_k := \arg\min_{\tilde\zeta_k} \mathbb{E}_{r \sim \mathcal{U}(0,1)} \left[ \|u_k(r; \tilde\zeta_k) - \hat u_k(r)\|_2^2 \right],
\label{fit_ut_eq}
\end{equation}
where $\mathcal{U}(0,1)$ denotes the uniform distribution. The concavity assumption is made at the modeling level; it characterizes client preferences in the sense of mechanism design, and any concave parametric family is admissible. The fitted parameters $\zeta_k$ are transmitted to the server as a lightweight proxy for the client's utility. The fitted utility $u_k^t(r; \zeta_k^t)$ tracks the empirical utility $\hat u_k^t(r)$ closely across rounds (Appendix~\ref{app:emp_verify}); improvements in the surrogate therefore translate to improvements in realized client utility, ensuring server-side optimization aligns with client-side rationality.

\textit{Discussion:} This work is situated within the Cross-Silo Federated Learning (FL) paradigm, which involves a limited number of high-capacity clients such as institutions or large corporate entities. While compute, bandwidth, and power can still be practical concerns in this setting, they are not the focus of this work: we do not include them into the proposed utility function, and client participation is driven solely by model-performance utility. We leave the integration of resource-aware costs into the utility model to future work.

\subsection{Solving Optimization}
\label{frame_opt}
To maximize the participation of rational clients, the server must move beyond traditional deterministic aggregation. A single weight vector $r_k$ often fails to satisfy the diverse utility thresholds $\gamma_k$ of all silos simultaneously. Rather than optimizing over a single, static weight vector $r_k$, the server generates a set of $n_s$ candidate strategies $\{r^j_k\}_{j=1}^{n_{s}}$ for $k=1, \dots, N$. The influence of these strategies is determined by a simplex vector $\mathbf{s} = \{ s_1, s_2, \dots, s_{n_s} \}$, where $s_j$ represents the probability of selecting the $j^\text{th}$ strategy. Crucially, in our framework, the distribution $\mathbf{s}$ is not a learned
parameter but is kept fixed for a communication round, with values drawn from
a Dirichlet distribution characterized by a concentration parameter $\delta$.
By fixing this distribution, the server provides a predictable yet diverse set
of outcomes. This diversity is essential for exploiting the Jensen's gap; as we
show in Proposition~\ref{prop:jgap}, the concentration parameter $\delta$
modulates the mixture geometry of $\mathbf{s}$, which in turn shapes the induced
Jensen surplus through the optimized strategy matrix $\mathbf{R}$. The server
communicates to client $k$ the expected aggregation weight
$\alpha_k = \sum_i s_i r_k^i$, which determines the client's perceived utility.
To sustain participation, the server seeks strategies under which each client's
expected utility clears its threshold,
\begin{equation}
\sum_{i=1}^{n_s}{ s_i\, u_k(r_k^i)}  \geq \gamma_k + \epsilon,
\label{thr_eq}
\end{equation}
where $\epsilon > 0$ is a margin enforcing strict satisfaction. Because the
utilities are concave, Jensen's inequality gives
$u_k(\sum_{i=1}^{n_s} s_i r_k^i) \geq \sum_{i=1}^{n_s} s_i\, u_k(r_k^i)$, so
satisfying Eq.~\ref{thr_eq} guarantees the realized utility at $\alpha_k$ also
clears the threshold.


We formulate the final optimization below via the slack variables $\mathbf{t}\coloneqq \{t_1,...,t_N \}.$ 
\begin{align}
& \underset{\mathbf{t,R}}{\text{minimize}}
&& \frac{1}{2} {\lVert \mathbf{t} \rVert}_{2}^2 \notag\\
& \text{subject to}
&& \mathbf{R}\mathbf{1}_N = \mathbf{1}_{n_s},  \mathbf{1}^{T}_{n_s} \mathbf{R} = \frac{n_s}{N}\mathbf{1}_N^{T} ,\mathbf{r}_k \succeq 0,  \mathbf{t} \succeq 0 \notag\\
&&& \sum_{i=1}^{n_s}{ s_i{u_k(r_k^i)}}  \geq \epsilon + \gamma_k -t_k, \forall k.
\label{main_opt}
\end{align}
The above optimization becomes deterministic by setting $n_s = 1$ and ignoring the constraint $\mathbf{1}^{T}_{n_s} \mathbf{R} = \frac{n_s}{N}\mathbf{1}_N^{T}$.  
The constraint  $\mathbf{1}^{T}_{n_s} \mathbf{R} = \frac{n_s}{N}\mathbf{1}_N^{T}$ ensures that rows of $\mathbf{R}$ are non-zero, if $N=n_s$ it implies $\mathbf{R}$ is doubly stochastic. The server recomputes $\mathbf{t}$ and $\mathbf{R}$ at every round; however, we suppress the temporal index to maintain readability.
\vspace{-0.1in}
\paragraph{Fallback Mechanism.}
If the optimization problem yields an outcome in which no client satisfies the individual rationality constraint, in a given round $t$, the server invokes a fallback strategy. It uniformly samples a client $k$ and constructs a strategy that prioritizes this client by setting $r_k^i = r_k$ for all strategies $i$, where $r_k$ is chosen such that $u_k(r_k;\zeta_k) > \gamma_k^t$. The remaining aggregation mass is distributed among the other clients so as to satisfy the aggregation constraints of $\mathbf{R}$. This guarantees that at least one client satisfies its individual rationality constraint and participates in the next round, preventing degenerate updates with zero participation.

\subsection{Client-Server Algorithms and Communication Protocol}
\label{algo_protocol}

In the Client-Update of Algorithm~\ref{alg:feduca_routines}, at each round $t$, client $k$ evaluates the validation accuracy of the global model $\boldsymbol{\theta}^{t}$ and its local model $\boldsymbol{\theta}_k^{t-1}$ on $\mathcal{D}_k^{\mathrm{val}}$. Using the announced expected aggregation weight $\alpha_k^{t-1} = \mathbb{E}_{i \sim \mathbf{s}^{t-1}}[r_k^{i,t-1}]$, it computes its utility $u_k(\alpha_k^{t-1}, \zeta_k^{t-1})$, reflecting the expected benefit under randomized aggregation. The client accepts the global model if the rationality condition holds; otherwise, it continues from its local model. From the chosen initialization, it performs $\tau$ local SGD steps to obtain an updated model, and updates its participation threshold and utility parameters using the realized validation performance and $\hat{u}_k(r)$. 

If the rationality condition is satisfied, the client transmits both its updated model and $\zeta_k$; otherwise, it reports only $\zeta_k$. This captures rational, utility-driven participation. In the Server update routine of Algorithm~\ref{alg:feduca_routines}, the server aggregates updates $\{\boldsymbol{\theta}_k^{t}, \gamma_k^t, \zeta_k^t\}_{k \in S_t}$, where $S_t$ denotes participating clients. For non-participating clients, the server reuses their most recent updates. Retaining these stale updates, rather than re-normalizing over active clients, is critical for consistent utility estimation and sustained participation. Since utility depends on a client’s historical contribution (tracked via past updates), discarding this information would disrupt future utility estimation and weaken the participation mechanism. Incorporating stale updates preserves continuity and supports long-term engagement, as also observed in prior work such as FedVarp~\citep{jhunjhunwala2022fedvarp}. The full FedUCA algorithm is provided in Appendix~\ref{app:FedUCA}. 

\textbf{Scope:} We assume truthful reporting of thresholds, utility parameters, and model updates. Appendix~\ref{app:misreporting} discusses the impact of  threshold inflation on the proposed feasibility framework. Robustness to Byzantine behavior is orthogonal to this work and can be incorporated using existing verification or robust aggregation techniques.

\begin{wrapfigure}{r}{0.5\textwidth} 
\vspace{-15pt}
\begin{algorithm}[H]
\footnotesize
\SetCommentSty{scriptsize}
\SetKwFunction{ClientUpdate}{ClientUpdate}
\SetKwFunction{ServerUpdate}{ServerUpdate}
\SetKwProg{Fn}{Function}{:}{}

\Fn{\ClientUpdate{$\boldsymbol{\theta}^{t},\,\alpha_k^{t-1},\,\boldsymbol{\theta}_{k}^{t-1}, \gamma_{k}^{t-1},\,\mathcal{D}_k,\,\mathcal{D}_k^{\text{val}}$}}{

    \Comment{Evaluate metrics \& check rationality limit}
    $\boldsymbol{\theta}_{k}^{t,0} \gets \boldsymbol{\theta}^{t} \text{ if } {u_{k}}(\alpha_k^{t-1},\zeta_{k}^{t-1}) > \gamma_{k}^{t-1} \text{ else } \boldsymbol{\theta}_{k}^{t-1}$\;

    \Comment{Local SGD execution}
    \For{$l \gets 1$ \KwTo $\tau$}{
        $\boldsymbol{\theta}_{k}^{t,l} \gets \boldsymbol{\theta}_{k}^{t,l-1} - \eta_c \nabla f_k(\boldsymbol{\theta}_{k}^{t,l-1};\xi_k^{t,l-1})$\;
    }

    \Comment{Update utility parameters \& transmit}
    Evaluate $\hat{u}_k(r)$ as in~\eqref{eq:10} to compute $\zeta_k^t$ as Eq.~\ref{fit_ut_eq}\;
    $\gamma_{k}^{t} \gets u_k(\alpha_k^{t-1}, \zeta_k^t)$\;

    \KwRet $(\boldsymbol{\theta}_{k}^{t,\tau},\, \gamma_{k}^{t},\, \zeta_k^t)$ if ${u_{k}}(\alpha_k^{t-1},\zeta_{k}^{t-1}) > \gamma_{k}^{t-1}$ else $(\gamma_{k}^{t},\, \zeta_k^t)$ \;
}

\vspace{5pt} 

\Fn{\ServerUpdate{$(\boldsymbol{\theta}_{k}^{t},\gamma_{k}^t,\zeta_k^t)_{k\in S_t},
(\gamma_{k}^t,\zeta_k^t)_{k \in S_t^c}$}}{

    \Comment{ compute pseudo-gradients}
    compute $\{\alpha_1^t,\ldots,\alpha_N^t\}$ from $\mathbf{R}^{t}$ and $\mathbf{s}^{t}$\;
    $\mathbf{d}_k^t \gets \frac{\boldsymbol{\theta}^{t} - \boldsymbol{\theta}_k^{t}}{\eta_c} \quad (\forall k \in S_t)$\;

    \Comment{Aggregate active and stale updates}
    $\mathbf{g}^t \gets \sum_{k\in S_t} \alpha_k^t \mathbf{d}_k^t + \sum_{k\in S_t^c} \alpha_k^t \mathbf{y}_k^{t}$\;
    $\boldsymbol{\theta}^{t+1} \gets \boldsymbol{\theta}^{t,0} - \eta \tau \mathbf{g}^t$\;

    \Comment{Update stale trackers \& solve for next round}
    $\mathbf{y}_{k}^{t+1} \gets \mathbf{d}_k^t \text{ if } k \in S_t \text{ else } \mathbf{y}_{k}^{t}$\;
    Sample $\mathbf{s}^{t+1}$ and solve $\mathbf{R}^{t+1}$ as in Eq.~\eqref{main_opt}\;

    \KwRet $\boldsymbol{\theta}^{t+1},\,\mathbf{s}^{t+1},\,\mathbf{R}^{t+1}$\;
}

\caption{FedUCA: Client-Server Updates}
\label{alg:feduca_routines}
\end{algorithm}
\vspace{-25pt}
\end{wrapfigure}



\subsection{Analysis}
\label{analysis}
We present our theoretical guarantees in three complementary parts. Propositions~\ref{prop:kkt} and~\ref{prop:jgap} characterize participation conditions and the expected Jensen's gap as a continuous, tunable function of the Dirichlet concentration $\delta$, establishing that surplus is available at moderate $\delta$ and vanishes at the extremes. Proposition~\ref{prop:prob_jgap} confirms that the event of a beneficial Jensen surplus occurs with non-zero probability, ensuring the mechanism is non-degenerate. Proposition~\ref{informal_prop} is a \emph{conditional convergence result}: it characterizes the optimization landscape FedUCA operates in, given that participation sustains a minimum cohort size $S \ge 1$ per round. Together, the first two results motivate why FedUCA can sustain participation; the third characterizes convergence quality once participation is sustained.

\begin{proposition}
Suppose $\epsilon = 0$, and all utilities $u_k(.)$ are concave then the solution to optimization~\ref{main_opt} satisfies either 
$\sum_{i=1}^{n_s}{ s_i{u_k(r_k^i)}}  =  \gamma_k -t_k$,  with $t_k > 0$ or $\sum_{i=1}^{n_s}{ s_i{u_k(r_k^i)}}  \geq  \gamma_k$, with $t_k = 0$. In the former case, the constraint is satisfied when $u_k(r_k^i) = \frac{s_i}{\sum_{j=1}^{n_s}s_j^2} (\gamma_k - t_k)$. If we further assume that $u_k(.)$  is strictly concave, the optimal solution $\mathbf{R}^*, \mathbf{t}^*$ is unique, if the former holds $\forall k \in [N] $.

\label{prop:kkt}
\end{proposition}


\begin{proof}[Proof Sketch] 
The proof follows by analyzing the KKT conditions. The details are provided in the appendix in proposition~\ref{app:prop_kkt}. The generality of the result holds independent of $\epsilon$ as it can always be absorbed into the $\gamma_k$. 
\end{proof}
\vspace{-0.1in}
 The proposition states that the solution exists in two classes either satisfying $\sum_{i=1}^{n_s}{ s_i{u_k(r_k^i)}}  \geq  \gamma_k$ (case $1$) or $\sum_{i=1}^{n_s}{ s_i{u_k(r_k^i)}}  =  \gamma_k -t_k$ (case $2$) with $t_k >0$. When the server gives the client $k$ the aggregation weight $\alpha_k = \sum_{i=1}^{n}s_ir_k^i$. Owing to the Jensen's inequality, in the case $1$ the clients participate as we have  $
u_k(\sum_{i=1}^{n_s}s_ir_k^i) \geq \sum_{i=1}^{n_s}{ s_i{u_k(r_k^i)}}  \geq \gamma_k
$. In the case $2$, we have $u_k(\sum_{i=1}^{n_s}s_ir_k^i) \geq \sum_{i=1}^{n_s}{ s_i{u_k(r_k^i)}}  = \gamma_k -t_k
$. In this scenario, the client may or may not participate depending on whether $u_k(\sum_{i=1}^{n_s}s_ir_k^i) \geq \gamma_k$ or $u_k(\sum_{i=1}^{n_s}s_ir_k^i) < \gamma_k$. 
Especially the participation is possible if $u_k(\sum_{i=1}^{n}s_ir_k^i)$ exceeds  $\sum_{i=1}^{n}{ s_i{u_k(r_k^i)}}$ by at least $t_k$. Towards this end, we analyze the induced Jensen's gap associated with client $k$,

\begin{equation}
{\mathcal{J}_k}(\mathbf{s}) = u_k(\sum_{i=1}^{n_s}s_ir_k^i) - \sum_{i=1}^{n_s}{ s_i{u_k(r_k^i)}},
\label{eq:jgap}
\end{equation}
as a function of $\mathbf{s}$. Since $\mathbf{s}$ is randomly chosen from the Dirichlet distribution parameterized by $\delta$. Unlike standard stochastic regularization, the Jensen's gap arises from the concavity of client utilities and directly induces a surplus in expected utility, which is explicitly leveraged to satisfy participation constraints. We note that ${\mathcal{J}_k}(\mathbf{s})$ is a random variable, and we characterize its expectation under the Dirichlet mixture distribution.

\begin{proposition}
Let the expected Jensen's gap for client $k$ be defined as $\bar{\mathcal{J}}_k(\delta) \coloneqq \mathbb{E}_{\mathbf{s}\sim \text{Dir}(\delta)} [\mathcal{J}_k(\mathbf{s})]$ for a concentration parameter $\delta > 0$. The instantaneous gap is given by $\mathcal{J}_k(\mathbf{s}) \coloneqq u_k\big(\sum_{i=1}^{n_s} s_i r_k^i\big) - \sum_{i=1}^{n_s} s_i u_k(r_k^i)$, where $u_k(\cdot)$ is a strictly concave utility function and the aggregation weights $r_k^i$ are obtained by solving Optimization~\ref{main_opt}, where the constraint is binding. The expected gap $\bar{\mathcal{J}}_k(\delta)$ vanishes at the extreme concentration regimes, specifically when $\delta \le \delta_{min}$ or $\delta \ge \delta_{max}$ (where $\delta_{min} \ll 1$ and $\delta_{max} \gg 1$). Furthermore, $\bar{\mathcal{J}}_k(\delta)$ is a continuous function of $\delta$ on the open interval $(\delta_{min}, \delta_{max})$.

\label{prop:jgap}
\end{proposition}
\begin{proof}[Proof Sketch] 
The proof follows by analyzing the cases $\delta \ll 1 $ and $\delta \gg 1$. We then argue the continuity of Jensen's gap as a function of the concentration parameter $\delta$ by appealing to the Dominated Convergence Theorem. Since the gap is positive and due to continuity it attains the maximum when $\delta \in (\delta_{min}, \delta_{max})$. The complete proof is provided in the proposition~\ref{app:jgap_dct}. 
\end{proof}
\vspace{-0.1in}
The proposition implies that choosing $\delta$ which is neither too small nor too big is better. Empirical measurements of $\bar{\mathcal{J}}_k(\delta)$ on CIFAR-10 confirm the predicted non-monotonic profile, with the gap smaller at $\delta = 0.01$ and $\delta = 100$ and peaking at $\delta \in [0.5, 1.0]$ (Section~\ref{emp:jen_gap}, Table~\ref{tab:jensens_gap_distribution}).

We now provide the sufficient conditions under which the Jensen's gap ${\mathcal{J}_k}(\mathbf{s})$ exceeds a particular threshold. In particular, we provide the lower bound on the probability that  ${\mathcal{J}_k}(\mathbf{s}) \geq t$.

\begin{proposition}
Suppose $u_k(\cdot)$ is concave, $\mathcal{J}_k(s)\ge 0$, and $t \leq \underset{\mathbf{s}\sim {Dir}(\delta)} {\mathbb{E}} [\mathcal{J}_k(\mathbf{s})]$ and $n_s \le N$ then $Pr[ \mathcal{J}_k(\mathbf{s}) \ge t ] \ge \frac{({\mathbb{E}}[\mathcal{J}_k(\mathbf{s})]-t)^2} {{M_J}^2} $, for some $M_J$ and  $M_J \le1$, if the utility is accuracy.
\label{prop:prob_jgap}
\end{proposition}
\vspace{-0.15in}
\begin{proof}[Proof Sketch] The key is to note that ${\mathcal{J}_k}(\mathbf{s})$ is a non-negative random variable, the proof follows by using the Paley-Zygmund inequality~\citep{PETROV20072703}. The proof details are in the proposition~\ref{app:pz} of the Appendix. The important thing is to note that this probability is non-zero.
\end{proof}
\vspace{-0.15in}
\textit{Discussion:} The stochastic formulation does not enlarge the simplex of aggregation weights. Rather, under concave utilities, it enables Jensen's gap surplus that can satisfy IR constraints when a single deterministic aggregation vector cannot. Thus, randomization can induce better participation. 

\textbf{Assumptions for Convergence:} We adpot the standard non-convex FL assumptions used in partial-participation analyses~\citep{jhunjhunwala2022fedvarp,nguyen2022federated}: each $f_k$ is $L$-smooth, stochastic gradients are unbiased with variance at most $\sigma^2$, client heterogeneity satisfies $\|\nabla f_k(\boldsymbol{\theta})-\nabla f(\boldsymbol{\theta})\|^2\le \sigma_g^2$, local gradients are bounded as $\|\nabla f_k(\boldsymbol{\theta})\|^2\le G$, stale delays satisfy $\nu_k^t\le \nu_{\max}$. We assume that the participation set $S_t$ is uniformly lower bounded, i.e., $1 \le S \le |S_t| \le N$, where \(S\) is the minimum number of participating clients per round. In our fallback implementation, the weakest guarantee corresponds to \(S=1\). 

\begin{proposition}[Convergence decomposition]
\label{prop:convergence_decomposition}
Under standard smoothness, bounded variance, bounded gradient, and bounded stale-delay assumptions, FedUCA satisfies
\[
\frac{1}{T}\sum_{t=0}^{T-1}
\mathbb{E}\|\nabla f(\boldsymbol{\theta}^{(t,0)})\|^2
\le
\mathcal{E}_{\mathrm{init}}
+
\mathcal{E}_{\mathrm{agg}}
+
\mathcal{E}_{\mathrm{var}}
+
\mathcal{E}_{\mathrm{drift}}
+
\mathcal{E}_{\mathrm{sys}} .
\]
Here, $f(\boldsymbol{\theta}) = \sum_{k=1}^{N}\beta_kf_k(\boldsymbol{\theta})$, $\sum_{k=1}^{N} \beta_k = 1$ , $T$ denotes rounds, $\mathcal{E}_{\mathrm{init}}=\mathcal{O}(1/\sqrt(T))$ is the optimization term, 
$\mathcal{E}_{\mathrm{agg}}$ is proportional to
$\frac{1}{T}\sum_{t=0}^{T-1}\mathbb{E}\|\boldsymbol{\beta}-\boldsymbol{\alpha}^t\|^2$,
and the remaining terms capture stochastic variance, local-update drift, stale-update drift, and data heterogeneity. The $\mathcal{E}_{\mathrm{sys}}$ contains a non-vanishing bias term $\mathcal{O}((N-S) \sigma^2) $. The full statement and proof are in the Appendix~\ref {conv_analysis}.
\label{informal_prop}
\end{proposition}
\vspace{-0.1in}
\begin{proof}[Proof Sketch]
\vspace{-0.15in}
 We write the exact FedUCA average-update direction as
$\bar{\mathbf{v}}^t=
\sum_{k=1}^N
\Big(
I_k^t\alpha_k^t\mathbf{h}_k^t+
(1-I_k^t)\alpha_k^t\mathbf{h}_k^{t-\nu_k^t}
\Big)$
where $I_k^t$ denotes participation, $\alpha_k^t$ is the FedUCA aggregation weight, and $\nu_k^t\le\nu_{\max}$ is the stale delay. Applying the smoothness descent lemma to this trajectory separates the progress term from stochastic variance, local drift, stale-update drift, and aggregation discrepancy.

Conditioning on the history ${H}_t$ makes the reported utilities, thresholds, participation decisions, and aggregation weights fixed for round $t$. We then add and subtract the full gradient
$\nabla f(\boldsymbol{\theta}^{(t,0)})=
\sum_{k=1}^N \beta_k\nabla f_k(\boldsymbol{\theta}^{(t,0)}),$
which yields the aggregation-discrepancy penalty $\mathcal{E}_{agg}$ upper bounded by $\mathbb{E}\|\boldsymbol{\beta}-\boldsymbol{\alpha}^t\|^2$. If $\boldsymbol{\beta} = (1/N)\mathbf{1}_N$, from Lemma~\ref{lem:agg_discrepancy}, we have  $\mathbb{E}\|\boldsymbol{\beta}-\boldsymbol{\alpha}^t\|^2 \leq (n_s - 1)/[N(n_s\delta + 1)]$ and eventually, we have  $\mathcal{E}_{agg} \leq C(n_s - 1)/[S(n_s\delta + 1)]$, where $C$ is a constant. 
The active-client drift follows from standard local-SGD bounds, while stale-client drift is controlled by the bounded-delay condition $\nu_k^t\le\nu_{\max}$ and the minimum active cohort size $S$. Telescoping the resulting descent inequality over $T$ rounds gives $
\frac{1}{T}\sum_{t=0}^{T-1}
\mathbb{E}\|\nabla f(\boldsymbol{\theta}^{(t,0)})\|^2
\le
\mathcal{E}_{\mathrm{init}}
+
\mathcal{E}_{\mathrm{agg}}
+
\mathcal{E}_{\mathrm{var}}
+
\mathcal{E}_{\mathrm{drift}}
+
\mathcal{E}_{\mathrm{sys}} .
$
All terms except the structural aggregation bias $\mathcal{E}_{\text{agg}}$ and the system error floor $\mathcal{E}_{\text{sys}}$ decay at a rate of $\mathcal{O}(1/T)$ or $\mathcal{O}(1/\sqrt{T})$. Consequently, FedUCA converges to an error neighborhood defined by these non-vanishing terms. 
\end{proof}

FedUCA should not be interpreted as an unbiased variant of FedAvg. The server deliberately modifies aggregation weights to satisfy individual rationality constraints, rather than solely approximating the full-participation descent direction. Consequently, the aggregation-discrepancy term $\mathcal{E}_{agg}$ is not merely a proof artifact; it captures the mechanism-induced bias incurred when the server trades aggregation fidelity for sustained rational participation. In the limiting regime where participation is full and the FedUCA weights match the target objective weights, this discrepancy vanishes and the standard FL behavior is recovered.

The bound $\mathbb{E}\|(1/N)\mathbf{1}_N - \boldsymbol{\alpha}^t\|^2 \le (n_s - 1)/[N(n_s\delta + 1)]$ (Lemma~\ref{lem:agg_discrepancy}) reveals a quantitative tradeoff governing FedUCA's convergence. The Dirichlet concentration $\delta$ simultaneously controls two opposing forces: increasing $\delta$ tightens the aggregation discrepancy (the bound decays as $1/(n_s\delta + 1)$) but reduces the Jensen surplus that sustains participation (Proposition~\ref{prop:jgap}). The optimal $\delta$ balances these forces, which is consistent with the empirical participation peak at $\delta \in [0.5, 1.0]$ (Figure~\ref{fig:var_delta}) and correlates with Jensen's-gap peak. Similarly, the strategy count $n_s$ enters the numerator linearly but also appears in the denominator through $n_s\delta$; for large $n_s$ the bound saturates at $\approx 1/(N\delta)$, confirming that additional strategies yield diminishing returns.
\section{Experiments}
\label{exp_sec}
\textbf{Datasets and Models:} We benchmark on three different datasets: CIFAR-10, CIFAR-100~\citep{cifar}, F-MNIST~\citep{xiao2017fashionmnistnovelimagedataset}. For the correct evaluation of a method's persuasive power, the design of the data split amongst clients is of utmost importance. We provide an extensive discussion of the data splits in~\ref{sec:data_prep} of the Appendix and request the readers to refer to the same. The data partition Dirichlet ($0.5, 1.0$) implies that training data is distributed according to the Dirichlet distribution based on the concentration parameter $0.5$, and the validation data uses $1.0$ concentration parameter similar to~\citep{acar2021federated}. We use the CVXPY~\citep{diamond2016cvxpy,agrawal2018rewriting}  to solve the optimization~\ref{main_opt} for FedUCA. For CIFAR-10, we adopt a CNN architecture akin to that of \citet{mcmahan2017communication} and~\citep{acar2021federated}, comprising two convolutional layers with 64 5 × 5 filters, followed by two fully connected layers with 384 and 192 units, respectively, and a softmax output layer. We use a fully-connected neural network architecture for F-MNIST with 2 layer MLP. The number of neurons in the layers are 200 and 100. For evaluation on CIFAR-100, we use a ResNet-18 model~\citep{he2016deep} pre-trained on ImageNet-1K~\citep{5206848} dataset where we use group normalization 
instead of batch normalization as in~\citep{acar2021federated}. While FedUCA theory is generic, for experiments we prefer our utility to be the concave quadratic, $u_k(r,\zeta_k) = -a_kr^2+b_kr+c_k$, where $a_k$ is positive, and $\zeta_k := \{a_k,b_k,c_k\}$. Since these are parameterized by $3$ real numbers, this is communication efficient. The hyperparameter details are in the Appendix~\ref{hyper_details}.\\

\begin{table*}[t]
  \caption{Server model performance on CIFAR-10, F-MNIST, and CIFAR-100 under different data splits. The data partition Dirichlet ($0.5, 1.0$) implies that training data is distributed according to the Dirichlet distribution based on the concentration parameter $0.5$, and the validation data uses $1.0$.}
  \centering
  \setlength{\tabcolsep}{3.0pt}        
  \renewcommand{\arraystretch}{1.0} 
  \captionsetup{font=scriptsize}         
  \small                             
  \scriptsize                                  

  \begin{tabular}{@{}l|cc|cc|cc@{}}
    \toprule
    & \multicolumn{2}{c|}{\textbf{CIFAR-10}}
    & \multicolumn{2}{c|}{\textbf{F-MNIST}}
    & \multicolumn{2}{c}{\textbf{CIFAR-100}} \\
    \cmidrule(l){2-3}\cmidrule(l){4-5}\cmidrule(l){6-7}
    & \textbf{Dirichlet 0.5,1.0} & \textbf{Dirichlet 1.0,1.0}
    & \textbf{Dirichlet 0.5,1.0} & \textbf{Dirichlet 1.0,1.0}
    & \textbf{Dirichlet 0.5,1.0} & \textbf{Dirichlet 1.0,1.0} \\
    \midrule
    \textbf{FedAvg}
      & {\scriptsize $16.63_{\pm1.16}$} 
      & {\scriptsize $18.31_{\pm1.66}$}
      & {\scriptsize $53.26_{\pm2.78}$} 
      & {\scriptsize $54.64_{\pm2.05}$}
      & {\scriptsize $04.05_{\pm0.06}$} 
      & {\scriptsize $03.63_{\pm0.50}$} \\
    \midrule
    \textbf{IncFL}
      & {\scriptsize $32.35_{\pm0.96}$} & {\scriptsize $34.60_{\pm0.16}$}
      & {\scriptsize $37.07_{\pm0.20}$} & {\scriptsize $36.69_{\pm0.04}$}
      & {\scriptsize $04.18_{\pm0.41}$} & {\scriptsize $12.89_{\pm0.94}$} \\
    \midrule
    \textbf{FedProx}
      & {\scriptsize $17.44_{\pm0.86}$} & {\scriptsize $17.63_{\pm2.14}$}
      & {\scriptsize $54.68_{\pm2.01}$} & {\scriptsize $53.79_{\pm0.55}$}
      & {\scriptsize $06.35_{\pm2.91}$} & {\scriptsize $04.06_{\pm0.22}$} \\
    \midrule
    \textbf{AAggFF-S}
      & {\scriptsize $16.13_{\pm0.50}$} & {\scriptsize $17.90_{\pm0.33}$}
      & {\scriptsize $52.84_{\pm3.18}$} & {\scriptsize $68.92_{\pm6.22}$}
      & {\scriptsize $5.60_{\pm0.33}$} & {\scriptsize $04.53_{\pm0.86}$} \\
    \midrule
    \textbf{FedAvg-WS}
      & {\scriptsize $47.70_{\pm0.47}$}
      & {\scriptsize $51.63_{\pm  0.71}$} 
      & {\scriptsize $77.60_{\pm  0.57}$}
      & {\scriptsize $75.75_{\pm  0.23}$} 
      & {\scriptsize $33.50_{\pm0.26}$} 
      & {\scriptsize $48.93_{\pm0.05}$}  \\
    \midrule
    \textbf{FedUCA (Ours)}
      & {\scriptsize\bfseries\boldmath $66.89_{\pm 01.56}$} 
      & {\scriptsize  \bfseries\boldmath $71.86_{\pm 1.17}$}
      & {\scriptsize\bfseries\boldmath $83.18_{\pm 0.40}$} & {\scriptsize\bfseries\boldmath $83.48_{\pm 0.85}$} 
      & {\scriptsize\bfseries\boldmath $42.50_{\pm 0.29}$} & {\scriptsize\bfseries\boldmath $50.15_{\pm 0.32}$} \\
    \bottomrule
  \end{tabular}
  \label{tab:performance1}
\end{table*}
\textbf{Baselines and evaluation protocol:} Since FedUCA operates without monetary incentives, market-based incentive baselines are not directly applicable. We equip four standard FL methods, FedAvg, FedProx, IncFL, and AAggFF-S~\citep{hahn2024pursuingoverallwelfarefederated}, with a rationality check (Eq.~\ref{RC}): a client accepts the global model only if it improves upon the local model on the client's held-out set. This preserves each method's original update rule and isolates the effect of aggregation design under rational behavior. To ensure a rigorous and fair comparison, we equip all baseline methods with the exact same stale update retention mechanism utilized by FedUCA. To further test whether baseline degradation is a transient initialization effect, we additionally evaluate FedAvg under a warm-start protocol (FedAvg-WS) with $K = 5$ rounds of forced full participation before the rationality check activates. All methods use identical architectures, data splits, training budgets, and participation dynamics; implementation details are in Appendix~\ref{sup:hyper}.
\begin{wrapfigure}{r}{0.65\textwidth}
  \vspace{-10pt}
  \centering
  \includegraphics[width=\linewidth]{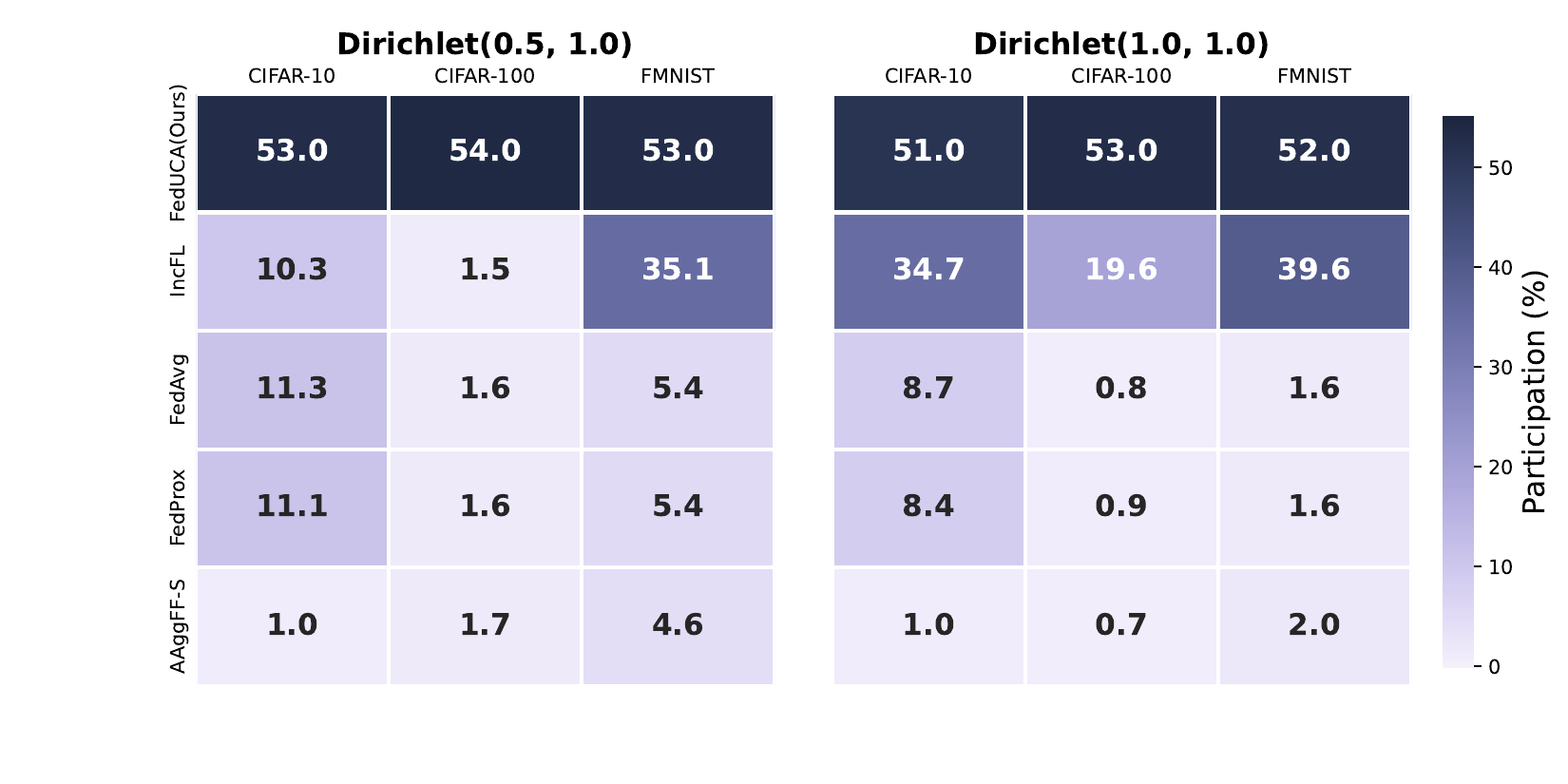}
  \caption{
    Comparison of participation rates across different FL methods on CIFAR-10, CIFAR-100, and FMNIST under varying degrees of data heterogeneity $\text{Dirichlet}(0.5,1.0)$ and $\text{Dirichlet}(1.0,1.0)$.
  }
  \label{fig:final_perf}
  \vspace{-10pt}
\end{wrapfigure}

\vspace{-0.1in}
\textbf{Results:} As shown in Table~\ref{tab:performance1} and Figure~\ref{fig:final_perf}, FedUCA achieves substantially higher accuracy and participation across the datasets and splits. Under rational participation, baselines suffer from a self-reinforcing collapse: the initial global model (an average of heterogeneous local models) often degrades individual client performance, clients reject it, the server receives fewer updates, and the next global model deteriorates further. This cycle drives participation below $12\%$ for FedAvg, FedProx, and AAggFF-S on CIFAR-10 and CIFAR-100 (Figure~\ref{fig:final_perf}). IncFL partially mitigates this but still plateaus at $\sim 35\%$ participation on CIFAR-10, Dir(1.0,1.0). FedUCA breaks this cycle by optimizing aggregation weights to satisfy individual rationality constraints directly, sustaining $51$-$54\%$ participation and achieving higher accuracy than the baselines. We provide further discussion of the rationality check and alternative thresholding strategies in Appendix~\ref {rc_ts}.  We additionally evaluate on $20$ Newsgroups dataset~\citep{20newsgroups} in Appendix~\ref{20_news_group}. In~\ref {sup:bm_details} of the Appendix, we provide a detailed discussion on how we introduce the rationality check by clients when evaluated for the benchmarking methods. In the Sec~\ref{sec:data_prep}, we provide details on the data preparation.

\textbf{Warm-start evaluation:} 
To test whether the baseline gap is a transient initialization effect, 
we evaluate FedAvg under a warm-start protocol (FedAvg-WS) with $K = 5$ 
rounds of forced full participation. Under the heterogeneous 
Dir$(0.5, 1.0)$ splits, FedUCA leads FedAvg-WS by $+19.2$ points on 
CIFAR-10, $+5.6$ on F-MNIST, and $+9.0$ on CIFAR-100; under 
Dir$(1.0, 1.0)$, FedUCA leads FedAvg-WS by $+20.2$ points on CIFAR-10, 
$+7.7$ points on F-MNIST, and $+1.2$ points on CIFAR-100. FedUCA's 
advantage is consistent across varying heterogeneity regimes. 
Warm-start is not a benign baseline modification: in a rational 
cross-silo deployment, forced participation presupposes monetary 
compensation or regulatory authority~\citep{kang2019incentive}, system 
components external to the aggregation method. FedUCA replaces these 
external components endogenously through participation-management, 
achieving superior performance without requiring infrastructure outside 
the aggregation mechanism.
\vspace{-0.1in}
\subsection{Impact on varying the Dirichlet Concentration $\delta$}
\vspace{-0.1in}
\begin{wrapfigure}{r}{0.45\textwidth}
  \vspace{-10pt}
  \centering
  \includegraphics[width=\linewidth]{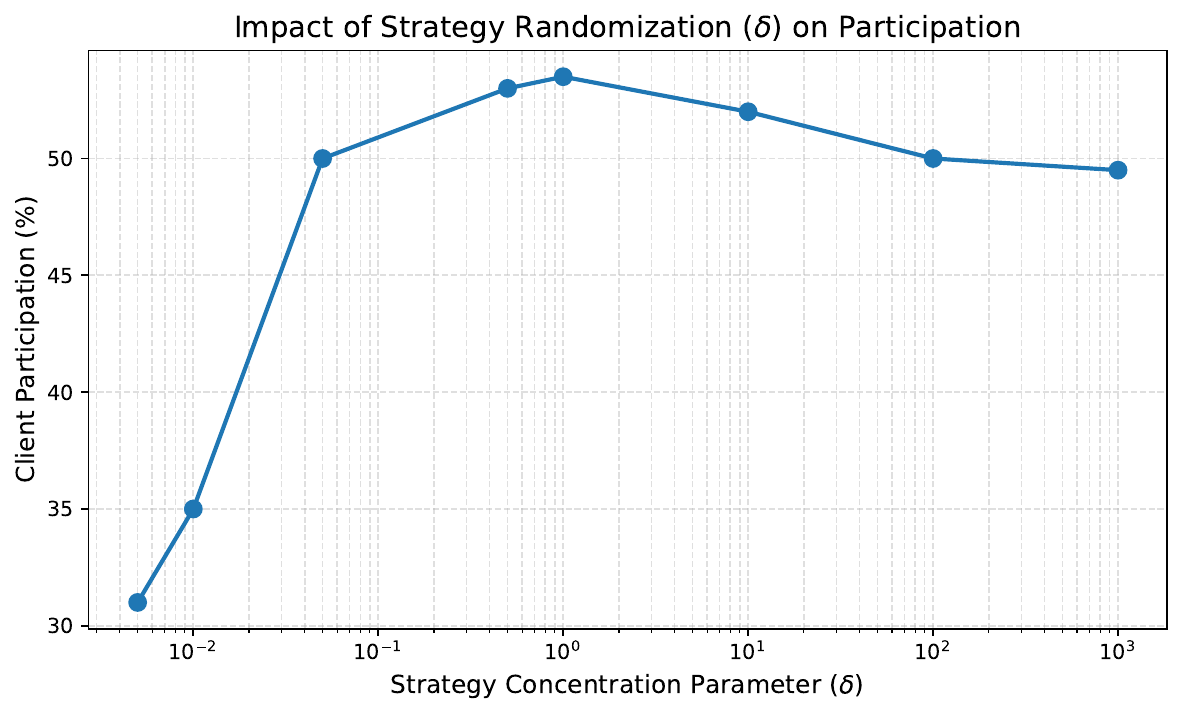}
  \caption{Impact of Dirichlet concentration $\delta$ on participation for CIFAR-10.}
  \label{fig:var_delta}
  \vspace{-25pt}
\end{wrapfigure}
Figure~\ref{fig:var_delta} illustrates the effect of the Dirichlet concentration parameter $\delta$. Very small $\delta$ concentrates strategies near simplex vertices, while very large $\delta$ concentrates them near the uniform vector. In both extremes, the induced Jensen surplus is limited. Intermediate values provide useful strategy diversity without collapsing to either extreme, yielding higher participation. This empirical trend provides a strong validation of our theoretical claim: strategic randomization is not a mere heuristic, but a mathematically grounded method to bridge the Jensen's gap and stabilize participation in rational FL. The participation rate exhibits the non-monotonic dependence on $\delta$ predicted by Proposition~\ref{prop:jgap}, rising from $\sim 35\%$ at $\delta=0.01$ to a peak of $\sim 54\%$ at $\delta=1$, before declining to $\sim 49\%$ at $\delta=1000$. Independent measurements of the Jensen gap on the same setup Appendix~\ref{emp:jen_gap}) show a similar non-monotonic trend, with larger gaps occurring in the same intermediate concentration regime. This alignment supports the proposed mechanism: intermediate Dirichlet concentrations induce useful strategy diversity, which increases the Jensen surplus available to Optimization~\ref{main_opt} and thereby improves participation feasibility.


\subsection{Comparison with single strategy}
\begin{wraptable}{r}{0.6\textwidth} 
    \vspace{-10pt} 
    \centering
    \caption{Server Accuracy between single and multiple strategies}
    \label{tab:k1_k5_data}
    \resizebox{\linewidth}{!}{ 
    \begin{tabular}{l c c}
        \toprule
        \multirow{2}{*}{\textbf{Strategies}} & \textbf{CIFAR-10} & \textbf{CIFAR-100} \\
        & (ACC \%/ Participation \%) & (ACC \% / Participation \%) \\
        \midrule
        $n_s=1$ & 63.80 (29) & 36.80 (37) \\
        $n_s=5$ & 66.89 (53) & 42.50 (54) \\
        \bottomrule
    \end{tabular}
    }
    \vspace{-10pt} 
\end{wraptable}

In the table.~\ref{tab:k1_k5_data} we compare the performance of the server model with $n_s=1$ vs $n_s=5$. By increasing the number of strategies ($n_s=5$), the solver optimizes the pessimistic surrogate, generating utility surplus via the Jensen's gap. This empirically confirms our theory that multiple strategies exploit the Jensen's gap and increase the participation. 

\subsection{Computation}

We evaluate FedUCA's computational footprint on NVIDIA-RTX A5000 GPUs and find the overhead to be minimal. On the client-side, estimating the utility curve only requires forward passes on the validation set and a lightweight least-squares fit. 
On the server-side, the server Optimization is a convex program independent of the model dimension $d$. 
The detailed breakdown is provided in the Appendix~\ref{sec:comp_cmplx}.

\section{Limitations and Future work}
\vspace{-0.07in} 
FedUCA addresses cross-silo federated learning under rational, performance-based client participation, with concave client utilities standard in mechanism design~\citep{pratt1964risk} and verified empirically in this setting (Appendix~\ref{app:emp_verify}). The framework assumes truthful reporting of utility parameters, which is supported in cross-silo deployments by repeated interaction and institutional accountability; extensions incorporating formal incentive-compatibility guarantees are a natural direction. On the theoretical side, Proposition~\ref{prop:convergence_decomposition} together with Lemma~\ref{lem:agg_discrepancy} couples the aggregation-discrepancy term to the Dirichlet concentration $\delta$, identifying a quantitative tradeoff with Jensen surplus. Bounding $S$ in terms of system parameters is a refinement we leave for future work.
\section{Conclusion}
\vspace{-0.07in}
We introduced FedUCA, which reformulates federated aggregation as a constrained feasibility problem enforcing individual rationality through utility-threshold constraints. A stochastic variant samples weights from a Dirichlet distribution, exploiting utility concavity and Jensen's gap to offset utility deficits, with the concentration parameter $\delta$ governing the expected gap. Experiments on CIFAR-10, CIFAR-100, F-MNIST, and 20 Newsgroups demonstrate gains in federation stability and global accuracy across modalities. By reframing aggregation as a feasibility problem under rational participation, FedUCA provides a practical, non-monetary alternative to equilibrium-based and incentive-driven approaches, enabling stable cross-silo collaboration without external rewards or strategic modeling.

\section*{Acknowledgment} This work is partially supported by the P3DX project, seed-funded by the Ministry of Electronics
and Information Technology (MeitY). The authors would also like to acknowledge compute support received from Kotak-IISc
AI-ML Center (KIAC), IISc.

\bibliography{main}
\bibliographystyle{plainnat}

\appendix

\section{Technical appendices and supplementary material}



\startcontents[appendix]

\printcontents[appendix]{}{1}{}

\subsection{Notation}

\begin{itemize}
    \item $N$: Number of clients
    \item $k \in [N]$: Client index
    \item $t$: Communication round index
    \item $\boldsymbol{\theta}^t$: Global server model at round $t$, we also use $\boldsymbol{\theta}^{t,0}$ for the same. 
    
    \item $\boldsymbol{\theta}_k^t$: Local model/update of client $k$ at round $t$
    \item $f_k(\boldsymbol{\theta})$: Local objective of client $k$
    \item $f(\boldsymbol{\theta})=\sum_{k=1}^N \beta_k f_k(\boldsymbol{\theta})$: Global objective
    \item $\boldsymbol{\beta}$: Target objective weights over clients
    \item $\alpha_k^t$: Effective aggregation weight assigned to client $k$ at round $t$
    \item $\mathbf{s}=(s_1,\ldots,s_{n_s})$: Probability vector over aggregation strategies
    \item $n_s$: Number of aggregation strategies
    \item $\mathbf{R}\in\mathbb{R}^{n_s\times N}$: Strategy matrix; row $i$ is one aggregation vector over clients
    \item $r_k^i$: Aggregation weight assigned to client $k$ under strategy $i$
    \item $\hat u_k(r)$: Empirical utility of client $k$ at aggregation weight $r$
    \item $u_k(r;\zeta_k)$: Fitted concave surrogate utility of client $k$
    \item $\zeta_k$: Parameters of client $k$'s fitted utility function
    \item $\gamma_k$: Minimum participation threshold of client $k$
    \item $t_k$: Slack variable for client $k$'s rationality constraint
    \item $S_t$: Set of clients participating at round $t$
    \item $S$: Minimum number of participating clients per round
    \item $I_k^t$: Indicator that client $k$ participates at round $t$
    \item $\nu_k^t$: Staleness delay of client $k$ at round $t$
    \item $\nu_{\max}$: Maximum allowed staleness delay
    \item $\mathbf{y}_k^t$: Stale update tracker for client $k$
    \item $\mathcal{J}_k(\mathbf{s})$: Jensen surplus for client $k$ under mixture $\mathbf{s}$
    \item $\delta$: Dirichlet concentration parameter for sampling $\mathbf{s}$
    \item $\mathcal{E}_{\text{init}}$: Transient optimization error term
    \item $\mathcal{E}_{\text{agg}}$: Aggregation-discrepancy / mechanism-induced bias term
    \item $\mathcal{E}_{\text{var}}$: Stochastic variance term
    \item $\mathcal{E}_{\text{drift}}$: Local and stale-update drift term
    \item $\mathcal{E}_{\text{sys}}$: System floor
\end{itemize}

\subsection{FedUCA Illustration and Algorithm}
\label{app:FedUCA}
\begin{algorithm}[!h]
\SetCommentSty{small}
\label{FedUCA}
\caption{FedUCA}
\label{alg:FedUAI_Server}
\KwIn{$\boldsymbol{\theta}$, $\mathbf{s}$ \Comment* {initial server model, initial aggregation strategy}}
\KwOut{$\boldsymbol{\theta}^{T}$ \Comment* {final aggregated server model in $T$ communication rounds}}

\SetKwInOut{Hyperparameters}{Hyperparameters}
\Hyperparameters{$T$, $\tau$, $n_s$,$\eta_c$,$\eta_s$,$\eta$
                 \Comment*{communication rounds, epochs, aggregation strategies, client learning rate, server learning rate and overall learning rate}}

\textbf{Initialize:}\\     
\Indp                         
$\boldsymbol{\theta}^{0}\leftarrow\boldsymbol{\theta}$; $\mathbf{s}^{-1}\leftarrow\mathbf{s}$; $\mathbf{R}^{-1}\leftarrow \frac{1}{N}\mathbf{1 1^T}$; $\mathbf{\gamma}_k^{-1}\leftarrow0$, $\boldsymbol{\theta}_k^{-1}\leftarrow\boldsymbol{\theta}$\;
\Indm

\For{$t \gets 0$ \KwTo $T-1$}{
    Server broadcasts $\boldsymbol{\theta}^{t}$ to all $N$ clients with indices $[N]$\;
        
    \For{$k \gets 1$ \KwTo $N$}{
        $\mathbf{r}_{k}^{\,t-1} \leftarrow \mathbf{R}^{\,t-1}_{:,k}$\;
        compute $\alpha_k^{t-1} = \sum_{j=1}^{n_s} s_j^{t-1}r_k^{j,t-1}$ \;
        Server sends $\boldsymbol{\theta}^{t}$, $\mathbf{s}^{t-1}$, 
        $\alpha_{k}^{t-1}$ to client $k$\;

        $\boldsymbol{\theta}_{k}^{t}, \gamma_{k}^{t}, \zeta_k^{t} = $ 
        \textbf{ClientUpdate}$(\boldsymbol{\theta}^{t},\,\alpha_k^{t-1},\,\boldsymbol{\theta}_{k}^{t-1}, \gamma_{k}^{t-1},\,
        \mathcal{D}_k,\,\mathcal{D}_k^{\text{val}})$\; \Comment*{if  participates} 
             \hspace{2in}{\textbf{or}} \\
          $\gamma_{k}^{t}, \zeta_k^{t} = $ 
       \textbf{ClientUpdate}$(\boldsymbol{\theta}^{t},\,\alpha_k^{t-1},\,\boldsymbol{\theta}_{k}^{t-1}, \gamma_{k}^{t-1},\,
        \mathcal{D}_k,\,\mathcal{D}_k^{\text{val}})$\;   \Comment*{if doesn't participate} 
    }
    $\boldsymbol{\theta}^{t+1}, \mathbf{s}^{t+1}, \mathbf{R}^{t+1} \gets$ 
    \textbf{SU}({$(\boldsymbol{\theta}_{k}^{\,t},\gamma_{k}^t,\zeta_k^t)_{k\in S_t},(\gamma_{k}^t,\zeta_k^t)_{k \in S_t^c } $})\;
}

\end{algorithm}

\paragraph{Explanation of FedUCA.}
Algorithm~\ref{alg:FedUAI_Server} describes the full \textbf{FedUCA} mechanism, an optimization-based aggregation strategy that explicitly models rational client behavior and aims to satisfy client utility constraints through server-side aggregation design.

The server maintains three evolving objects: the global model $\boldsymbol{\theta}^t$, a probability vector $\mathbf{s}^t$ governing the server’s signaling strategy, and a matrix $\mathbf{R}^t$ whose columns represent client-specific aggregation weight strategies. These quantities are initialized prior to training and updated iteratively over $T$ communication rounds.

At each round $t$, the server broadcasts its intent to all $N$ clients. For each client $k$, the server extracts the corresponding aggregation vector   $\mathbf{r}_k^{t-1}$ from the current matrix $\mathbf{R}^{t-1}$, computes the aggregation weight $\alpha_k^{t-1}$ and transmits to the client along with the current global model $\boldsymbol{\theta}^{t-1}$. This transmission constitutes the server’s signaling action.

Upon receiving the signal, each client executes the \textbf{ClientUpdate} procedure. In this step, client $k$ evaluates its  utility under the received aggregation weight using its validation data. Based on this evaluation, the client updates its local model and determines a utility threshold $\gamma_k^{t}$ reflecting its willingness to participate. The client then returns its updated local model and strategic information $\zeta_k^t$ to the server.

After collecting responses from all clients, the server invokes the Server Update \textbf{SU} routine. This step aggregates the local models while simultaneously sampling the distribution $\mathbf{s}^t$ and optimizing the aggregation weight matrix $\mathbf{R}^t$. The update is computed by solving a constrained optimization problem that balances model improvement against the probability of satisfying client participation constraints. As a result, the server refines both the global model and its future signaling strategy.

This iterative process continues for $T$ rounds. By jointly updating the global model and the aggregation signaling mechanism, FedUCA enables the server to strategically influence client participation while accounting for rational, utility-driven behavior. The final output $\boldsymbol{\theta}^T$ represents the learned global model under this federated learning framework.

\paragraph{A note on round indexing.}
In our implementation, FedUCA's first communication round ($t=0$) is a
seeding round: all clients participate without an individual-rationality
check, and the server simply aggregates the shared initialization, so no
utility curves are fit and no learning occurs in this round. In the
pseudocode this is abstracted by initializing the client models
$\boldsymbol{\theta}_k^{-1}$ and the server model $\boldsymbol{\theta}^0$ to
the same state, so the $t=0$ iteration leaves the model unchanged and the
rational machinery engages from $t=1$ onward. The rational FedAvg baseline,
by contrast, trains all clients in round~0 itself, so its first round is a
genuine training round. FedUCA therefore operates with one fewer effective
training round than this baseline.

\subsection{Rationality Check for benchmark methods}
\label{sup:bm_details}

Algorithms~\ref{fedavg_client} and~\ref{fedavg} detail the integration of rationality check into FedAvg. Extending the rationality check to other algorithms follows a similar approach and is relatively straightforward, as it closely resembles the implementation used for FedAvg. The employed rationality check operates as follows: if the accuracy of the received global model is higher than that of the client's local model, the client accepts the global model and proceeds with training initialized from it. If the rationality check is not satisfied the client proceeds with its own training. It is to be noted that one can use the negative of the loss as well instead of accuracy.

\begin{algorithm}[!h]
\caption{ FedAvg: \textit{Client Update}}
\label{fedavg_client}
\SetKwFunction{ClientUpdate}{ClientUpdate}
\SetKwProg{Fn}{Function}{:}{}
\Fn{\ClientUpdate{$\boldsymbol{\theta}^{t},\,\boldsymbol{\theta}_{m}^{t-1},\,\mathcal{D}_m,\,\mathcal{D}_m^{\text{val}},\,\text{client $m$}$}}{
   \BlankLine
    \
  \Comment{Rationality Check}
    \If{$(\mathbf{acc}(\boldsymbol{\theta}^{t},\mathcal{D}_m^{\text{val}})> \mathbf{acc}(\boldsymbol{\theta}_{m}^{t-1},\mathcal{D}_m^{\text{val}}))$} {
       
     $\boldsymbol{\theta}_{m}^{t,0} \gets {\boldsymbol{\theta}^{t}}$\\
      \For{$l \gets 1$ \KwTo $\tau$}{
        $\boldsymbol{\theta}_{m}^{t,l} \gets \boldsymbol{\theta}_{m}^{t,l-1} - \eta_c\nabla f_k\!\bigl(\boldsymbol{\theta}_{m}^{t,l-1};\xi_m^{t,l-1} \bigr)$\Comment*[f]{mini‑batch SGD step}
    }
     \KwRet $\boldsymbol{\theta}_{m}^{t,\tau}$\;
    } 
    \Else{
      $\boldsymbol{\theta}_{m}^{t,0} \gets \boldsymbol{\theta}_{m}^{t-1}$\\
      \For{$l \gets 1$ \KwTo $\tau$}{
        $\boldsymbol{\theta}_{m}^{t,l} \gets \boldsymbol{\theta}_{m}^{t,l-1} - \eta_c\nabla f_k\!\bigl(\boldsymbol{\theta}_{m}^{t,l-1};\xi_m^{t,l-1} \bigr)$\Comment*[f]{mini‑batch SGD step}
    }
        \KwRet \;
    }}


\end{algorithm}

\begin{algorithm}[!h]
\caption{FedAvg with Rationality}
\label{fedavg}
\KwIn{$\boldsymbol{\theta}^0,T,\nabla{f_{k}(\boldsymbol{\theta}_k^0)} = 0$} 
\KwOut{$\boldsymbol{\theta}^{T}$}

For $t = 0$ (round zero), initialize all client models to $\theta$ and all the clients train to form $\theta_k^0$, and participate in Federation and server aggregates $\theta^1 = \frac{1}{N} \sum_{k=1}^{N}\theta_k^0.$\\

\For{$t \gets 1$ \KwTo $T-1$}{
    Server broadcasts to all $N$ clients with indices $[N]$ \;
    \For{$m \gets 1$ \KwTo $N$}{
        Server sends $\boldsymbol{\theta}^{t}$ to client $m$\;
        $\boldsymbol{\theta}_{m}^{\,t} \gets$ \textbf{ClientUpdate}({$\boldsymbol{\theta}^{t},\,\boldsymbol{\theta}_{m}^{t-1},\,\mathcal{D}_m,\,\mathcal{D}_m^{\text{val}},\,\text{client $m$}$})\;
    }
    Update the participating set $S_t$.\\
    $\boldsymbol{\theta}_k^t \gets\boldsymbol{\theta}_k^{t-1}$, $\nabla{f_k(\boldsymbol{\theta}_k^t)} \gets \nabla{f_k(\boldsymbol{\theta}_k^{t-1})} $,  if $k \notin S_t,   \forall  k \in [N]$\;
        
    $\boldsymbol{\theta}^{t+1} \gets \frac{1}{N} \sum_{k}\boldsymbol{\theta}_k^t$   
}

\end{algorithm}

\subsection{Empirical Verification of Concavity of the Utility}
\label{app:emp_verify}

\begin{figure}[htbp]
    \centering
    \includegraphics[width=0.8\textwidth]{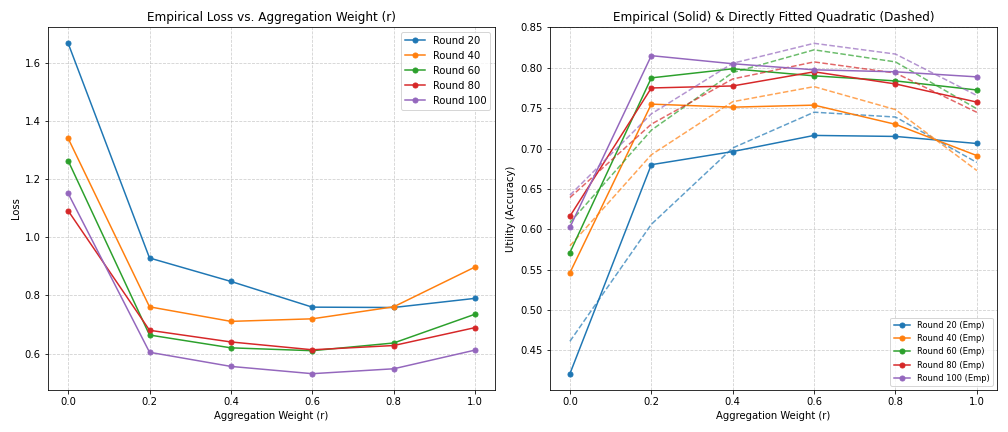}
    
    \vspace{10pt} 
    
    \includegraphics[width=0.8\textwidth]{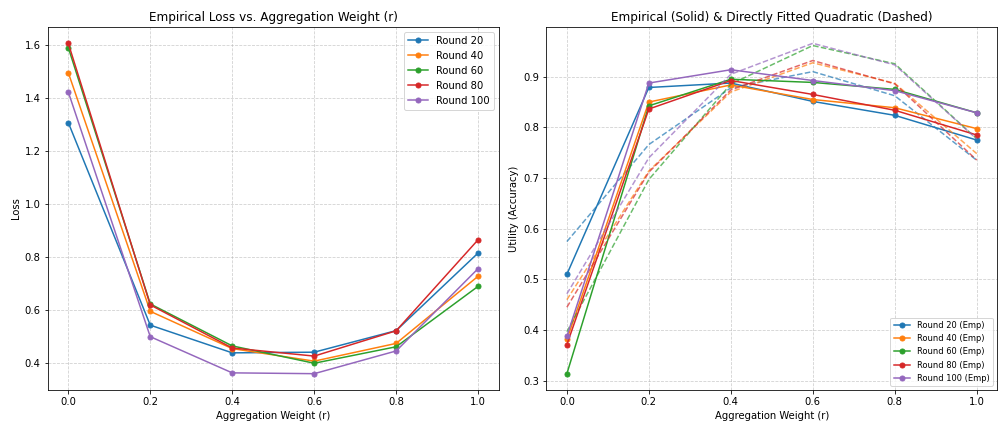}
    
    \caption{Representative empirical loss and utility profiles across two different clients, rounds, on the CIFAR-10 dataset. Across settings, loss curves exhibit convex-like behavior along the interpolation path, while utility curves remain smooth and unimodal. The dashed curves denote concave quadratic fits, supporting the surrogate utility model used by FedUCA.}
    \label{fig:client_dynamics_c10}
\end{figure}

\begin{figure}[htbp]
    \centering
    \includegraphics[width=0.8\textwidth]{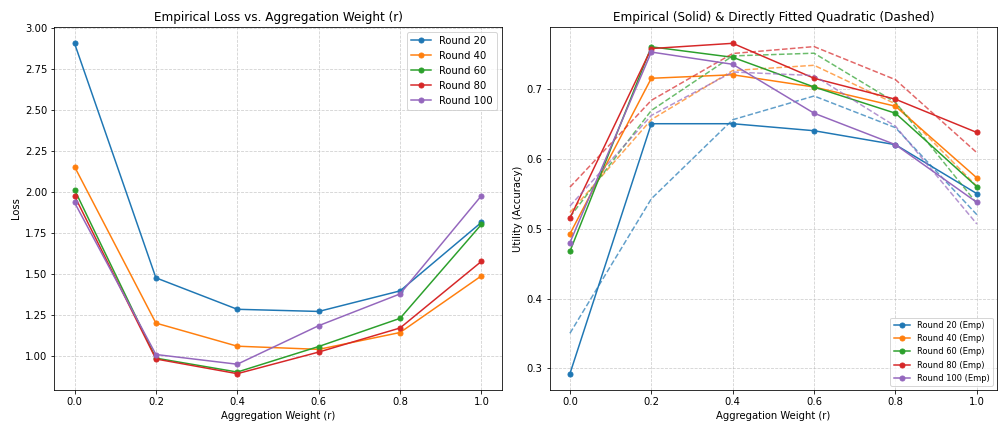}
    
    \vspace{10pt} 
    
    \includegraphics[width=0.8\textwidth]{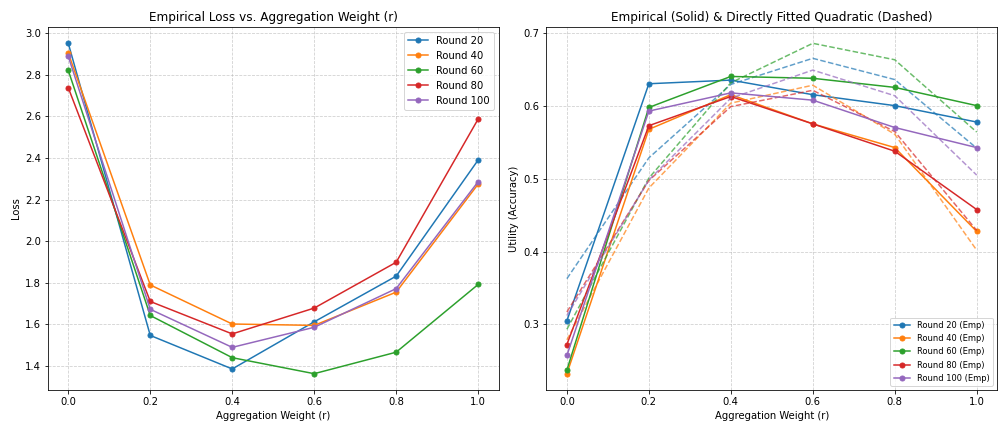}
    
    \caption{Representative empirical loss and utility profiles across two different clients, rounds, on the CIFAR-100 dataset. Across settings, loss curves exhibit convex-like behavior along the interpolation path, while utility curves remain smooth and unimodal. The dashed curves denote concave quadratic fits, supporting the surrogate utility model used by FedUCA.}
    \label{fig:client_dynamics_c100}
\end{figure}

\begin{table}[t]
\centering
\caption{Goodness-of-fit of concave quadratic surrogate for client utility across datasets and rounds.}
\label{tab:concavity_fit}
\begin{tabular}{lcc}
\toprule
Dataset & Avg.\ $R^2$ Score \\
\midrule
CIFAR-10  & 0.83  \\
CIFAR-100 & 0.82  \\
\bottomrule
\end{tabular}
\label{r2_fit_concave}
\end{table}

Table~\ref{r2_fit_concave} reports the goodness-of-fit of the concave quadratic surrogate $u_k(r; \zeta_k) = -a_k r^2 + b_k r + c_k$ against the empirical utility profile $\hat u_k(r)$ across clients and communication rounds. Average $R^2$ values of $0.83$ on CIFAR-10 and $0.82$ on CIFAR-100 confirm that the concave quadratic family provides a strong fit to empirical utility profiles in the federated deep learning setting, supporting the modeling assumption introduced in Section~\ref{est_client_ut}.

Figures~\ref{fig:client_dynamics_c10} and~\ref{fig:client_dynamics_c100} display representative empirical loss and utility profiles for two clients across multiple communication rounds. Across settings, loss curves exhibit convex-like behavior along the interpolation path while utility curves remain smooth, unimodal, and well-approximated by the concave quadratic fits (dashed). This confirms that the concavity assumption, standard in mechanism design, holds reasonably, as verified  empirically in this setting.

\subsection{Discussion on Utility Design}
\label{discuss_util_deisgn}

In our utility design, client $k$ removes its own contribution from the global model using the weight sent by the server, giving the held-out model
$\boldsymbol{\theta}_{-k}^t
= \frac{\boldsymbol{\theta}^t - \alpha_k^{t-1}\,\boldsymbol{\theta}_k^{t-1}}
       {1 - \alpha_k^{t-1}},$ and then, after local training, designs its utility $u_k(r)$ over the
interpolated model $r\,\boldsymbol{\theta}_k^{t} + (1-r)\,\boldsymbol{\theta}_{-k}^t$,
where $\boldsymbol{\theta}_k^{t}$ is its locally trained model. The endpoint
$r=1$ is the freshly trained local model, which also sets the rationality
threshold, and $r=0$ is the held-out rest. The curve does not pass through the
realized global $\boldsymbol{\theta}^t$: by construction
$\boldsymbol{\theta}^t = \alpha_k^{t-1}\boldsymbol{\theta}_k^{t-1}
+ (1-\alpha_k^{t-1})\boldsymbol{\theta}_{-k}^t$ lies on the line through the
previous-round model $\boldsymbol{\theta}_k^{t-1}$, whereas the curve is built
on the locally trained model $\boldsymbol{\theta}_k^{t}$. The check therefore
evaluates a forward-looking model at the server's weighting rather than the
model actually aggregated. 

The present design carries a one-round lag. The server aggregates using
clients' current-round models, but each client builds $u_k(r)$ from the
held-out global $\boldsymbol{\theta}_{-k}^t$, which is anchored to
$\boldsymbol{\theta}^t$ and the round $t-1$ client model
$\boldsymbol{\theta}_k^{t-1}$, so the fitted curve is mildly optimistic
relative to the realized global. On CIFAR-10 this appears as a small
reconcilable gap: the fitted-curve mechanism reaches $53\%$ participation and
$66\%$ accuracy, a strict check on the realized global accuracy gives $40\%$ participation and $58\%$ accuracy, and a $3\%$ slack on the realized check recovers the fitted-curve
numbers, with the gap shrinking as training proceeds. The accurate alternative
is a two-pass protocol per round, in which clients refit their utilities
against the current global without further local training and the server
re-aggregates, which removes the lag at the cost of roughly doubling
communication. Closing this one-round optimism while staying single-pass is the
open problem for a future utility design.

\subsection{Data Preparation}
\label{sec:data_prep}
For the correct evaluation of the persuasion capabilities of any FL algorithm, we propose that, for a cross-silo setup, the data distribution amongst clients should allow for some individual training capabilities of clients while limiting their scope. This would allow for clients to be "self-interested", which could result in a poor server model if clients do not choose to participate. Such a distribution also enables the server to meaningfully influence participation through aggregation design. In general, the data distribution should satisfy:
\begin{itemize}
\item Clients have the option to train a local model to a reasonable accuracy if the server model does not benefit the clients.
\item If the clients do choose to participate in the federation there are some benefits to gain from the federated training.
\item It models real-world settings where the distribution of train and validation data split may not be identical.
\item Clients should not participate by default regardless of the aggregation method, as can occur in nearly i.i.d. settings; otherwise, it becomes difficult to compare how effectively different FL algorithms sustain participation.
\end{itemize}

The notation $(0.5, 1.0)$ indicates that a Dirichlet concentration parameter of $0.5$ is used to generate the training data split, while a concentration of $1.0$ is used for the validation data split. We first divide the dataset into sets of classes assigned to different clients. Within each client, the training data is sampled using the $0.5$ concentration, and the validation data using the $1.0$ concentration. The resulting training and validation split for the $(0.5, 1.0)$ setting is illustrated in Figure \ref{fig:cifar_pt5_1heatmap}. A similar training and validation split for the $(1.0, 1.0)$ setting is illustrated in Figure \ref{fig:cifar_1_1_heatmap}.

\begin{figure}[htbp]
    \centering
    \includegraphics[width=0.75\textwidth]{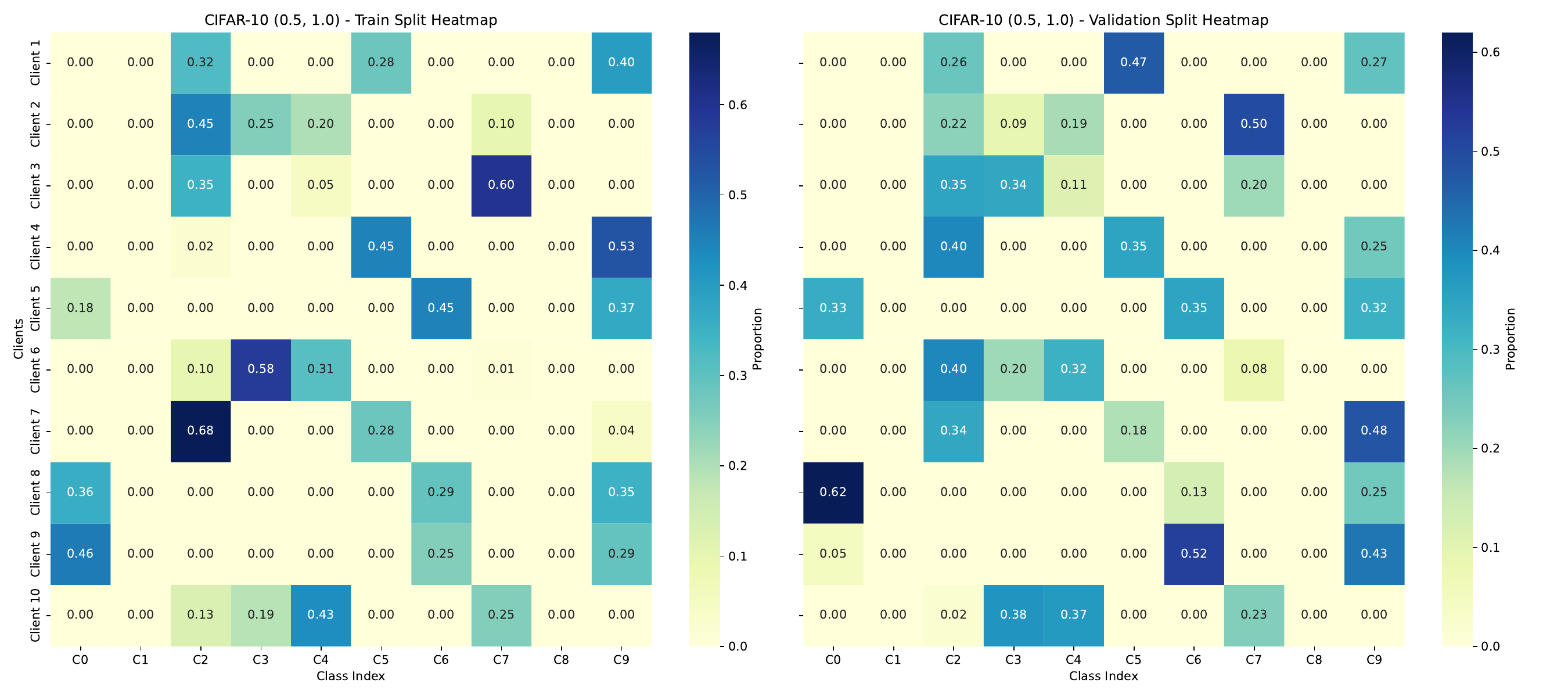} 
    \caption{Data distribution heatmap for CIFAR-10 (0.5, 1.0) showing Training and Validation splits across 10 clients.}
    \label{fig:cifar_pt5_1heatmap}
\end{figure}

\begin{figure}[htbp]
    \centering
    \includegraphics[width=0.75\textwidth]{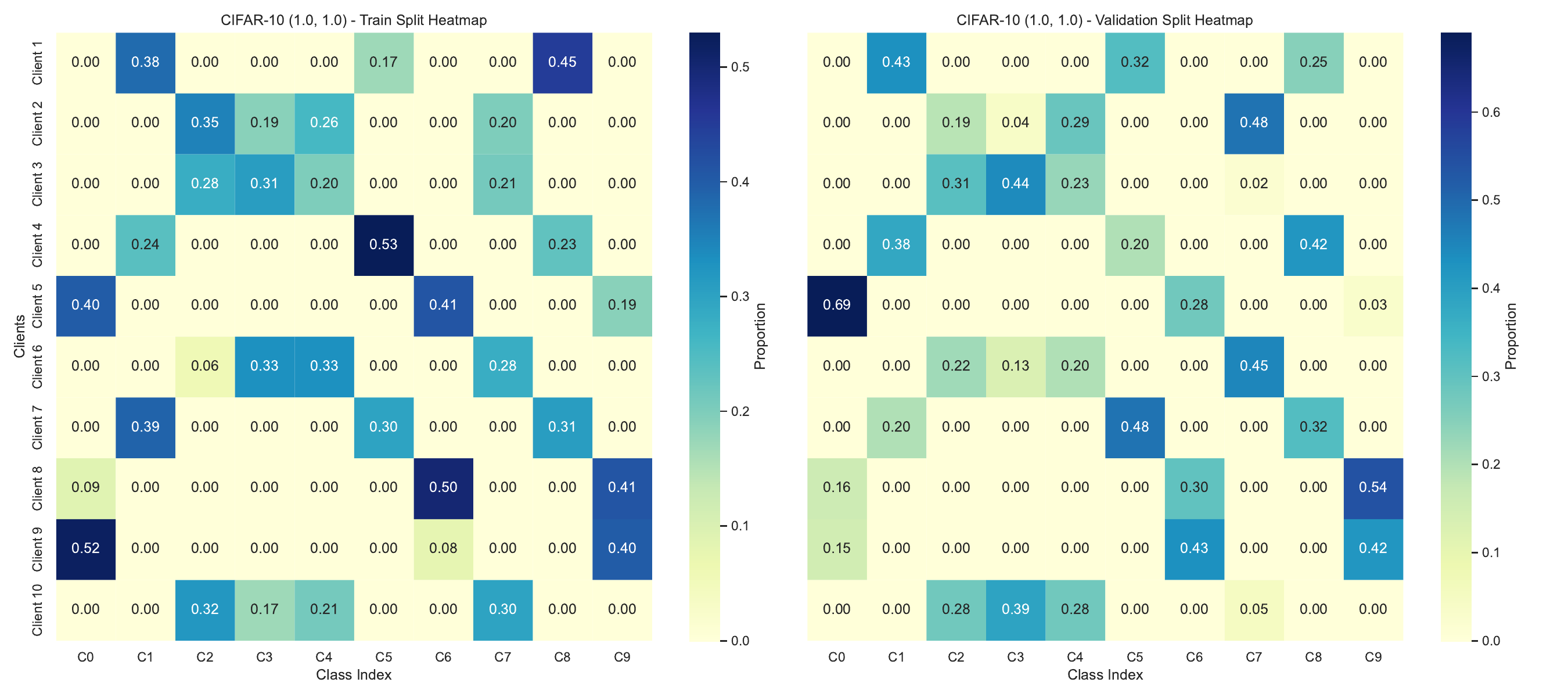} 
    \caption{Data distribution heatmap for CIFAR-10 (1.0, 1.0) showing Training and Validation splits across 10 clients.}
    \label{fig:cifar_1_1_heatmap}
\end{figure}


\subsection{F-MNIST and CIFAR-100}
For F-MNIST, we employ the same data distribution strategy as for CIFAR-10. For CIFAR-100, since we have 100 classes, we distribute 20 classes to each client in a similar fashion as CIFAR-10. 

\subsection{Experimental and Implementation Details}
\label{sup:hyper}
 We adopt hyperparameter settings similar to those used in FedDyn. We develop our experiments on the code released by FedDyn~\cite{acar2021federated}. Specifically, we set the number of local training epochs to $\tau = 5$. The local learning rate is $\eta_c = 0.1$, and the global learning rate is $\eta_s = 1.0$. For the CIFAR-10/100 and Fashion-MNIST (F-MNIST) datasets, we conduct training for a total of 100 communication rounds. We report the performance as the average accuracy over the final 10 rounds to ensure stability and reflect steady-state behavior. \\
We use mini-batch SGD as the optimizer for all our experiments. To account for variability in initialization, we repeat each experiment using three different random seeds. We report both the mean and standard deviation of the performance metrics across these three runs. \\

\subsection{HyperParameter Settings}
\label{hyper_details}
\begin{table}[htbp]
\centering
\caption{Experimental setup and hyperparameter configurations for the benchmark datasets.}
\label{tab:hyperparameters}
\begin{tabular}{@{}lccccc@{}}
\toprule
\textbf{Dataset} & \textbf{Clients ($N$)} & \textbf{Rounds ($T$)} & \textbf{Epochs ($\tau$)} & \textbf{Model Architecture} & \textbf{Optimizer ($\eta$)} \\ \midrule
CIFAR-10  & 10 & 100 & 5 & 5-layer CNN  & SGD (0.1) \\
F-MNIST   & 10 & 100 & 5 & 2-layer MLP & SGD (0.1)  \\
CIFAR-100 & 20 & 100 & 5 & ResNet-18   & SGD (0.1) \\ \bottomrule
\end{tabular}
\end{table}

Table \ref{tab:hyperparameters} details the foundational experimental setup and hyperparameter configurations across our benchmark datasets. To ensure a rigorous evaluation across varying task complexities, we scale the model capacity and federation size accordingly: utilizing a 2-layer MLP for FMNIST and a 5-layer CNN for CIFAR-10 in a 10-client cross-silo environment, while deploying a deeper ResNet-18 backbone across 20 clients for the more challenging CIFAR-100 dataset. To explicitly isolate the performance impact of our proposed aggregation mechanism, local optimization dynamics are held strictly constant across all environments, employing standard SGD (learning rate $\eta = 0.1$) for $\tau = 5$ local epochs over $T = 100$ global communication rounds. For FedUCA, we set $n_s=5$ and $\delta=0.5$ by default. The value of $\epsilon$ for $10$ clients is set to $0.1$ and for $20$ clients it is set to $0.05$.

\subsection{Experiments on Natural Language Data}
\label{20_news_group}
\begin{table}[htp]
\centering
\caption{Performance comparison on the NLP dataset (20 Newsgroups). We use Bi-directional LSTM as the model Architecture. We run for 50 communication rounds with 10 clients. FedUCA outperforms all baselines.}   
\label{tab:nlp_results}
\begin{tabular}{@{}lcc@{}}
\toprule
\textbf{Method} & \textbf{Participation Rate (\%)} & \textbf{Acc (\%)} \\ \midrule
IncFL   & 07.00 & $14.25_{\pm 0.20}$ \\
FedProx & 07.00 & $22.34_{\pm 0.40}$ \\
AAggFF-S   & 07.60 & $22.82_{\pm 0.30}$ \\
FedAvg  & 06.50 & $16.25_{\pm 0.18}$ \\
\textbf{FedUCA} & \textbf{66.0} & $\mathbf{48.33_{\pm 0.57}}$ \\ \bottomrule
\end{tabular}
\end{table}

Table~\ref{tab:nlp_results} evaluates FedUCA beyond vision tasks on the 20 Newsgroups NLP dataset using a bidirectional LSTM. FedUCA achieves $48.33\%$ server accuracy, outperforming the strongest standard baseline, AAggFF-S, by $25.51$ percentage points.


This shows that the utility-constrained aggregation continues to improve rational participation and global performance in an NLP setting. These results suggest that FedUCA's participation-management mechanism is not limited to vision datasets.

\subsection{Empirical Verification of Jensen's Gap}
\label{emp:jen_gap}

\begin{table}[htp]
\centering
\caption{Empirical verification of Proposition~\ref{prop:jgap}. We report the per-round total Jensen's gap $\sum_k \mathcal{J}_k^t$ measured over 100 communication rounds on CIFAR-10, summarized three ways: the mean over all rounds, the mean over the 50 rounds with the largest gap (Top-50), and the mean over the 25 rounds with the largest gap (Top-25). The Top-50 and Top-25 columns characterize the rounds in which the surplus mechanism is most active, typically rounds in which a substantial fraction of clients have binding IR constraints. The non-monotonic dependence on $\delta$ predicted by Proposition~\ref{prop:jgap} is observed across all three columns: the gap is smallest at extreme concentrations ($\delta = 0.01$ and $\delta = 100$) and peaks at moderate values $\delta \in [0.5, 1.0]$.} 
\label{tab:jensens_gap_distribution}
\begin{tabular}{@{}lccc@{}}
\toprule
\textbf{Dirichlet Concentration} ($\delta$) & \textbf{Mean} & \textbf{Top-50 Mean} & \textbf{Top-25 Mean} \\ \midrule
$0.01$ & 0.20 & 0.40 & 0.81 \\
$0.5$  & 2.48 & 4.10 & 5.32 \\
$1.0$  & 2.77 & 4.26 & 5.18 \\
$10$   & 1.19 & 1.60 & 1.82 \\
$100$  & 0.31 & 0.48 & 0.67 \\ \bottomrule
\end{tabular}
\end{table}

Table~\ref{tab:jensens_gap_distribution} reports the empirical Jensen's gap $\sum_k \mathcal{J}_k^t$ across communication rounds on CIFAR-10, for five Dirichlet concentrations spanning four orders of magnitude. We summarize the per-round gap three ways: the unconditional mean over all 100 rounds, and the mean over the top 50 and top 25 rounds. The latter two characterize the rounds in which the surplus mechanism is most active, which are typically the rounds in which the optimizer is working hardest to satisfy IR constraints for a heterogeneous set of clients.

Two observations align the data with Proposition~\ref{prop:jgap}. First, the dependence on $\delta$ is non-monotonic: the gap vanishes at $\delta = 0.01$ (where the Dirichlet collapses to a vertex) and at $\delta = 100$ (where it concentrates near the uniform vector), and peaks at $\delta \in [0.5, 1.0]$. Second, this peak location correlates with the participation maximum in Figure~\ref{fig:var_delta}, providing two independent measurements that confirm the theoretical prediction. The amplification across the Top-50 and Top-25 columns indicates that the surplus is concentrated in the rounds where the IR mechanism most needs to deliver, rather than spread uniformly across training.

\subsection{Empirical Verification of the Aggregation Discrepancy Bound}
\label{app:agg_verif}
\begin{figure}[htbp]
    \centering
    
    \begin{subfigure}{0.48\textwidth}
        \centering
        \includegraphics[width=\linewidth]{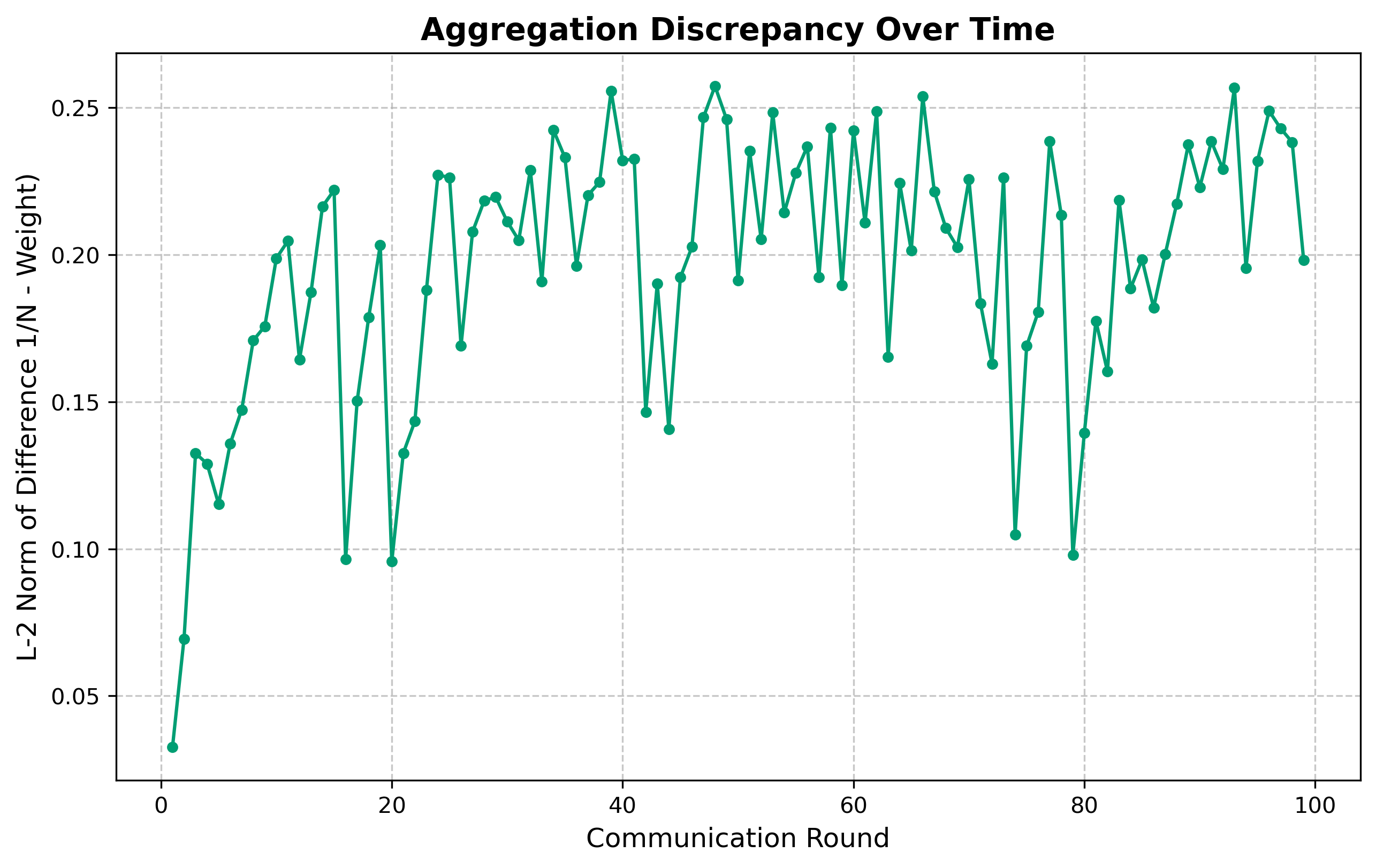}
        \caption{Results on CIFAR-10}
        \label{fig:cifar10}
    \end{subfigure}
    \hfill 
    \begin{subfigure}{0.48\textwidth}
        \centering
        \includegraphics[width=\linewidth]{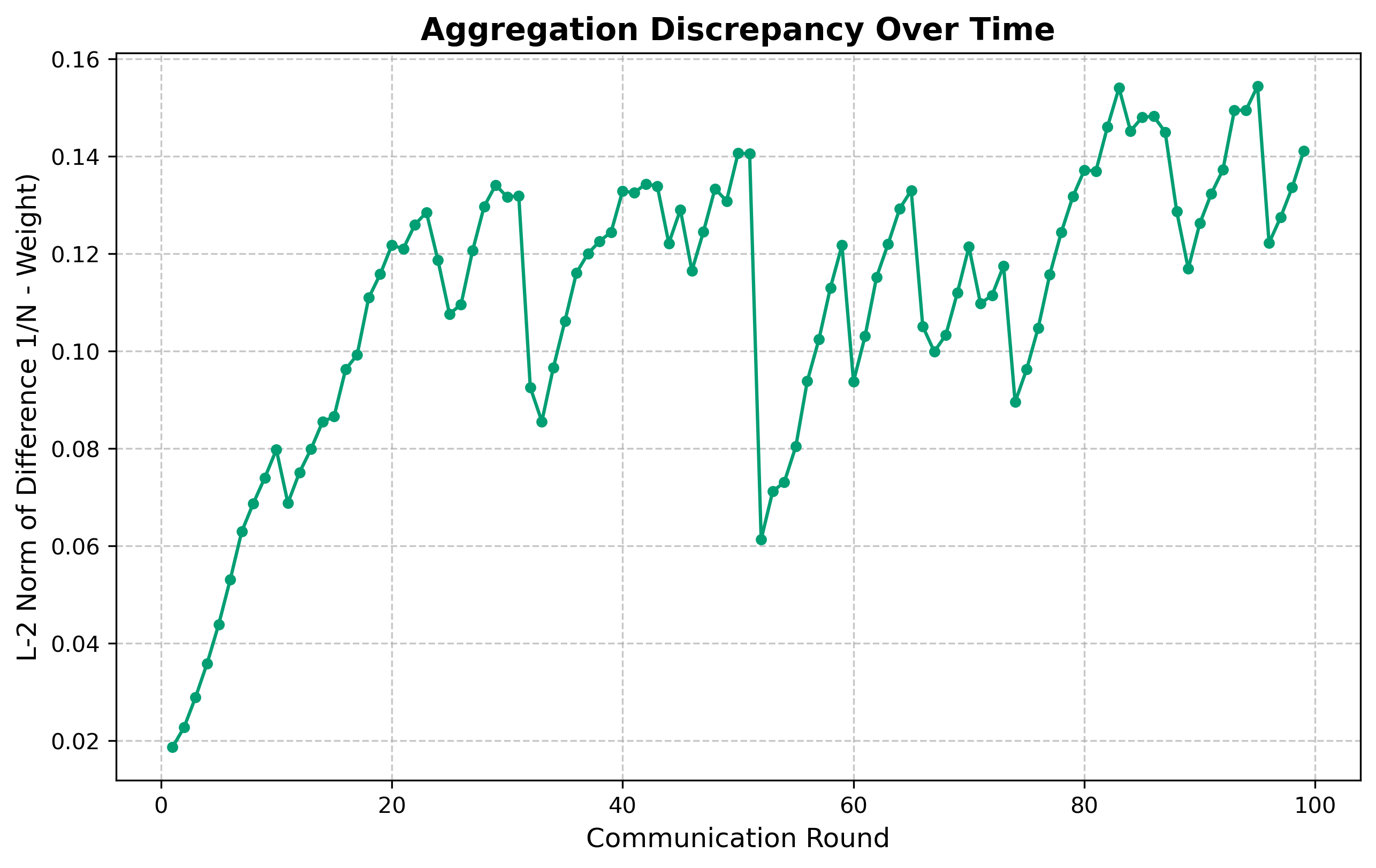}
        \caption{Results on CIFAR-100}
        \label{fig:cifar100}
    \end{subfigure}
    
 \caption{We plot the discrepancy $\|\mathbf{1}/N - \boldsymbol{\alpha}^t\|_2$ across communication rounds for FedUCA with $n_s = 5$ on (a) CIFAR-10 ($N=10$) and (b) CIFAR-100 ($N=20$). The gradual increase reflects growing heterogeneity in fitted client utilities $\zeta_k$ as local models diverge, which the IR mechanism translates into larger deviations from uniform aggregation; this is the intended behavior of utility-constrained aggregation, not a stability issue.}
    \label{fig:cifar_comparison_disc}
\end{figure}


Lemma~\ref{lem:as_agg_discrepancy} establishes the point-wise bound
$\|\beta - \alpha^t\|^2 \le (n_s-1)/N$, which holds almost surely at
every round and follows from the constraints of Optimization~\ref{main_opt}
alone, where $\beta = \mathbf{1}/N$. Figure~\ref{fig:cifar_comparison_disc} plots
the realized per-round discrepancy $\|\beta - \alpha^t\|_2$; taking square
roots gives the almost-sure ceiling
$\|\beta - \alpha^t\|_2 \le \sqrt{(n_s-1)/N}$. We track $\|\beta - \alpha^t\|_2$
empirically across all communication rounds for FedUCA with $n_s = 5$ on
CIFAR-10 ($N=10$) and CIFAR-100 ($N=20$); the empirical values stay well
below the ceiling throughout, peaking near $0.25$ against
$\sqrt{0.4}\approx 0.63$ on CIFAR-10 and near $0.15$ against
$\sqrt{0.2}\approx 0.45$ on CIFAR-100. 

Two additional observations are worth noting. First, the discrepancy increases in the early rounds and saturates after roughly round 30. This trajectory reflects the dynamics of the utility-constrained aggregation: in the initial rounds, fitted client utilities $\{\zeta_k\}$ are similar across clients, and the optimizer can satisfy IR constraints with weights close to uniform. As local models diverge under heterogeneous data, client utility curves become more heterogeneous and the optimizer assigns weights further from $\beta$ to maintain participation. The behavior is consistent with the IR mechanism functioning as intended-deviations from uniform aggregation are the means by which Jensen's-gap surplus is allocated to satisfy individual rationality. Second, CIFAR-100 exhibits smaller discrepancy than CIFAR-10 in absolute terms, consistent with the $1/\sqrt{N}$ dependence of the ceiling $\sqrt{(n_s-1)/N}$:
doubling the federation size scales the ceiling by $1/\sqrt{2}$, and the
empirical reduction roughly tracks this prediction.


\subsection{Impact due to Inflated Thresholds}
\label{app:misreporting}

\begin{table}[h]
\centering
\caption{Impact of strategic threshold inflation on FedUCA (CIFAR-10, Dir(0.5, 1.0)). Universal inflation of the individual rationality threshold $\gamma_k^t$ shrinks the server's feasibility space, leading to participation collapse and severe accuracy degradation.}
\label{tab:threshold_ablation}
\begin{tabular}{@{}lcc@{}}
\toprule
\textbf{Threshold Inflation} & \textbf{Participation (\%)} & \textbf{Accuracy (\%)} \\ \midrule
$0\%$ (truthful)             & 53.0                        & $66.89_ {\pm 1.56}$       \\
$+5\%$                       & 40.0                        & $63.90_{ \pm 1.23}$       \\
$+10\%$                      & 1.3                         & $10.80_{\pm 0.01}$       \\ \bottomrule
\end{tabular}
\end{table}

In the table~\ref{tab:threshold_ablation}, we evaluate FedUCA's robustness against strategic misreporting by simulating clients who attempt to free-ride by inflating their reported thresholds $\gamma_k$ by $5\%$ and $10\%$. We model universal inflation: under common knowledge that an incentive to inflate available to one client is available to all, so the rational analysis must consider the deviation in which every client inflates.

Inflated thresholds shrink the server's feasibility space (Optimization~\ref{main_opt}): the server cannot satisfy clients whose demands exceed what the aggregation can deliver, and drops them from the participating cohort. At $+5\%$ inflation, participation falls from $53\%$ to $40\%$ and global accuracy degrades by roughly $3$ points. At $+10\%$, the optimization becomes broadly infeasible, participation collapses to $1.3\%$, and accuracy degrades to random guessing ($10.80\%$).

If any client believes inflation is profitable, every rational client believes the same and inflates simultaneously, but universal inflation triggers federation-wide collapse, in which no client receives utility. Importantly, we do not claim truthful reporting as a dominant strategy. Our framework explicitly assumes truthful reporting, and this outcome is purely an observation from this specific experiment we ran, in which clients report and apply inflated thresholds in their own participation check. Within that setting, inflation degrades feasibility and participation. This finding holds strictly for this tested scenario, and we claim nothing beyond it. Establishing formal incentive compatibility remains a direction for future work.

\subsection{Computational Complexity}
\label{sec:comp_cmplx}
All of our experiments are run on NVIDIA-RTX A5000 GPUs. For CIFAR-10, with 10 clients trained sequentially, the average time of execution for FedUCA is 5575s (the time spent in optimizer is 6s) for 100 communication rounds,  while for FedAvg we have an execution time of 4583s. The peak memory utilization for FedUCA is 5654 MB, while for FedAvg we get a maximum utilization of 5633 MB. For FedProx, we have an execution time of 4591s with a peak utilization of 5634 MB. IncFL has a peak memory utilization of 5640 MB with an average execution time of 4718s. For AAggFF-S we have an average execution time of 4761s with a peak memory utilization of 5640 MB.\\

Note: These wall-clock numbers reflect our current implementation. In particular, FedUCA, which is our method, runs substantially more logging relative to the baselines, potentially inflating its measured time; the numbers should therefore be treated as an upper bound on FedUCA's cost. The method's only intrinsic overhead over a rationality-check-equipped baseline is (i)  validation forward passes for utility estimation, and (ii) the server-side aggregation solve.

On the client-side, the local model training incurs $n_{epoch} \cdot |D_k^{train}| \cdot (c_f + c_b)$, $|D_k^{train}|$ denotes the number of training points, $c_f, c_b$ denotes the cost of forward and backward passes.

FedUCA incurs extra computation only for utility estimation. If $G_r$ denotes model interpolations, we need $G_r$ evaluations of $\hat{u}_k(r)$ on the validation set at $G_r \cdot |D_k^{val}| \cdot c_f$ and a least-squares fit of three parameters. ($|D_k^{val}|$ denotes validation data points).

Empirically, on the CIFAR-10 data set, with $n_{epoch} = 5$, 5000 training points, and 800 validation points on A5000-GPU, this takes on average $\sim$0.5s time while the local training takes $\sim$5s of time, per communication round. Compared to FedAvg, it takes $\sim$0.5s extra time to fit the utility curve.

On the server side, FedUCA solves the Optimization \eqref{main_opt}. This addition is a convex program in $n_s \cdot N$ variables, independent of the model dimension $d$. For $N=10$, $n_s = 5$, it totals 0.06 s per round, which is negligible.

On the impact of the number of clients ($N$) and strategies $n_s$, it should be noted that client-side utility computation is independent of them; only server optimization depends on them. We quantify it in Table \ref{runtime_table}.

\begin{table}[h!]
\caption{Run-time of the server Optimization as a function of $N$ and $n_s$}
\centering
\renewcommand{\arraystretch}{1.2}
\begin{tabular}{l ccccc}
\toprule
\textbf{Strategies ($n_s$)} & \textbf{N=10} & \textbf{N=20} & \textbf{N=40} & \textbf{N=80} & \textbf{N=160} \\
\hline
\multicolumn{1}{c}{5}  & 6.0 & 10.0 & 18.3 & 36.0 & 72.3 \\
\hline
\multicolumn{1}{c}{10} & 6.2 & 10.4 & 19.4 & 40.0 & 79.7 \\
\hline
\multicolumn{1}{c}{20} & 6.5 & 11.5 & 22.0 & 45.0 & 92.0 \\
\hline
\multicolumn{1}{c}{30} & 7.1 & 12.4 & 25.0 & 51.0 & 101.0 \\
\hline
\multicolumn{1}{c}{40} & 7.3 & 13.6 & 28.0 & 58.0 & 123.0 \\
\hline
\end{tabular}
\label{runtime_table}
\end{table}

Column header = Clients ($N$). Row header = Strategies ($n_s$). Entries = time per-round in seconds $\times$ 100.

Linear in clients: doubling $N$ roughly doubles time across all strategy counts (6.0 $\rightarrow$ 10 $\rightarrow$ 18.3 $\rightarrow$ 36 $\rightarrow$ 72.3). Cost scales as $\mathcal{O}(N)$.

Sublinear in strategies ($n_s$): 8$\times$ more strategies ($n_s = 5 \rightarrow 40$) adds only minor overhead depending on $N$ (at $N=10$, the time goes from 6.0 s to 7.3s when $n_s$ goes from 5 to 40, for $100$ rounds).

\subsection{On the rationality Check and alternate thresholding strategies}
\label{rc_ts}
\paragraph{Relaxed Rationality check}
Clients may adopt a relaxed form of the threshold, replacing the rationality
check in Eq.~\ref{RC} by
$acc(\boldsymbol{\theta}^{t},\mathcal{D}_k^{\text{val}}) + m_k >
 acc(\boldsymbol{\theta}_{k}^{t-1},\mathcal{D}_k^{\text{val}})$
and the IR constraint in Eq.~\ref{eq:ir_constraint} by
$u_k(r_k;\zeta_k) + m_k > \gamma_k$, where $m_k \ge 0$ is a client-chosen
safety margin. We report $m_k = 0$ throughout, which is the strictest setting:
any $m_k > 0$ only relaxes the bar, and therefore it cannot decrease participation.
The optimization meets the constraints at a lower value of the slack variable
$t_k$, so clients clear the threshold with a smaller required $m_k$, and in
particular at $m_k = 0$.

\paragraph{Alternate thresholding strategies}
To evaluate the relative benefits of FedUCA under an alternate threshold
choice, we re-run all methods under a single shared threshold. FedUCA's native threshold is the accuracy read from the updated
utility curve at the current aggregation weight, which the baselines cannot
form since they lack the aggregation-weight channel ($\alpha_k^{t-1}$ and
$u_k(r)$). We therefore use a channel-free threshold common to every method, a
convex combination of the local and global validation accuracies: a client
participates in the next round if the assigned model clears
$(1-l_w)\, acc(\boldsymbol{\theta}^{t},\mathcal{D}_k^{\text{val}})
+ l_w\, acc(\boldsymbol{\theta}_{k}^{t},\mathcal{D}_k^{\text{val}})$
with $l_w = 0.4$, and each baseline uses the analogous check at round $t+1$,
$acc(\boldsymbol{\theta}^{t+1},\mathcal{D}_k^{\text{val}}) >
(1-l_w)\, acc(\boldsymbol{\theta}^{t},\mathcal{D}_k^{\text{val}})
+ l_w\, acc(\boldsymbol{\theta}_{k}^{t},\mathcal{D}_k^{\text{val}})$.
Under this shared threshold the baselines still collapse to a few percent
participation while FedUCA sustains $33.3\%$ in Dirichlet($1,1$) setting (Table~\ref{tab:results_acc_std}).

\begin{table}[h]
\centering
\caption{Average accuracy (standard deviation as subscript) and average
participation across 3 seeds under the shared convex-combination threshold
($l_w = 0.4$).}
\label{tab:results_acc_std}
\begin{tabular}{lcccc}
\toprule
\multirow{2}{*}{Method} & \multicolumn{2}{c}{CIFAR-10 Dirichlet($0.5,1$)} & \multicolumn{2}{c}{CIFAR-10 Dirichlet($1,1$)} \\
\cmidrule(lr){2-3} \cmidrule(lr){4-5}
 & Accuracy (\%) & Participation (\%) & Accuracy (\%) & Participation (\%) \\
\midrule
FedAvg       & $25.73_{1.19}$ & 26 & $16.33_{1.27}$ & 2 \\
FedProx      & $22.40_{1.97}$ & 23 & $12.62_{4.49}$ & 1 \\
IncFL        & $24.95_{0.43}$ & 28 & $19.42_{1.55}$ & 2 \\
AAggFF-S     & $27.38_{3.31}$ & 28 & $16.50_{6.96}$ & 1 \\
FedAvg-WS    & $47.13_{1.03}$ & 43 & $47.65_{1.08}$ & 5 \\
FedUCA (Ours)& $\mathbf{48.08}_{3.60}$ & 44 & $\mathbf{66.95}_{2.68}$ & 33 \\
\bottomrule
\end{tabular}
\end{table}

\subsection{Discussion on Comparison with Monetary Incentives}

Monetary mechanisms are not applicable baselines in our setting, which assumes no payment channel; participation costs are performance-based, not financial. We nonetheless quantify the comparison analytically.

Both schemes deal with an individual-rationality constraint, but in different currencies: FedUCA generates the needed slack internally as Jensen surplus (from utility concavity at no cost), whereas a monetary mechanism supplies it externally as payment. Neither framework fixes an exchange rate between accuracy-utility and money. What is well posed is the price at which the monetary route buys the same slack/participation.

\textit{Setup.} We use the standard cross-silo reward-budget mechanism of~\cite{ijcai2023p498}: a reverse auction with even allocation (per round budget) $B_t = B_{total}/T$, clients ranked by quality-to-bid ratio, admitted down the ranking until the round budget is exhausted, and winners paid the critical price. We adopt their equal-quality configuration with per-round bids i.i.d. $U(a,b)$.

\textit{Cost is quadratic in participation.}

For a fixed $n$, with equal quality, ranking by quality-to-bid coincides with ranking by bid, so the critical price is $b_{(n+1)}$ and the round cost of admitting $n$ clients is $C(n) = n \cdot b_{(n+1)}$. For $U(a,b)$, $E[b_{(k)}] = a + (b - a)k/(N + 1)$, giving

$$
E[C(n)] = n \cdot [a + (b - a)(n + 1)/(N + 1)]
$$

quadratic in $n$ with positive leading coefficient $(b - a)/(N + 1)$: the marginal price of each additional client strictly increases. For $N=10$, $T=100$, bids $U(0.5,1.5)$, matching FedUCA's 53\% participation costs a total budget of roughly 520-680, and reaching 90\% costs roughly 1270, so 1.8x the participation requires 2.4x the budget. The table below captures this in detail.

\begin{table}[h!]
\caption{Monetary budget required to buy the Participation}
\centering
\renewcommand{\arraystretch}{1.2}
\begin{tabular}{lcc}
{Participation} & {n} & {Total budget for 100 rounds} \\
\hline
\textbf{10\%} & 1 & 68 \\
\hline
\textbf{30\%} & 3 & 259 \\
\hline
\textbf{50\%} & 5 & 523 \\
\hline
\textbf{60\%} & 6 & 682 \\
\hline
\textbf{80\%} & 8 & 1055 \\
\hline
\textbf{90\%} & 9 & 1268 \\
\hline
\end{tabular}
\end{table}

\textit{Full participation is excluded by design:}

The critical price is set by the first excluded client, so a losing bidder must always remain; the mechanism saturates at $(N-1)/N$ (90\% for $N=10$) at any budget. Purchased participation is capped structurally, not merely by spend.

\subsection{Additional Experiments}

\subsubsection{Comparison with Personalized baselines}

\begin{table}[h!]
\caption{Performance comparison of personalized baselines with Dirichlet(0.5,1.0).}
\centering
\renewcommand{\arraystretch}{1.2}
\begin{tabular}{lcccc}
 & \multicolumn{2}{c}{\textbf{CIFAR-10}} & \multicolumn{2}{c}{\textbf{FMNIST}} \\
\cline{2-3} \cline{4-5}
{ Method} & Accuracy & {Participation (\%)} & Accuracy & {Participation (\%)} \\
\hline
{pFedMe}~\cite{t2020personalized} & $22.75_{\pm 0.04}$ & 11 & $47.30_{ \pm 0.20}$ & 42 \\
\hline
{Ditto}~\cite{li2021ditto} & $27.45_{\pm 0.06}$ & 8 & $56.75_{\pm 0.30}$ & 37 \\
\hline
{Per-Fedavg}~\cite{NEURIPS2020_24389bfe} & $19.25_{\pm 0.03}$ & 1 & $36.15_{\pm 0.02}$ & 1 \\
\hline
{FedUCA} & $66.89_{\pm1.56}$ & 53 & $83.18_{\pm 0.40}$ & 53 \\
\hline
\end{tabular}

\label{tab:merged_personalized_baselines}
\end{table}

In Table~\ref{tab:merged_personalized_baselines}, we compare personalized baselines under the identical rationality check and protocol applied to all baselines. On CIFAR-10, all three personalized methods sustain low shared-model participation (Per-FedAvg 1\%, Ditto 8\%, pFedMe 11\%) and correspondingly low shared-model accuracy (19 to 27\%), versus FedUCA's 53\% and 66.89\%. On the easier F-MNIST, the effect is milder (Ditto 37\%, pFedMe 42\%) but still below FedUCA (53\%, 83.18\%). This is not a deficiency of these methods but a consequence of their objective: each client's personalized model raises its acceptance bar for the global model, so clients reject the shared model and it receives too few updates to train; this is consistent with the participation collapse in Figure 1.

\subsubsection{Experiments with cubic utility fit and loss as utility}

\begin{table}[htbp]

\centering
\caption{Performance on the CIFAR-10 dataset with Dirichlet partitioning ($ 0.5, 1.0$). Results are reported as $\text{mean}_{\text{std}}$.}
\vspace{0.2cm}
\begin{tabular}{lc}
\toprule
\textbf{Method} & \textbf{Accuracy (\%)} \\
\midrule
Loss as utility and Quadratic fit & $66.10_{1.30}$ \\
Accuracy as utility and Cubic Fit & $66.50_{1.10}$ \\
Accuracy as utility and Quadratic Fit (Default) & $66.89_{1.56}$ \\
\bottomrule
\end{tabular}
\label{tab:merged_accuracy}
\end{table}

Regarding admissible metrics: not every metric is admissible, only those whose measured curve is concave in the interpolation weight $r \in [0, 1]$. The reasoning stems from the interpolation path itself. At $r = 0$, the client's own model is excluded from the aggregate, and at $r = 1$, the rest of the federation is excluded. Because both endpoints discard critical information, the curve naturally peaks at an interior point for any metric that reflects overall model quality. Validation accuracy satisfies this concavity condition, as shown by our plots in Section~\ref{app:emp_verify}, and negative validation loss works similarly well. To further validate our approach, we performed an experiment fitting a 3rd-order (cubic) polynomial. Because the empirical data exhibits diminishing returns, even the 3rd-order polynomial behaves concavely within the region of interest $[0, 1]$. As verified in the table~\ref{tab:merged_accuracy}, the overall performance remains largely unaffected whether we use negative validation loss as the utility metric or apply a cubic fit.

\subsection{Proofs of Propositions}
\begin{proposition}
Suppose $\epsilon = 0$, and all utilities $u_k(.)$ are concave then the solution to optimization~\ref{main_opt} satisfies either 
$\sum_{i=1}^{n_s}{ s_i{u_k(r_k^i)}}  =  \gamma_k -t_k$,  with $t_k > 0$ or $\sum_{i=1}^{n_s}{ s_i{u_k(r_k^i)}}  \geq  \gamma_k$, with $t_k = 0$. In the former case, the constraint is satisfied when $u_k(r_k^i) = \frac{s_i}{\sum_{j=1}^{n_s}s_j^2} (\gamma_k - t_k)$. If we further assume that $u_k(.)$  is strictly concave, the optimal solution $\mathbf{R}^*, \mathbf{t}^*$ is unique, if the former holds $\forall k \in [N] $.
\label{app:prop_kkt}
\end{proposition}

\begin{proof}

We form the lagrangian $L(\mathbf{t},\textbf{R})$ as 
\begin{align}
L(\mathbf{t},\textbf{R}, \boldsymbol{\lambda_1},\boldsymbol{\lambda_2},\beta_k^i,\xi_k,\mu_k ) &= \frac{1}{2}{\Vert \mathbf{t} \rVert}_2^2 + 
\boldsymbol{\lambda_1}^T(\mathbf{R}\mathbf{1}_{N} -\mathbf{1}_{n_s})+\boldsymbol{\lambda_2}^T(\mathbf{R}^T\mathbf{1}_{n_s}- \frac{n_s}{N}\mathbf{1}_{N}) \\ \nonumber
&- \sum_{i=1,k=1}^{n_s,N} \beta_{k}^{i} r_k^i  - \sum_{k=1}^{N} \xi_k(\sum_{i=1}^{n}{ s_i{u_k(r_k^i)}}  -   \gamma_k +t_k)
-\sum_{k=1}^{N}\mu_k t_k
\end{align}


by setting $\nabla_\mathbf{t}L = 0$ we have the following 

\begin{equation}
t_k = \mu_k + \xi_k
\label{grad_eq1}
\end{equation}

by setting $\nabla_{r_k^i}L = 0$ we have the following,  $\forall i \in [n_s], \forall k \in [N]$
\begin{equation}
\boldsymbol{\lambda}_1(i) + \boldsymbol{\lambda}_2(k) =\beta_k^i + \xi_{k}s_i \nabla_{r_k^i}{u_k(r_k^i)}
\label{grad_eq2}
\end{equation}

By the feasibility of equality constraints, we have
\begin{equation}
\mathbf{R}\mathbf{1}_{N} = \mathbf{1}_{n_s}
\end{equation}

\begin{equation}
\mathbf{R}^T\mathbf{1}_{n_s} =  \frac{n_s}{N} \mathbf{1}_{N}
\end{equation}


By the feasibility of inequality constraints, we have $\forall k \in [N]$

\begin{equation}
\sum_{i=1}^{n}{ s_i{u_k(r_k^i)}}  \geq  \gamma_k -t_k
\end{equation}

By the complementary slackness of the inequality constraints, we have $\forall k \in [N]$

\begin{equation}
(\sum_{i=1}^{n}{ s_i{u_k(r_k^i)}}  -\gamma_k +t_k)\xi_k = 0
\label{slack1}
\end{equation}

\begin{equation}
\mu_k t_k = 0
\label{slack2}
\end{equation}

And  $\forall k \in [N]$, $\forall i \in [n_s]$

\begin{equation}
 r_k^i \beta_k^i = 0
 \label{slack3}
\end{equation}

From Eq.~\ref{slack2} we see that when $t_k > 0$ the Eq.~\ref{grad_eq1} will simplify to $t_k = \xi_k$. Since we have assumed $t_k > 0$ and so is $\xi_k$.

But the slackness assumption in Eq.~\ref{slack1} implies that 

\begin{equation}
\sum_{i=1}^{n}{ s_i{u_k(r_k^i)}}  = \gamma_k -t_k
\label{eq_may_sat}
\end{equation}

If $\xi_k = 0$, then this will imply from the condition in.~\ref{slack1} we have the following 

\begin{equation}
\sum_{i=1}^{n}{ s_i{u_k(r_k^i)}}  \geq \gamma_k -t_k
\label{eq_sat}
\end{equation}

But From the condition in.~\ref{grad_eq1} we have $t_k = \mu_k$ but the slackness also implies that $t_k\mu_k = 0$. This implies that $t_k = \mu_k = \xi_k = 0$.

The Eq.~\ref{eq_sat} simplifies to the following 

\begin{equation}
\sum_{i=1}^{n}{ s_i{u_k(r_k^i)}}  \geq \gamma_k
\label{eq_sat1}
\end{equation}

From Eq.~\ref{eq_sat1} and.~\ref{eq_may_sat} we conclude the first part of the proof.

The second part of the proof can be established by writing as system of equations with unknowns in $u_k(r_k^i)$. It is  easy to verify that the condition $u_k(r_k^i) = \frac{s_i}{\sum_{j=1}^{n_s}s_j^2}(\gamma_k - t_k)$ satisfies the condition $\sum_{i=1}^{n}{ s_i{u_k(r_k^i)}}  =  \gamma_k -t_k$. 

Finally, we can prove this uniqueness by contradiction. Suppose there are two distinct optimal strategy matrices, $\mathbf{R}_1^*$ and $\mathbf{R}_2^*$, that yield the same optimal (minimal) slack $\boldsymbol{t}^*$. Owing to the strict concavity of the utility function $u_k$ considered for a client $k$, let us form a convex combination 
\begin{equation*}
\mathbf{R} = \alpha \mathbf{R}_1^* + (1-\alpha) \mathbf{R}_2^*,  \forall \alpha \in (0, 1).
\end{equation*}
Evaluating the utility constraint under this mixed strategy yields:

\begin{equation*}
\sum_{i=1}^{n_s}s_i u_k(r_k^i) > \alpha \sum_{i=1}^{n_s}s_i u_k(r_{1,k}^{*i}) + (1-\alpha) \sum_{i=1}^{n_s}s_i u_k(r_{2,k}^{*i})
\end{equation*}


Finally we have 

\begin{equation*}
\sum_{i=1}^{n_s}s_i u_k(r_k^i) > \alpha (\gamma_k - t_k^*) + (1-\alpha) (\gamma_k -t_k^*) = \gamma_k - t_k^*
\end{equation*}


This demonstrates that under the combined strategy $\mathbf{R}$, the utility strictly exceeds the required threshold bounded by $t_k^*$. Consequently, the slack $t_k^*$ can be further reduced while still satisfying the individual rationality constraint. This directly contradicts the premise that $\boldsymbol{t}^*$ is the optimal (minimal) slack. Thus, the strategy solution $r_k^i$ generated by the server must be unique.

\end{proof}

\begin{proposition}
Let the expected Jensen's gap for client $k$ be defined as $\bar{\mathcal{J}}_k(\delta) \coloneqq \mathbb{E}_{\mathbf{s}\sim \text{Dir}(\delta)} [\mathcal{J}_k(\mathbf{s})]$ for a concentration parameter $\delta > 0$. The instantaneous gap is given by $\mathcal{J}_k(\mathbf{s}) \coloneqq u_k\big(\sum_{i=1}^{n_s} s_i r_k^i\big) - \sum_{i=1}^{n_s} s_i u_k(r_k^i)$, where $u_k(\cdot)$ is a strictly concave utility function and the aggregation weights $r_k^i$ are obtained by solving Optimization~\ref{main_opt}, where the constraint is binding. The expected gap $\bar{\mathcal{J}}_k(\delta)$ vanishes at the extreme concentration regimes, specifically when $\delta \le \delta_{min}$ or $\delta \ge \delta_{max}$ (where $\delta_{min} \ll 1$ and $\delta_{max} \gg 1$). Furthermore, $\bar{\mathcal{J}}_k(\delta)$ is a continuous function of $\delta$ on the open interval $(\delta_{min}, \delta_{max})$.
\label{app:jgap_dct}
\end{proposition} 


\begin{proof}
We establish that $\bar{\mathcal{J}}_k(\delta)$ is well-defined and continuous on $(\delta_{\min},\delta_{\max})$.

The aggregation matrix $\mathbf{R}^*(\mathbf{s})$ is obtained from the parametric optimization problem in~\eqref{main_opt}. The feasible set $\mathcal{R}$ is a non-empty compact polytope defined by linear equality and inequality constraints, and is independent of $\mathbf{s}\in\Delta$.

Define
\[
V(\mathbf{R},\mathbf{s})
=
\frac{1}{2}\|\mathbf{t}(\mathbf{R},\mathbf{s})\|_2^2,
\qquad
t_k(\mathbf{R},\mathbf{s})
=
\max\Big(0,\gamma_k-\sum_{i=1}^{n_s}s_i u_k(r_k^i)\Big).
\]
For fixed $\mathbf{s}$, $V(\mathbf{R},\mathbf{s})$ is continuous in $\mathbf{R}$ because $u_k$ is continuous and the max operator is continuous. For fixed $\mathbf{R}$, $V(\mathbf{R},\mathbf{s})$ is continuous, hence measurable, in $\mathbf{s}$. Thus, $V$ is a Carathéodory function. Since the feasible correspondence is constant, compact-valued, and measurable, the Measurable Maximum Theorem. (Refer to Theorem 18.19 of~\citep{aliprantis1999infinite})
implies that the argmin correspondence admits a Lebesgue-measurable selection $\mathbf{R}^*(\mathbf{s})$.

Each coordinate $r_k^i(\mathbf{s})$ is a coordinate projection of $\mathbf{R}^*(\mathbf{s})$. Since coordinate projections are continuous, each $r_k^i(\mathbf{s})$ is measurable. Therefore, because $u_k$ is continuous, the mappings $u_k(r_k^i(\mathbf{s}))$ and
\[
u_k\!\left(\sum_{i=1}^{n_s}s_i r_k^i(\mathbf{s})\right)
\]
are measurable. Hence,
\[
\mathcal{J}_k(\mathbf{s})
=
u_k\!\left(\sum_{i=1}^{n_s}s_i r_k^i(\mathbf{s})\right)
-
\sum_{i=1}^{n_s}s_i u_k(r_k^i(\mathbf{s}))
\]
is measurable.

Moreover, since $r_k^i(\mathbf{s})\in[0,1]$ and $u_k$ is continuous on the compact interval $[0,1]$, there exists $M_u<\infty$ such that $|u_k(r)|\le M_u$ for all $r\in[0,1]$. Thus,
\[
|\mathcal{J}_k(\mathbf{s})|
\le
\left|u_k\!\left(\sum_i s_i r_k^i(\mathbf{s})\right)\right|
+
\sum_i s_i |u_k(r_k^i(\mathbf{s}))|
\le 2M_u.
\]
Let $M_J=2M_u$.

Now write
\[
\bar{\mathcal{J}}_k(\delta)
=
\int_{\Delta}\mathcal{J}_k(\mathbf{s})p(\mathbf{s};\delta)\,d\mathbf{s},
\qquad
p(\mathbf{s};\delta)
=
\frac{1}{B(\delta)}
\prod_{i=1}^{n_s}s_i^{\delta-1}.
\]

Let $\delta_n\to\delta\in(\delta_{\min},\delta_{\max})$. Since $\{\delta_n\}$ converges, there exist constants $0<a<b<\infty$ such that $\delta_n,\delta\in[a,b]$ for all sufficiently large $n$. For $\mathbf{s}$ in the interior of the simplex,
\[
p(\mathbf{s};\delta_n)\to p(\mathbf{s};\delta)
\]
pointwise. Boundary points have Lebesgue measure zero on the simplex, so this convergence holds almost everywhere.

Since $B(\cdot)$ is continuous and strictly positive on $[a,b]$, there exists
\[
M_B=\max_{\eta\in[a,b]}\frac{1}{B(\eta)}<\infty.
\]
Also, for $s_i\in(0,1)$ and $\delta_n\ge a$,
\[
s_i^{\delta_n-1}\le s_i^{a-1}.
\]
Therefore,
\[
|\mathcal{J}_k(\mathbf{s})p(\mathbf{s};\delta_n)|
\le
M_J M_B \prod_{i=1}^{n_s}s_i^{a-1}
\coloneqq G(\mathbf{s}).
\]
The function $G$ is integrable over $\Delta$ because $a>0$ and
\[
\int_{\Delta}\prod_{i=1}^{n_s}s_i^{a-1}\,d\mathbf{s}
=
B(a)
<
\infty.
\]

Hence, by the Dominated Convergence Theorem,
\[
\lim_{n\to\infty}\bar{\mathcal{J}}_k(\delta_n)
=
\int_{\Delta}\mathcal{J}_k(\mathbf{s})p(\mathbf{s};\delta)\,d\mathbf{s}
=
\bar{\mathcal{J}}_k(\delta).
\]
Thus, $\bar{\mathcal{J}}_k(\delta)$ is continuous on $(\delta_{\min},\delta_{\max})$.

\medskip
We now consider the extreme concentration regimes.

As $\delta \to 0$, only one of $s_i \approx 1$ and the remaining $s_{j} \approx 0, \forall j \ne i$. This amounts to using the single strategy, which is a deterministic case. In this case, we have the following. 

\begin{equation}
u_k(\sum_{i=1}^{n_s}s_ir_k^i) \approx u_k(r_k^i) 
\label{app:zjgap1}
\end{equation}

\begin{equation}
\sum_{i=1}^{n_s}s_iu_k(r_k^i) \approx u_k(r_k^i) 
\label{app:zjgap2}
\end{equation}

These two equations, Eq.~\ref{app:zjgap1} and Eq.~\ref{app:zjgap2}, imply the zero Jensen's gap.


As $\delta \to \infty$, all the strategies are $s_i \approx \frac{1}{n_s}$. From the proposition~\ref{app:prop_kkt} we know that the solution $r_k^i$ admits the following equality 

\begin{equation}
u_k(r_k^i) = \frac{s_i}{\sum_{j=1}^{n_s} s_j^2} ({\gamma_k - t_k})
\label{app:r_k_sol}
\end{equation}

Since all the $s_i$ admits to the same value the above equation simplifies to 
\begin{equation*}
u_k(r_k^i) =\gamma_k - t_k 
\end{equation*}
This is independent of $i$, Since $t_k$ is the slack, its adjusted so that the solution exists and satisfies the constraint $ \sum_{i=1}^{n_s} r_k^i = \frac{n_s}{N}$. Since they all are equal we have $r_k^i = \frac{1}{N}$.

We have $u_k(\sum_{i=1}^{n_s} s_i r_k^i) = u_k(\frac{1}{N})$ and $\sum_{i=1}^{n_s}s_iu_k(r_k^i) = u_k(\frac{1}{N})$. Resulting in the zero Jensen's gap.



Hence, the expected gap $\bar{\mathcal{J}}_k(\delta)$ vanishes in both extremes.
\end{proof}

\begin{proposition}
Suppose $u_k(\cdot)$ is concave, $\mathcal{J}_k(s)\ge 0$, and $t \leq \underset{\mathbf{s}\sim {Dir}(\delta)} {\mathbb{E}} [\mathcal{J}_k(\mathbf{s})]$ and $n_s \le N$ then $Pr[ \mathcal{J}_k(\mathbf{s}) \ge t ] \ge \frac{({\mathbb{E}}[\mathcal{J}_k(\mathbf{s})]-t)^2} {{M_J}^2} $, for some $M_J$ and  $M_J \le1$, if the utility is accuracy.

\label{app:pz}
\end{proposition}
\begin{proof}

By Paley-Zygmund inequality we have the following for any positive random variable $X$ and a constant $\beta$ we have the following

\begin{equation}
Pr[X \ge \beta \mathbb{E}[X]] \ge (1-\beta)^2 \frac{\mathbb{E}^2[X]}{\mathbb{E}[X^2]}
\label{pz_ineq}
\end{equation}

Since this is a Jensen's gap and its non negative and its a random variable, and by assumption, we have the 
$\frac{t}{\mathbb{E}[\mathcal{J}_k(\mathbf{s})]} \leq 1$

Now set $X = \mathcal{J}_k(\mathbf{s})$ and $\beta = \frac{t}{\mathbb{E}[\mathcal{J}_k(\mathbf{s})]}$ and substitute in the Eq.~\ref{pz_ineq}.

\begin{equation}
Pr[ \mathcal{J}_k(\mathbf{s}) \ge t ] \ge \frac{(\mathbb{E}[\mathcal{J}_k(\mathbf{s})]-t)^2} {\mathbb{E}[{\mathcal{J}_k^2(\mathbf{s})]}}
\label{app:des_ineq}
\end{equation}

We note that minimum value of ${{\mathcal{J}_k}(\mathbf{s})}^2$ is $0$ and from the previous proposition it can be seen that ${{\mathcal{J}_k}(\mathbf{s})} \le M_J$. The reason we claim $M_J \leq 1$ is since, if the utillity is accuracy the  the max value is $1$ and a lower quantity is  subtracted from the same.   

Plugging the above in the Eq.~\ref{app:des_ineq} gives the desired result.

\end{proof}

\subsection{Assumptions for Convergence analysis}
We state our assumptions as below; these  assumptions are commonly made in the FL 
literature such as~\cite{jhunjhunwala2022fedvarp,nguyen2022federated}   
\begin{assumption}
$\lVert \nabla f_k(\mathbf{x}) - \nabla f_k(\mathbf{y}) \rVert \leq L \lVert \mathbf{x} - \mathbf{y} \rVert$  
\label{assm:1}
\end{assumption}
\vspace{-0.1in}
This assumption implies that the loss function of the client $k$, i.e, $f_k(.)$, is Lipschitz Smooth.
\begin{assumption}
$\mathbb{E}_{\xi} {\lVert  \nabla f_k(\mathbf{w,\xi}) - \nabla{f_k( \boldsymbol{\theta})}  \rVert}^2 \leq \sigma ^2$ for all $k\in [N]$ and
$\mathbb{E}_{\xi}[\nabla f_k( \boldsymbol{\theta},\xi)] = \nabla f_k( \boldsymbol{\theta}) $
\label{assm:2}
\end{assumption}
\vspace{-0.1in}
This implies that the variance of the gradients is bounded and the gradients are unbiased.
\begin{assumption}
${\lVert \nabla{f_k( \boldsymbol{\theta})}  - \nabla{f( \boldsymbol{\theta})} \rVert}^2 \leq \sigma^2_g$. 
\label{assm:3}
\end{assumption}
\vspace{-0.1in}
This implies that the dissimilarity between local and global gradients is bounded, commonly referred to as the bounded gradient dissimilarity assumption. 
\begin{assumption}
$\lVert \nabla{f_k( \boldsymbol{\theta})}\rVert^2 \leq G$ for all $k \in [N] $. 
\label{assm:4}
\end{assumption}
\vspace{-0.1in}
This assumption implies that local gradients are bounded. This is the same as loss functions being Lipschitz continuous and differentiable~\cite{nguyen2022federated}.
\begin{assumption}
The staleness of the client model at the server is bounded by a maximum delay of $\nu_{max}$.
\label{assm:5}
\end{assumption}

\begin{assumption}
Let $S_t$ denote the number of clients participating in the round $t$. The we assume $1\le S \le |S_t| \le N$.
\end{assumption}

The above assumption implies that while the exact subset of participating clients dynamically changes based on the state-dependent mechanism, the number of active clients per round is strictly bounded below by $S$.

\subsection{Update rules for the FedUCA}

The global objective function is $ f ( \boldsymbol{\theta})=\sum_{k=1}^N \beta_kf_k( \boldsymbol{\theta})$. 


We further define the following two notions:

\begin{equation}
\text{Averaged stochastic Gradient:}\quad 
\mathbf{d}_k^t=\frac{1}{\tau} \sum_{r=0}^{\tau-1} \nabla f_k( \boldsymbol{\theta}_k^{(t,r)},\xi_k^{t,r}),
\end{equation}
\begin{equation}
\text{Averaged Gradient:} \quad
\mathbf{h}_k^t=\frac{1}{\tau} \sum_{r=0}^{\tau-1} \nabla f _k( \boldsymbol{\theta}_k^{(t,r)}).
\end{equation}

\begin{equation}
\text{Client Update:}\quad 
 \boldsymbol{\theta}_k^{t,\tau} =  \boldsymbol{\theta}^{t,0} - \eta_c\tau \mathbf{d}_k^t
\label{local_update_eq}
\end{equation}

Then, the update of the global model between two rounds is as follows:
\begin{align}
 \boldsymbol{\theta}^{t+1,0} 
&=  \boldsymbol{\theta}^{t,0} - \eta \tau \left[  \sum_{k\in S_t} \alpha^t_k\mathbf{d}_k^t  +  \sum_{k\in S_t^c} \alpha^t_k \mathbf{y}_k^t\right] 
\label{last_up_eq}
\end{align}

We further define the total variance $\sigma^2_{total}$ as 
\begin{equation}
\sigma^2_{total}\coloneqq\sigma^2+\sigma^2_g+G
\label{tvar}
\end{equation}

In the above Eq.~\ref{last_up_eq} $\mathbf{y}_k^t \coloneqq \mathbf{d}_k^{t-\nu_k^t}$ and $\eta \coloneqq \eta_s\eta_c$. In the $\nu_k^t$ is the stale update of the client $k$ relative to the current round $t$.

Our analysis utilizes the proof techniques of the FedBuff~\cite{nguyen2022federated}.

\subsection{Convergence Analysis}
\label{conv_analysis}

\begin{lemma}[Expected-Aggregation Discrepancy]
\label{lem:agg_discrepancy}
Under the constraints of Optimization~\ref{main_opt}, if the strategies $s$ are drawn from a symmetric Dirichlet distribution $Dir(\delta \mathbf{1}_{n_s})$, the expected aggregation discrepancy is strictly bounded by:
\begin{equation}
\mathbb{E}_s||\frac{1}{N}\mathbf{1}_{N} - \boldsymbol{\alpha}^t||^2 \le \frac{n_s - 1}{N(n_s \delta + 1)},
\end{equation}

Further $\mathcal{E}_{agg} \leq  C\frac{n_s-1}{S(n_s\delta+1)}$, where  $\mathcal{E}_{agg} = \frac{12 K N G}{S} \big( \frac{1}{T} \sum_{t=0}^{T-1} \mathbb{E} \| \boldsymbol{\beta} - \boldsymbol{\alpha}^t \|^2 \big)$.

\end{lemma}

\begin{proof}

From the constraints of Optimization~\ref{main_opt}, the strategy matrix $\mathbf{R}$ satisfies the column-sum constraint $\mathbf{1}_{n_s}^T \mathbf{R} = \frac{n_s}{N} \mathbf{1}_N^T$. Taking the transpose yields $\mathbf{R}^T \mathbf{1}_{n_s} = \frac{n_s}{N} \mathbf{1}_N$. 
Since the mean of a symmetric Dirichlet distribution is $\mathbb{E}[\mathbf{s}] = \frac{1}{n_s}\mathbf{1}_{n_s}$, mapping this expected strategy through {any} feasible matrix $\mathbf{R}$ deterministically recovers the uniform target:
\begin{equation}
\mathbf{R}^T \mathbb{E}[\mathbf{s}] = \mathbf{R}^T \left(\frac{1}{n_s}\mathbf{1}_{n_s}\right) = \frac{1}{n_s}\left(\frac{n_s}{N}\mathbf{1}_N\right) = \frac{1}{N}\mathbf{1}_N = \frac{1}{N}\mathbf{1}_N
\end{equation}
Consequently, the realized aggregation weights $\boldsymbol{\alpha}^t = \mathbf{R}^T s$ can be expressed as a zero-mean perturbation around the ideal uniform target:
\begin{equation}
\boldsymbol{\alpha}^t = \mathbf{R}^T \mathbf{s} = \mathbf{R}^T(\mathbb{E}[\mathbf{s}] + (\mathbf{s} - \mathbb{E}[\mathbf{s}])) = \frac{1}{N}\mathbf{1}_N + \mathbf{R}^T(\mathbf{s} - \mathbb{E}[\mathbf{s}])
\end{equation}
Thus, the discrepancy norm simplifies exactly to $||\frac{1}{N}\mathbf{1}_N - \boldsymbol{\alpha}^t||_2^2 = ||\mathbf{R}^T(\mathbf{s} - \mathbb{E}[\mathbf{s}])||_2^2$.

Applying the sub-multiplicative property of induced matrix norms, we bound the discrepancy:
\begin{equation}
||\mathbf{R}^T(\mathbf{s} - \mathbb{E}[\mathbf{s}])||_2^2 \le ||\mathbf{R}^T||_2^2 \cdot ||\mathbf{s} - \mathbb{E}[\mathbf{s}]||_2^2
\end{equation}
Because $\mathbf{R}$ contains strictly non-negative elements, we bound its spectral norm ($||\mathbf{R}^T||_2^2$) using the inequality $||\mathbf{A}||_2^2 \le ||\mathbf{A}||_1 ||\mathbf{A}||_\infty$:
\begin{itemize}
    \item $||\mathbf{R}^T||_1 = \max_i \sum_j \mathbf{R}_{i,j} = \max (\text{row sums of } \mathbf{R}) = 1$.
    \item $||\mathbf{R}^T||_\infty = \max_j \sum_i \mathbf{R}_{i,j} = \max (\text{col sums of } \mathbf{R}) = \frac{n_s}{N}$.
\end{itemize}
Therefore, for \textit{any} feasible optimization outcome, the spectral norm is deterministically bounded by $||\mathbf{R}^T||_2^2 \le \frac{n_s}{N}$.

Taking the expectation over the Dirichlet distribution $\mathbf{s}$:
\begin{equation}
\mathbb{E}_\mathbf{s}||\frac{1}{N}\mathbf{1}_N - \boldsymbol{\alpha}^t||_2^2 \le \frac{n_s}{N} \mathbb{E}_\mathbf{s}||\mathbf{s} - \mathbb{E}[\mathbf{s}]||_2^2
\end{equation}
The expectation term is exactly the trace of the Dirichlet covariance matrix $\Sigma_s$. For $\mathbf{s} \sim Dir(\delta \mathbf{1}_{n_s})$, the trace is known analytically:
\begin{equation}
\text{Tr}(\Sigma_s) = \frac{n_s - 1}{n_s(n_s \delta + 1)}
\end{equation}
Substituting this into the inequality yields the final bound:
\begin{equation}
\mathbb{E}_s||\frac{1}{N}\mathbf{1}_N - \boldsymbol{\alpha}^t||_2^2 \le \frac{n_s}{N} \left( \frac{n_s - 1}{n_s(n_s \delta + 1)} \right) = \frac{n_s - 1}{N(n_s \delta + 1)}
\end{equation}

Finally substituting this, in $\mathcal{E}_{agg}$ gives the bound  

$\mathcal{E}_{agg} \leq C\frac{n_s-1}{S(n_s\delta+1)}$, where $C = 12KG$.
\end{proof}

\begin{lemma}[Point-wise Aggregation Discrepancy Bound]
\label{lem:as_agg_discrepancy}
Under the constraints of Optimization.~\ref{main_opt} with $\beta_k = 1/N$, $n_s > 1$ the FedUCA aggregation weights satisfy
\begin{equation}
    \|\boldsymbol{\beta} - \boldsymbol{\alpha}^t\|^2 \le \frac{n_s - 1}{N} \quad \text{almost surely, for every } t.
\end{equation}
In particular, $\mathcal{E}_{\text{agg}} \le  O\left(\frac{n_s - 1}{S}\right)$, where $\mathcal{E}_{agg} \coloneqq \frac{12 K N G}{S} \big( \frac{1}{T} \sum_{t=0}^{T-1} \mathbb{E} \| \boldsymbol{\beta} - \boldsymbol{\alpha}^t \|^2 \big)$
\end{lemma}

\begin{proof}
For any $x\in[0,M]$ and any $\mu\in[0,M]$, we have
\[
(x-\mu)^2=x^2-2\mu x+\mu^2 .
\]
Since $0\le x\le M$, it follows that $x^2\le Mx$. Hence,
\[
(x-\mu)^2
\le Mx-2\mu x+\mu^2
=
(M-2\mu)x+\mu^2 .
\]
Taking $M=\frac{n_s}{N}$ and $\mu=\frac{1}{N}$, we obtain
\[
\left(\alpha_k-\frac{1}{N}\right)^2
\le
\left(\frac{n_s}{N}-\frac{2}{N}\right)\alpha_k
+
\frac{1}{N^2}
=
\frac{n_s-2}{N}\alpha_k+\frac{1}{N^2}.
\]
Summing over $k\in[N]$ and using $\sum_{k=1}^{N}\alpha_k=1$ gives
\[
\sum_{k=1}^{N}
\left(\alpha_k-\frac{1}{N}\right)^2
\le
\frac{n_s-2}{N}\sum_{k=1}^{N}\alpha_k
+
\frac{N}{N^2}
=
\frac{n_s-2}{N}+\frac{1}{N}
=
\frac{n_s-1}{N}.
\]
This implies that
\[
\mathcal{E}_{\mathrm{agg}}
\le
12KG\left(\frac{n_s-1}{S}\right).
\]
\end{proof}

\begin{lemma}
The gradients of the client's loss $\nabla{f_k( \boldsymbol{\theta}_k^{t,r};\xi_k^{t,r})}$ satisfy the bound 
\begin{align*}
\mathbb{E}{\Bigl\|  \nabla{f_k( \boldsymbol{\theta}_k^{t,r};\xi_k^{t,r})} \Bigr\|}^2 \leq 3\sigma^2_{total}
\end{align*}
\label{sec:local_var_lemma}
\end{lemma}
\begin{proof}
The  proof is provided in~\cite{nguyen2022federated}
\end{proof}

\begin{lemma}
The iterates generated by FedUCA $ \boldsymbol{\theta}^t$ satisfy the following\\
\begin{align*}
\mathbb{E}_{\xi_t}{\Bigl\|  \nabla{f_k( \boldsymbol{\theta}^t)- \mathbf{h}_k^t} \Bigr\|}^2
&\leq {3\eta_{c}^2L^2} \sigma^2_{total} \tau (\tau-1)
\end{align*}
\label{sec:lm2}
\end{lemma}
\begin{proof}

\begin{align*}
 \mathbb{E}_{\xi_t}\Bigl\|
\nabla f_k( \boldsymbol{\theta}^t) - \mathbf{h}_k^t
\Bigr\|^2
&=
\mathbb{E}_{\xi_t}\Bigl\|
\frac{1}{\tau} \sum_{r=0}^{\tau-1} \nabla f_k( \boldsymbol{\theta}^t)
-\nabla f_k( \boldsymbol{\theta}_k^{(t,r)})
\Bigr\|^2\\
& \leq 
\frac{1}{\tau}\sum_{r=0}^{\tau-1} \mathbb{E}_{\xi_t}\Bigl\| \nabla f_k( \boldsymbol{\theta}^t)
-\nabla f_k( \boldsymbol{\theta}_k^{(t,r)})
\Bigr\|^2 \\
& \overset{{a}}{\leq}  \frac{L^2}{\tau}\sum_{r=0}^{\tau-1} \mathbb{E}_{\xi_t}\Bigl\|   \boldsymbol{\theta}^t-  \boldsymbol{\theta}_k^{(t,r)}\Bigr\|^2\\
&=  \frac{L^2 \eta_c^2}{\tau}\sum_{r=0}^{\tau-1} \mathbb{E}_{\xi_t} \Bigl\| \sum_{m=0}^{r-1} \nabla f_k( \boldsymbol{\theta}_k^{t,m} ;\xi_k^{t,m})  \Bigr\|^2\\
& \overset{{b}}{\leq} \frac{L^2\eta_c^2}{\tau}\sum_{r=0}^{\tau-1} r  \sum_{m=0}^{r-1} \mathbb{E}_{\xi_t} \Bigl\| \nabla f_k( \boldsymbol{\theta}_k^{t,m} ;\xi_k^{t,m})  \Bigr\|^2\\
& \overset{{c}}{\leq} \frac{3L^2\eta_c^2}{\tau}\sum_{r=0}^{\tau-1} r^2 (\sigma^2 + \sigma^2_g + G)\\
&\leq 3L^2 \eta_c^2\sigma^2_{total} (\tau-1)\tau
\end{align*}

Here $a$ follows from the Assumption.~\ref{assm:1}. $b$ follows from the client update.~\ref{local_update_eq}, finally $c$ is via the Lemma.~\ref{sec:local_var_lemma}.
\end{proof}


Because the client's participation decision $I_k^t$ is evaluated deterministically based on the Optimization   parameters and rationality thresholds broadcasted in the history $H_t$, it is formally $H_t$-measurable. Consequently, $I_k^t$ behaves as a constant under the conditional expectation $\mathbb{E}[\cdot \mid H_t]$, rendering it independent of the specific   realization $\alpha_k^t$ for that round. This measurability allows us to treat the active cohort as a fixed set for the round $t$ analysis, cleanly isolating the stochasticity purely to the server's   draw.

\begin{lemma}
The iterates generated by FedUCA $ \boldsymbol{\theta}^t$ satisfy the following
\begin{align*}
   \mathbb{E} \Bigl\| \sum_{k=1}^{N} \alpha_k^t (I_k^t \mathbf{h}_k^t) \Bigr\|^2 
    &\leq  9L^2 \eta_c^2\sigma^2_{total} (\tau-1)\tau + 3 \sigma_g^2 + 3 \mathbb{E} \left[ \Bigl\| \nabla f( \boldsymbol{\theta}^{(t,0)}) \Bigr\|^2 \right]
\end{align*}
\label{sec:lm3}
\end{lemma}

\begin{proof}
Let $V_{act} = \sum_{k=1}^N \alpha_k^t (I_k^t \mathbf{h}_k^t)$. Because the broadcasted   weights reside on the probability simplex ($\alpha_k^t \ge 0$, $\sum_{k=1}^N \alpha_k^t = 1$), we exploit the convexity of the squared $L_2$-norm. Applying Jensen's Inequality pushes the norm inside the summation:
\small{
\begin{align}
    \mathbb{E} \|V_{act}\|^2 &\le \mathbb{E} \left[ \sum_{k=1}^N \alpha_k^t \Bigl\| I_k^t \mathbf{h}_k^t \Bigr\|^2 \right] \label{eq:jensens_push}
\end{align}
}

Since $I_k^t \in \{0,1\}$, we have $\|I_k^t \mathbf{h}_k^t\|^2 = I_k^t \|\mathbf{h}_k^t\|^2$. By using the relaxed triangle inequality, we decompose this into the following
\small{
\begin{align}
    \mathbb{E} \|V_{act}\|^2 
    &\leq 3 \underbrace{\mathbb{E} \left[ \sum_{k=1}^{N} \alpha_k^t I_k^t \Bigl\| \mathbf{h}_k^t - \nabla f_k( \boldsymbol{\theta}^{(t,0)}) \Bigr\|^2 \right]}_{F_1} + 3 \underbrace{\mathbb{E} \left[ \sum_{k=1}^{N} \alpha_k^t I_k^t \Bigl\| \nabla f_k( \boldsymbol{\theta}^{(t,0)}) - \nabla f( \boldsymbol{\theta}^{(t,0)}) \Bigr\|^2 \right]}_{F_2} \notag \\
    &\quad + 3 \underbrace{\mathbb{E} \left[ \sum_{k=1}^{N} \alpha_k^t I_k^t \Bigl\| \nabla f( \boldsymbol{\theta}^{(t,0)}) \Bigr\|^2 \right]}_{F_3} \label{eq_F_normalized_split}
\end{align}
}

{
\small{
\begin{align}
    F_1 &= \mathbb{E} \left[ \sum_{k=1}^{N} I_k^t \alpha_k^t \Bigl\| \mathbf{h}_k^t - \nabla f_k( \boldsymbol{\theta}^{(t,0)}) \Bigr\|^2 \right] \notag \\
     &= \mathbb{E}_{H_t} \Bigg[ \sum_{k=1}^{N} \mathbb{E} \big[ I_k^t \alpha_k^t \big]  \mathbb{E} \bigg[ \Bigl\| \mathbf{h}_k^t - \nabla f_k( \boldsymbol{\theta}^{(t,0)}) \Bigr\|^2 \mathrel{\Big|} H_t \bigg] \Bigg] \notag \\
    &\leq \Big( 3L^2 \eta_c^2\sigma^2_{total} (\tau-1)\tau \Big) \times \mathbb{E} \left[ \sum_{k=1}^{N} I_k^t \alpha_k^t \right] \leq 3L^2 \eta_c^2\sigma^2_{total} (\tau-1)\tau
    \label{eq_F1_norm}
\end{align}
}
}

The first equality follows via the conditional expectation. The second holds because the SGD path variance is independent of the   realization given $H_t$. The final inequality holds due to the simplex constraint, which guarantees $\sum I_k^t \alpha_k^t \le \sum \alpha_k^t = 1$.  

\begin{align}
    F_2 &= \mathbb{E} \left[ \sum_{k=1}^{N} I_k^t \alpha_k^t \Bigl\| \nabla f_k( \boldsymbol{\theta}^{(t,0)}) - \nabla f( \boldsymbol{\theta}^{(t,0)}) \Bigr\|^2 \right] \leq \sigma_g^2 \mathbb{E} \left[ \sum_{k=1}^{N} I_k^t \alpha_k^t \right]  \leq \sigma_g^2
    \label{eq_F2_norm}
\end{align}

\small{
\begin{align}
    F_3 &= \mathbb{E} \left[ \Bigl\| \nabla f( \boldsymbol{\theta}^{(t,0)}) \Bigr\|^2 \left( \sum_{k=1}^{N} I_k^t \alpha_k^t \right) \right] \leq \mathbb{E}\Bigl\| \nabla f( \boldsymbol{\theta}^{(t,0)}) \Bigr\|^2 \label{eq_F3_norm}
\end{align}
}
Combining $F_1,F_2, F_3$ we conclude the proof.
\end{proof}

\begin{lemma}
\begin{align}
\sum_{k=1}^N \mathbb{E} \Bigl\| \mathbf{h}_k^{t-\nu_k^t} \Bigr\|^2 \leq 9N L^2 \eta_c^2\sigma^2_{total} (\tau-1)\tau + 3N\sigma_g^2 + 3NG 
\end{align}
\label{sec:stale_h_norm}
\end{lemma}

\begin{proof}
\begin{align}
    \sum_{k=1}^N \mathbb{E} \Bigl\| \mathbf{h}_k^{t-\nu_k^t} \Bigr\|^2 &= \sum_{k=1}^N \mathbb{E} \Bigl\| \big(\mathbf{h}_k^{t-\nu_k^t} - \nabla f_k( \boldsymbol{\theta}^{(t-\nu_k^t,0)})\big) + \big(\nabla f_k( \boldsymbol{\theta}^{(t-\nu_k^t,0)}) - \nabla f( \boldsymbol{\theta}^{(t-\nu_k^t,0)})\big) \notag \\
    &\quad + \nabla f( \boldsymbol{\theta}^{(t-\nu_k^t,0)}) \Bigr\|^2 \notag \\
    &\leq 3 \sum_{k=1}^N \mathbb{E} \Bigl\| \mathbf{h}_k^{t-\nu_k^t} - \nabla f_k( \boldsymbol{\theta}^{(t-\nu_k^t,0)}) \Bigr\|^2  + 3 \sum_{k=1}^N \mathbb{E} \Bigl\| \nabla f_k( \boldsymbol{\theta}^{(t-\nu_k^t,0)}) - \nabla f( \boldsymbol{\theta}^{(t-\nu_k^t,0)}) \Bigr\|^2 \notag \\
    &\quad + 3 \sum_{k=1}^N \mathbb{E} \Bigl\| \nabla f( \boldsymbol{\theta}^{(t-\nu_k^t,0)}) \Bigr\|^2 \notag \\
    &\leq 3 \sum_{k=1}^N \Big( 3L^2 \eta_c^2\sigma^2_{total} (\tau-1)\tau \Big) + 3 \sum_{k=1}^N \Big( \sigma_g^2 \Big) + 3 \sum_{k=1}^N \Big( G \Big) \notag \\
    &= 9N L^2 \eta_c^2\sigma^2_{total} (\tau-1)\tau + 3N\sigma_g^2 + 3NG
\end{align}
\end{proof}

\begin{lemma}
Let the global objective gradient be defined as $\nabla f( \boldsymbol{\theta}^{(t,0)}) = \sum_{k=1}^N \beta_k \nabla f_k( \boldsymbol{\theta}^{(t,0)})$. Under the assumption that the expected squared gradient norm is bounded ($\mathbb{E} \|\nabla f_k( \boldsymbol{\theta})\|^2 \leq G$), the iterates generated by the FedUCA mechanism satisfy the following bound on the directional drift:
\begin{align}
    &\mathbb{E} \Bigl\| \sum_{k=1}^N I_k^t \alpha_k^t \left( \nabla f_k( \boldsymbol{\theta}^{(t,0)}) - \mathbf{h}_k^t \right) \Bigr\|^2 \leq 3 L^2 \eta_c^2 \sigma_{total}^2 \tau(\tau-1)
\end{align}
\label{sec:lm_drift}
\end{lemma}

\begin{proof}
Bounding (Local SGD Drift):
 Let $M$ be the left hand side. Because the broadcasted   weights reside on the probability simplex ($\alpha_k^t \geq 0$, $\sum_{k=1}^N \alpha_k^t = 1$) and $I_k^t \in \{0, 1\}$, it holds that $\sum_{k=1}^N I_k^t \alpha_k^t \leq 1$. 
 
 Exploiting the convexity of the squared $L_2$-norm, we apply Jensen's Inequality to push the norm inside the summation, completely avoiding the dimensional penalty of Cauchy-Schwarz:
\begin{align}
    M = \mathbb{E} \Bigg[ \Bigl\| \sum_{k=1}^N I_k^t \alpha_k^t \big( \nabla f_k( \boldsymbol{\theta}^{(t,0)}) - \mathbf{h}_k^t \big) \Bigr\|^2 \Bigg] \notag \leq \mathbb{E} \Bigg[ \sum_{k=1}^N I_k^t \alpha_k^t \Bigl\| \nabla f_k( \boldsymbol{\theta}^{(t,0)}) - \mathbf{h}_k^t \Bigr\|^2 \Bigg] \label{eq:convex_drift}
\end{align}

We evaluate the expectation by conditioning on the history filtration $H_t$. Because $I_k^t$ is $H_t$-measurable, and the local SGD drift variance is conditionally independent of the specific   realization given $H_t$:
\begin{align}
    M &\leq \mathbb{E}_{H_t} \Bigg[ \sum_{k=1}^N  I_k^t \alpha_k^t   \mathbb{E} \bigg[ \Bigl\| \nabla f_k( \boldsymbol{\theta}^{(t,0)}) - \mathbf{h}_k^t \Bigr\|^2 \mathrel{\Big|} H_t \bigg] \Bigg] \notag
\end{align}

Substituting the standard bounded local SGD drift $3L^2 \eta_c^2 \sigma_{total}^2 \tau(\tau-1)$:
\begin{align}
    M &\leq \Big( 3L^2 \eta_c^2 \sigma_{total}^2 \tau(\tau-1) \Big) \mathbb{E} \left[ \sum_{k=1}^N I_k^t \alpha_k^t \right] \leq 3L^2 \eta_c^2 \sigma_{total}^2 \tau(\tau-1)
\end{align}

The final inequality strictly holds because the sum of the filtered simplex weights is bounded by $1$ ($\sum_{k=1}^N I_k^t \alpha_k^t \leq \sum_{k=1}^N \alpha_k^t = 1$). This concludes the proof.
\end{proof}

\begin{lemma}
Let the global objective gradient be defined as $\nabla f( \boldsymbol{\theta}^{(t,0)}) = \sum_{k=1}^N \beta_k \nabla f_k( \boldsymbol{\theta}^{(t,0)})$. Under the assumption that the expected squared gradient norm is bounded ($\mathbb{E} \|\nabla f_k( \boldsymbol{\theta})\|^2 \leq G$) the drift introduced by the stale trajectory satisfies:
\begin{align*}
\mathbb{E} \Bigl\| \sum_{k=1}^N (1 - I_k^t) \alpha_k^t \left( \nabla f_k( \boldsymbol{\theta}^{(t,0)}) - \mathbf{h}_k^{t-\nu_k^t} \right) \Bigr\|^2 & \leq  (N - S) \left( 12 L^2 \eta^2 \tau^2 \nu_{max}^2 N \sigma_{total}^2 \right)\\&+ (N-S)\left(3 L^2 \eta_c^2 \sigma_{total}^2 \tau(\tau-1) \right)
\end{align*}
\label{sec:lm_stale_drift}
\end{lemma}

\begin{proof}
Let $D$ be the left hand side, representing the Stale Trajectory Drift.  We apply Cauchy-Schwarz strictly over the set of dropped clients $S_t^c = \{k : I_k^t = 0\}$. The physical size of this dropped set is $|S_t^c| = N - |S_t|$. Because the Optimization mechanism guarantees a minimum active set of $S \leq |S_t|$, the dropped set size is deterministically bounded from above by $N - S$:

\begin{align}
    D &= \mathbb{E} \Bigg[ \Bigl\| \sum_{k \in S_t^c} \alpha_k^t \left( \nabla f_k( \boldsymbol{\theta}^{(t,0)}) - \mathbf{h}_k^{t-\nu_k^t} \right) \Bigr\|^2 \Bigg] \notag \\
    &\leq \mathbb{E} \Bigg[ \underbrace{|S_t^c|}_{\leq N - S} \sum_{k=1}^N (1 - I_k^t) (\alpha_k^t)^2 \Bigl\| \nabla f_k( \boldsymbol{\theta}^{(t,0)}) - \mathbf{h}_k^{t-\nu_k^t} \Bigr\|^2 \Bigg] \notag \\
    &\leq (N - S) \mathbb{E}_{H_t} \Bigg[ \sum_{k=1}^N \big[ (\alpha_k^t)^2 \big] \notag  \cdot \underbrace{\big[ 1 - I_k^t  \big]}_{= 1 - I_k^t \leq 1} \notag \\
    &\quad \times \mathbb{E} \bigg[ \Bigl\| \nabla f_k( \boldsymbol{\theta}^{(t,0)}) - \mathbf{h}_k^{t-\nu_k^t} \Bigr\|^2 \mathrel{\Big|} H_t \bigg] \Bigg] \notag \\
    &\leq (N - S) \sum_{k=1}^N \mathbb{E} \left[ (\alpha_k^t)^2 \Bigl\| \nabla f_k( \boldsymbol{\theta}^{(t,0)}) - \mathbf{h}_k^{t-\nu_k^t} \Bigr\|^2 \right] \label{D2_cs}
\end{align}

Because the elements of the simplex strictly satisfy $\sum_{k=1}^N (\alpha_k^t)^2 \leq (\sum_{k=1}^N \alpha_k^t) = 1$, we cleanly decouple the maximum element:
\begin{align*}
    D &\leq (N - S) \mathbb{E} \left[ \max_k \Bigl\| \nabla f_k( \boldsymbol{\theta}^{(t,0)}) - \mathbf{h}_k^{t-\nu_k^t} \Bigr\|^2 \sum_{k=1}^N (\alpha_k^t)^2 \right] \\
    &\leq (N - S) \max_k \mathbb{E} \Bigl\| \nabla f_k( \boldsymbol{\theta}^{(t,0)}) - \mathbf{h}_k^{t-\nu_k^t} \Bigr\|^2
\end{align*}

To bound the inner norm, we expand $\mathbf{h}_k^{t-\nu_k^t}$ and split it into the global server delay drift ($C_1$) and local SGD drift ($C_2$):
\begin{align}
    \mathbb{E} \Bigl\| \nabla f_k( \boldsymbol{\theta}^{(t,0)}) - \mathbf{h}_k^{t-\nu_k^t} \Bigr\|^2 \leq \frac{2L^2}{\tau} \sum_{r=0}^{\tau-1} \bigg( \underbrace{\mathbb{E} \Bigl\|  \boldsymbol{\theta}^{(t,0)} -  \boldsymbol{\theta}^{(t-\nu_k^t,0)} \Bigr\|^2}_{C_1} + \underbrace{\mathbb{E} \Bigl\|  \boldsymbol{\theta}^{(t-\nu_k^t,0)} -  \boldsymbol{\theta}_k^{t-\nu_k^t,r} \Bigr\|^2}_{C_2} \bigg)
\end{align}

Bounding $C_1$ (Server Delay Drift):
We unroll the global trajectory over $\nu_k^t$ rounds, splitting the update into the active set $S_m$ and the stale set $S_m^c$. Applying Cauchy-Schwarz to each set yields their respective true sizes $|S_m|$ and $(N - |S_m|)$, which seamlessly collapse to exactly $N$:
\begin{align}
    C_1 &= \mathbb{E} \Bigg\| \sum_{m=t-\nu_k^t}^{t-1} \eta \tau \bigg( \sum_{j \in S_m} \alpha_j^m \mathbf{h}_j^m + \sum_{j \in S_m^c} \alpha_j^m \mathbf{h}_j^{m-\delta_j} \bigg) \Bigg\|^2 \notag \\
    &\leq 2 \eta^2 \tau^2 \nu_k^t \sum_{m=t-\nu_k^t}^{t-1} \bigg( |S_m| \sum_{j \in S_m} \mathbb{E}[(\alpha_j^m)^2] (3\sigma_{total}^2) \notag \\
    &\quad + (N - |S_m|) \sum_{j \in S_m^c} \mathbb{E}[(\alpha_j^m)^2] (3\sigma_{total}^2) \bigg) \notag \\
    &\leq 6 \eta^2 \tau^2 \nu_k^t \sum_{m=t-\nu_k^t}^{t-1} \bigg( N \sigma_{total}^2 \sum_{j=1}^N \mathbb{E}[(\alpha_j^m)^2] \bigg) \notag \\
    &\leq 6 \eta^2 \tau^2 \nu_{max}^2 N \sigma_{total}^2
\end{align}

Bounding $C_2$ (Local SGD Drift):
Using standard local update bounds:
\begin{align}
    C_2 &\leq 3 \eta_c^2 r^2 \sigma_{total}^2
\end{align}

Substituting $C_1$ and $C_2$:
\begin{align}
\mathbb{E} \Bigl\| \nabla f_k( \boldsymbol{\theta}^{(t,0)}) - \mathbf{h}_k^{t-\nu_k^t} \Bigr\|^2  &\leq \frac{2L^2}{\tau} \sum_{r=0}^{\tau-1} \Big( 6 \eta^2 \tau^2 \nu_{max}^2 N \sigma_{total}^2 + 3 \eta_c^2 r^2 \sigma_{total}^2 \Big) \notag \\
    &\leq 12 L^2 \eta^2 \tau^2 \nu_{max}^2 N \sigma_{total}^2 + 3 L^2 \eta_c^2 \sigma_{total}^2 \tau(\tau-1) \label{inner_stale_drift}
\end{align}


Substituting \eqref{inner_stale_drift} back into $D$ concludes the proof.
\end{proof}

\begin{lemma}
\begin{align}
\mathbb{E} \Bigl\| \mathbf{e}_{stale}^t \Bigr\|^2 \leq  (N-S) N  \frac{\sigma^2}{\tau} 
\end{align}
\label{lm:stale_err}
\end{lemma}

\begin{proof}
Let $S_t^c = \{k : I_k^t = 0\}$ denote the set of dropped clients in round $t$. The size of this dropped cohort is $|S_t^c| = N - |S_t|$. Because the Optimization mechanism guarantees a minimum active cohort of $|S_t| \geq S$, the dropped set size is strictly bounded from above by $|S_t^c| \leq N - S$. 

\begin{align}
    \mathbb{E} \Bigl\| \mathbf{e}_{stale}^t \Bigr\|^2 
    &= \mathbb{E} \Biggl[ \Bigl\| \sum_{k \in S_t^c} \alpha_k^t (\mathbf{d}_k^{t-\nu_k^t} - \mathbf{h}_k^{t-\nu_k^t}) \Bigr\|^2 \Biggr] \notag \\
    &\overset{(a)}{\leq} \mathbb{E} \Biggl[ \underbrace{|S_t^c|}_{\leq N-S} \sum_{k \in S_t^c} (\alpha_k^t)^2 \Bigl\| \mathbf{d}_k^{t-\nu_k^t} - \mathbf{h}_k^{t-\nu_k^t} \Bigr\|^2 \Biggr] \notag \\
    &\overset{(b)}{=} \mathbb{E}_{H_t} \Biggl[ |S_t^c| \sum_{k=1}^N (1 - I_k^t) \mathbb{E}_{\alpha_k^t \mid H_t} \left[ (\alpha_k^t)^2 \right]  \cdot \mathbb{E} \bigg[ \Bigl\| \mathbf{d}_k^{t-\nu_k^t} - \mathbf{h}_k^{t-\nu_k^t} \Bigr\|^2 \mathrel{\Big|} H_t \bigg] \Biggr] \notag \\
    &\overset{(c)}{\leq} (N-S) \sum_{k=1}^N \mathbb{E} \Biggl[ (1 - I_k^t) (\alpha_k^t)^2 \Bigl\| \mathbf{d}_k^{t-\nu_k^t} - \mathbf{h}_k^{t-\nu_k^t} \Bigr\|^2 \Biggr] \notag \\
    &\overset{(d)}{\leq} (N-S) \sum_{k=1}^N \mathbb{E} \Biggl[ (1) \cdot (1) \cdot \Bigl\| \mathbf{d}_k^{t-\nu_k^t} - \mathbf{h}_k^{t-\nu_k^t} \Bigr\|^2 \Biggr] \notag \\
    &\overset{(e)}{\leq} (N-S) \sum_{k=1}^N \frac{\sigma^2}{\tau} \notag \\
    &= (N-S) N \frac{\sigma^2}{\tau} \label{eq_stale_noise_bound}
\end{align}

(a) Cauchy-Schwarz Inequality over Dropped Set: We restrict the summation strictly to the dropped client set $S_t^c$ and apply Cauchy-Schwarz. This seamlessly extracts the maximum possible dropped cohort size $(N-S)$ as a coefficient without dividing by dropout probabilities.

(b) Filtration and Measurability: Because the dropout indicator $I_k^t$ is evaluated deterministically given the history $H_t$, it is $H_t$-measurable. Similarly, the past local SGD noise is already realized within the history filtration, allowing us to factor it conditionally.

(c) Tower Property and Set Re-expansion: We transition back to the unconditional expectation, re-expanding the summation to all $N$ clients by using the measurable indicator $(1 - I_k^t)$.

(d) Deterministic Upper Bounds: The indicator is strictly bounded by $(1 - I_k^t) \leq 1$. Furthermore, to decouple the current   weight $\alpha_k^t$ from the past gradient noise (which influenced the current model $ \boldsymbol{\theta}^{(t,0)}$), we apply the deterministic simplex bound $\alpha_k^t \leq 1$, meaning $(\alpha_k^t)^2 \leq 1$.

(e) Unconditional Local Variance: With the random variables decoupled, we apply the standard assumption for the variance of the local SGD updates: $\mathbb{E}\|\mathbf{d}_k^{t-\nu_k^t} - \mathbf{h}_k^{t-\nu_k^t}\|^2 \leq \sigma^2/\tau$. Summing this constant upper bound over all $N$ clients yields the final result.
\end{proof}

\begin{lemma}
Let $\bar{\mathbf{v}}^t = \sum_{k=1}^N \Big( I_k^t \alpha_k^t \mathbf{h}_k^t + (1 - I_k^t) \alpha_k^t \mathbf{h}_k^{t-\nu_k^t} \Big)$. The unified drift satisfies:
\begin{align}
 \mathbb{E} \Bigl\| \nabla f( \boldsymbol{\theta}^{(t,0)}) - \bar{\mathbf{v}}^t \Bigr\|^2 &\leq 3 \, \underbrace{\mathbb{E} \Bigl\| \sum_{k=1}^N \left( \beta_k - \alpha_k^t \right) \nabla f_k( \boldsymbol{\theta}^{(t,0)}) \Bigr\|^2}_{\text{Unified Bias}} + 9 L^2 \eta_c^2 \sigma_{total}^2 \tau(\tau-1) \notag \\
 &\quad + 3(N - S) \Big( 12 L^2 \eta^2 \tau^2 \nu_{max}^2 N \sigma_{total}^2 + 3 L^2 \eta_c^2 \sigma_{total}^2 \tau(\tau-1) \Big)
\end{align}
\label{sec:lm_unified_drift}
\end{lemma}

\begin{proof}
We expand the total combined update inside the directional drift norm, adding and subtracting the ideal global gradients evaluated at the current server model $\nabla f_k( \boldsymbol{\theta}^{(t,0)})$. Applying the relaxed triangle inequality ($\|a+b+c\|^2 \leq 3\|a\|^2 + 3\|b\|^2 + 3\|c\|^2$) separates the terms:

\begin{align}
 &\mathbb{E} \Bigl\| \nabla f( \boldsymbol{\theta}^{(t,0)}) - \sum_{k=1}^N I_k^t \alpha_k^t \mathbf{h}_k^t  - \sum_{k=1}^N (1 - I_k^t) \alpha_k^t \mathbf{h}_k^{t-\nu_k^t} \Bigr\|^2 \notag \\
 &\leq 3 \underbrace{\begin{aligned} &\mathbb{E} \Bigl\| \sum_{k=1}^N \beta_k \nabla f_k( \boldsymbol{\theta}^{(t,0)})- \sum_{k=1}^N \Big( I_k^t + (1 - I_k^t) \Big) \alpha_k^t \nabla f_k( \boldsymbol{\theta}^{(t,0)}) \Bigr\|^2 \end{aligned}}_{U_1 \text{ (Bias)}} \notag \\
 &\quad + 3 \underbrace{\mathbb{E} \Bigl\| \sum_{k=1}^N I_k^t \alpha_k^t \left( \nabla f_k( \boldsymbol{\theta}^{(t,0)}) - \mathbf{h}_k^t \right) \Bigr\|^2}_{U_2 \text{ (Active Drift)}} + 3 \underbrace{\mathbb{E} \Bigl\| \sum_{k=1}^N (1 - I_k^t) \alpha_k^t \left( \nabla f_k( \boldsymbol{\theta}^{(t,0)}) - \mathbf{h}_k^{t-\nu_k^t} \right) \Bigr\|^2}_{U_3 \text{ (Stale Drift)}}
\end{align}

For the Unified Bias ($U_1$), we observe that the indicator variables for the active and dropped sets share the exact same ideal gradient evaluation. We factor it out conditioning on the history $H_t$:
\begin{align}
 U_1 &= \mathbb{E}_{H_t} \Bigg[ \mathbb{E} \bigg[ \Bigl\| \sum_{k=1}^N \beta_k \nabla f_k( \boldsymbol{\theta}^{(t,0)})  - \sum_{k=1}^N \Big( I_k^t + (1 - I_k^t) \Big) \alpha_k^t \nabla f_k( \boldsymbol{\theta}^{(t,0)}) \Bigr\|^2 \mathrel{\Big|} H_t \bigg] \Bigg]
\end{align}

Because $I_k^t + (1 - I_k^t) = 1$ strictly, the participation indicators perfectly collapse. The sampling strategy reconstructs the full gradient space, leaving only the structural discrepancy between the objective weights $\beta_k$ and the server's   realization $\alpha_k^t$:
\begin{align}
 U_1 &= \mathbb{E} \Bigl\| \sum_{k=1}^N \left( \beta_k - \alpha_k^t \right) \nabla f_k( \boldsymbol{\theta}^{(t,0)}) \Bigr\|^2 \label{perfect_bias_collapse}
\end{align}

For $U_2$, we substitute the bounded active local SGD drift from Lemma \ref{sec:lm_drift}. Because the coefficient is multiplied by $3$, we obtain:
\begin{align}
 U_2 &\leq 3 \Big( 3 L^2 \eta_c^2 \sigma_{total}^2 \tau(\tau-1) \Big) = 9 L^2 \eta_c^2 \sigma_{total}^2 \tau(\tau-1)
\end{align}

For $U_3$, we substitute the bounded stale trajectory drift from Lemma \ref{sec:lm_stale_drift}, applying the multiplier:
\begin{align}
 U_3 &\leq 3(N - S) \Big( 12 L^2 \eta^2 \tau^2 \nu_{max}^2 N \sigma_{total}^2  + 3 L^2 \eta_c^2 \sigma_{total}^2 \tau(\tau-1) \Big)
\end{align}

Combining the optimized bounds for $U_1$, $U_2$, and $U_3$ completes the proof.
\end{proof}

\label{supp:proof}
\begin{proposition}
\label{thm:decomposed_convergence}
Let the sequence of global server models be denoted by $ \boldsymbol{\theta}^{(t,0)}$ for $t \in \{0, \dots, T-1\}$. Assume the global objective is bounded below by $f^*$, with an initial suboptimality gap of $D_0  = f( \boldsymbol{\theta}^{(0,0)}) - f^*$. Let the target objective weights be defined as the vector $\boldsymbol{\beta} = [\beta_1, \dots, \beta_N]^T$, and the dynamically assigned active   weights at round $t$ be defined as $\boldsymbol{\alpha}^t = [\alpha_1^t, \dots, {\alpha}_N^t]^T$.
Given the learning rate condition $-\frac{\eta\tau}{4}(1-12L\eta\tau) \leq -\eta\tau\frac{S}{8K}$ (for some $K>0$) and assuming locally bounded gradients ($\|\nabla f_k\|^2 \leq G$), the average expected gradient norm is bounded by five discrete error components:

\small{
\begin{align}
    \frac{1}{T} \sum_{t=0}^{T-1} \mathbb{E} \Bigl\| \nabla f( \boldsymbol{\theta}^{(t,0)}) \Bigr\|^2 &\leq \mathcal{E}_{init} + \mathcal{E}_{agg}  + \mathcal{E}_{ var} + \mathcal{E}_{drift} + \mathcal{E}_{sys} 
    \label{eq:decomposed_convergence}
\end{align}
}
\normalsize

here, $f(\boldsymbol{\theta}) = \sum_{k=1}^{N}\beta_kf_k(\boldsymbol{\theta})$, $\sum_{k=1}^{N} \beta_k = 1$ and  the discrete error components are defined as follows: the optimization rate is $\mathcal{E}_{init} = \frac{8K D_0}{\eta \tau S T}$, the structural Optimization bias is $\mathcal{E}_{agg} = \frac{12 K N G}{S} \big( \frac{1}{T} \sum_{t=0}^{T-1} \mathbb{E} \| \boldsymbol{\beta} - \boldsymbol{\alpha}^t \|^2 \big)$, the stochastic   variance is $\mathcal{E}_{var} = \frac{8 K L \eta \sigma^2}{S} \big( \frac{1}{T} \sum_{t=0}^{T-1} \mathbb{E} \| \boldsymbol{\alpha}^t \|^2 \big)$, the temporal drift penalty is $\mathcal{E}_{drift} = \frac{8K}{\eta \tau S} \Omega$, and the irreducible system error floor is $\mathcal{E}_{sys} = \frac{8K}{\eta \tau S} \big( \Psi + \eta (N-S) N \sigma^2 \big) + \frac{8K L \eta \sigma^2 N(N-S)}{S}$, with $\Omega$ representing the aggregate temporal drift penalties (local updates and server delays) and $\Psi$ representing the structural data heterogeneity.
\label{full_prop_app}
\end{proposition}

\begin{proof}

 The following inequality follows from smoothness assumptions.
\small{
\begin{align}
    &\mathbb{E} \left[ f ( \boldsymbol{\theta}^{(t+1,0)}) \right] - \mathbb{E} \left[ f ( \boldsymbol{\theta}^{(t,0)}) \right] \leq \mathbb{E} \left[ \left\langle \nabla f ( \boldsymbol{\theta}^{(t,0)}),  \boldsymbol{\theta}^{(t+1,0)} -  \boldsymbol{\theta}^{(t,0)} \right\rangle \right] + \frac{L}{2} \mathbb{E} \left[ \Bigl\|  \boldsymbol{\theta}^{(t+1,0)} -  \boldsymbol{\theta}^{(t,0)} \Bigr\|^2 \right] \notag \\
    &\leq -\tau \eta \underbrace{\mathbb{E} \left[ \left\langle \nabla f ( \boldsymbol{\theta}^{(t,0)}), \mathbf{v}^t \right\rangle \right]}_{T_{1}} \notag + L\tau^2\eta^2 \underbrace{\mathbb{E} \left[ \Bigl\| \sum_{k=1}^N I_k^t \alpha_k^t \mathbf{d}_k^t \Bigr\|^2 \right]}_{T_2 \text{ (Fresh Variance)}} \notag  + L\tau^2\eta^2 \underbrace{\mathbb{E} \left[ \Bigl\| \sum_{k=1}^N (1 - I_k^t) \alpha_k^t \mathbf{d}_k^{t-\nu_k^t} \Bigr\|^2 \right]}_{T_3 \text{ (Stale Variance)}} \label{main_eq_unified}
\end{align}
}
Where the total update is defined as $\mathbf{v}^t = \mathbf{e}_{fresh}^t + \mathbf{e}_{stale}^t +\bar{\mathbf{v}}^t$ here $\mathbf{e}_{fresh}^t = \sum_{k=1}^{N} I_k^t \alpha_k^t (\mathbf{d}_k^t - \mathbf{h}_k^t)$  and $ \mathbf{e}_{stale}^t =  \sum_{k = 1}^{N} (1-I_k^t) \alpha_k^t (\mathbf{d}_k^{t-\nu_k^t} - \mathbf{h}_k^{t-\nu_k^t})$. $\bar{\mathbf{v}} = \sum_{k =1}^{N}  I_k^t \alpha_k^t \mathbf{h}_k^t + \sum_{k = 1}^{N}  (1-I_k^t) \alpha_k^t \mathbf{h}_k^{t-\nu_k^t}$  

The expectation is taken by conditioning over the randomness of history $H_t$ up to time $t$ or the Filtration up to time $t$.T be specific we repeatedly use $\mathbb{E}[.] = \mathbb{E}_{H_t}\mathbb{E} [.|H_t]$


We now bound $-\eta \tau T_1$
\small{
\begin{align}
    T_1 &= \mathbb{E} \left[ \left\langle \nabla f ( \boldsymbol{\theta}^{(t,0)}), \bar{\mathbf{v}}^t + \mathbf{e}_{fresh}^t + \mathbf{e}_{stale}^t \right\rangle \right] = \mathbb{E} \left[ \left\langle \nabla f ( \boldsymbol{\theta}^{(t,0)}), \bar{\mathbf{v}}^t \right\rangle \right] + \underbrace{\mathbb{E} \left[ \left\langle \nabla f ( \boldsymbol{\theta}^{(t,0)}), \mathbf{e}_{fresh}^t \right\rangle \right]}_{= \mathbf{0} \text{ (Tower Rule)}} \notag  \\ 
    &+ \mathbb{E} \left[ \left\langle \nabla f ( \boldsymbol{\theta}^{(t,0)}), \mathbf{e}_{stale}^t \right\rangle \right] \notag \\
    &= \frac{1}{2} \mathbb{E} \Bigl\| \nabla f( \boldsymbol{\theta}^{(t,0)}) \Bigr\|^2 + \frac{1}{2} \mathbb{E} \Bigl\| \bar{\mathbf{v}}^t \Bigr\|^2 - \frac{1}{2} \mathbb{E} \Bigl\| \nabla f( \boldsymbol{\theta}^{(t,0)}) - \bar{\mathbf{v}}^t \Bigr\|^2 + \mathbb{E} \left[ \left\langle \nabla f ( \boldsymbol{\theta}^{(t,0)}), \mathbf{e}_{stale}^t \right\rangle \right] 
\end{align}

}

By Young's inequality we have the following 

{\small
\begin{align}
    \mathbb{E} \left[ \left\langle \nabla f ( \boldsymbol{\theta}^{(t,0)}), \mathbf{e}_{stale}^t \right\rangle \right] \geq -\frac{1}{4} \mathbb{E} \Bigl\| \nabla f( \boldsymbol{\theta}^{(t,0)}) \Bigr\|^2 - \mathbb{E} \Bigl\| \mathbf{e}_{stale}^t \Bigr\|^2 \notag
\end{align}

\begin{align}
    -\eta \tau T_1 &\leq -\frac{\eta \tau}{4} \mathbb{E} \Bigl\| \nabla f( \boldsymbol{\theta}^{(t,0)}) \Bigr\|^2  + \frac{\eta \tau}{2} \mathbb{E} \Bigl\| \nabla f( \boldsymbol{\theta}^{(t,0)}) - \bar{\mathbf{v}}^t \Bigr\|^2  \notag\\
    &+ \eta \tau \mathbb{E} \Bigl\| \mathbf{e}_{stale}^t \Bigr\|^2 \label{T1_final_upper_bound_corrected}
\end{align}

We bound $T_2$ as below

\begin{align}
    T_2 &= \mathbb{E} \left[ \Bigl\| \sum_{k=1}^N I_k^t \alpha_k^t \mathbf{d}_k^t \Bigr\|^2 \right] \overset{(a)}{=} \mathbb{E} \left[ \Bigl\| \sum_{k=1}^N I_k^t \alpha_k^t \mathbf{h}_k^t \Bigr\|^2 \right] + \sum_{k=1}^N \mathbb{E} \left[ (I_k^t \alpha_k^t)^2 \Bigl\| \mathbf{d}_k^t - \mathbf{h}_k^t \Bigr\|^2 \right] \notag \\
    &\overset{(b)}{\leq} \mathbb{E} \left[ \Bigl\| \sum_{k=1}^N I_k^t \alpha_k^t \mathbf{h}_k^t \Bigr\|^2 \right] + \frac{\sigma^2}{\tau} \sum_{k=1}^N \mathbb{E} \left[ I_k^t (\alpha_k^t)^2 \right] \overset{(c)}{=} \mathbb{E} \left[ \Bigl\| \sum_{k=1}^N I_k^t \alpha_k^t \mathbf{h}_k^t \Bigr\|^2 \right] + \frac{\sigma^2}{\tau} \sum_{k=1}^N \mathbb{E}_{H_t} \left[ \mathbb{E}[I_k^t \mid H_t] (\alpha_k^t)^2 \right] \notag \\
    &\leq  \mathbb{E} \left[ \Bigl\| \sum_{k=1}^N I_k^t \alpha_k^t \mathbf{h}_k^t \Bigr\|^2 \right] + \frac{\sigma^2}{\tau} \sum_{k=1}^N \mathbb{E} \left[ (\alpha_k^t)^2 \right]\label{final_T2_simplified}
\end{align}
}
\normalsize

(a) Exact Variance Decomposition: The fresh local sampling noise $\mathbf{e}_k^t = \mathbf{d}_k^t - \mathbf{h}_k^t$ is conditionally zero-mean given $H_t$ and independent across clients, making all cross-terms evaluate strictly to zero. We can therefore separate the expected trajectory from the sampling variance.\\
(b) Unconditional Noise Bound: We substitute the standard local variance bound $\mathbb{E}[\|\mathbf{d}_k^t - \mathbf{h}_k^t\|^2] \leq \sigma^2/\tau$. Because $(I_k^t)^2$ is a binary indicator, it simplifies exactly to $I_k^t$. \\
(c) Tower Rule Filtration: We apply the Tower Rule ($\mathbb{E}[\cdot] = \mathbb{E}_{H_t}[\mathbb{E}[\cdot \mid H_t]]$) specifically to the noise sum. The algorithm's   weight generation for the current round ($\alpha_k^t$) is entirely $H_t$-measurable, so it factors outside the inner conditional expectation.


We now bound $T_3$ as below:
\begin{align}
    T_3 &= \mathbb{E} \left[ \Bigl\| \sum_{k=1}^N (1 - I_k^t) \alpha_k^t \mathbf{d}_k^{t-\nu_k^t} \Bigr\|^2 \right] = \mathbb{E} \left[ \Bigl\| \sum_{k \in S_t^c} \alpha_k^t \mathbf{d}_k^{t-\nu_k^t} \Bigr\|^2 \right] \notag \\
    &\overset{(a)}{\leq} \mathbb{E} \left[ \underbrace{|S_t^c|}_{\leq N-S} \sum_{k \in S_t^c} (\alpha_k^t)^2 \Bigl\| \mathbf{d}_k^{t-\nu_k^t} \Bigr\|^2 \right] \notag \\
    &\overset{(b)}{=} \mathbb{E}_{H_t} \left[ |S_t^c| \sum_{k=1}^N (1 - I_k^t) \mathbb{E}_{\alpha_k^t \mid H_t} \left[ (\alpha_k^t)^2 \right] \mathbb{E} \bigg[ \Bigl\| \mathbf{d}_k^{t-\nu_k^t} \Bigr\|^2 \mathrel{\Big|} H_t \bigg] \right] \notag \\
    &\overset{(c)}{\leq} (N-S) \sum_{k=1}^N \mathbb{E} \left[ (1 - I_k^t) (\alpha_k^t)^2 \Bigl\| \mathbf{d}_k^{t-\nu_k^t} \Bigr\|^2 \right] \notag \\
    &\overset{(d)}{\leq} (N-S) \sum_{k=1}^N \mathbb{E} \left[ (1) \cdot (1) \cdot \Bigl\| \mathbf{d}_k^{t-\nu_k^t} \Bigr\|^2 \right] \notag \\
    &\overset{(e)}{\leq} (N-S) \sum_{k=1}^N \left( \frac{\sigma^2}{\tau} + \mathbb{E} \Bigl\| \mathbf{h}_k^{t-\nu_k^t} \Bigr\|^2 \right) \notag \\
    &= (N-S) N  \frac{\sigma^2}{\tau} + (N-S)  \sum_{k=1}^N \mathbb{E} \Bigl\| \mathbf{h}_k^{t-\nu_k^t} \Bigr\|^2 \label{final_T3_perfected}
\end{align}

(a) Cauchy-Schwarz Inequality over Dropped Set: We restrict the summation strictly to the dropped client set $S_t^c = \{k : I_k^t = 0\}$ and apply Cauchy-Schwarz. Because the Optimization mechanism guarantees a minimum active cohort of $S$, the dropped set size is deterministically bounded from above by $|S_t^c| \leq N - S$. This cleanly extracts the coefficient without requiring division by dropout probabilities.

(b) Filtration and Measurability: We condition on the history $H_t$. The dropout indicator $I_k^t$ and the active set size $|S_t^c|$ are evaluated deterministically given $H_t$, making them $H_t$-measurable. Furthermore, the past stale update $\mathbf{d}_k^{t-\nu_k^t}$ is fully realized within the history filtration. This allows us to cleanly decouple the expectation of the current   weights $\alpha_k^t$.

(c) Tower Property and Set Re-expansion: We transition back to the unconditional expectation, re-expanding the summation to all $N$ clients by multiplying by the indicator $(1 - I_k^t)$, and applying the structural upper bound $(N-S)$.

(d) Deterministic Upper Bounds: The indicator is strictly bounded by $(1 - I_k^t) \leq 1$. To safely eliminate the   weights, we apply the deterministic simplex bound $\alpha_k^t \leq 1$, which strictly implies $(\alpha_k^t)^2 \leq 1$. 

(e) Local SGD Second Moment: We apply the standard variance decomposition for local SGD updates ($\mathbb{E}\|\mathbf{x}\|^2 = \text{Var}(\mathbf{x}) + \|\mathbb{E}[\mathbf{x}]\|^2$). The variance of the local update is bounded by $\sigma^2/\tau$, and its expected drift is exactly $\mathbf{h}_k^{t-\nu_k^t}$. Summing these components yields the final bound.

By substituting $T_1, T_2, T_3$ we get the following 

\small{
\begin{align}
    \mathbb{E} \left[ f ( \boldsymbol{\theta}^{(t+1,0)}) \right] - \mathbb{E} \left[ f ( \boldsymbol{\theta}^{(t,0)}) \right] &\leq {-\frac{\eta \tau}{4} \mathbb{E} \Bigl\| \nabla f( \boldsymbol{\theta}^{(t,0)}) \Bigr\|^2} + {\frac{\eta \tau}{2} \mathbb{E} \Bigl\| \nabla f( \boldsymbol{\theta}^{(t,0)}) - \bar{\mathbf{v}}^t \Bigr\|^2}+ {\eta \tau \mathbb{E} \Bigl\| \mathbf{e}_{stale}^t \Bigr\|^2} \notag \\
    &\quad + {L\tau^2\eta^2 \mathbb{E} \left[ \Bigl\| \sum_{k=1}^N I_k^t \alpha_k^t \mathbf{h}_k^t \Bigr\|^2 \right] + L\tau^2\eta^2 (N-S) \sum_{k=1}^N \mathbb{E} \Bigl\| \mathbf{h}_k^{t-\nu_k^t} \Bigr\|^2} \notag \\
    &\quad + {L\tau\eta^2\sigma^2 \left( \sum_{k=1}^N \mathbb{E} \left[ (\alpha_k^t)^2 \right] + N(N-S) \right)} \label{main_eq_simplified}
\end{align}
}
\normalsize
Substituting the bounded components from Lemma \ref{lm:stale_err}, Lemma \ref{sec:lm_unified_drift}, Lemma \ref{sec:lm3}, and Lemma \ref{sec:stale_h_norm} into the master descent inequality \eqref{main_eq_simplified}, we collect and group the terms by their functional source. 

Extracting the global gradient norms yields the primary descent condition. The remaining terms strictly bound the variance injected by the Optimization mechanism, local client drift, and stale updates:

\small{
\begin{align}
\mathbb{E} \left[ f ( \boldsymbol{\theta}^{(t+1,0)}) \right] - \mathbb{E} \left[ f ( \boldsymbol{\theta}^{(t,0)}) \right] &\leq \underbrace{-\eta \tau \left( \frac{1}{4} - 3L\eta\tau \right) \mathbb{E} \Bigl\| \nabla f( \boldsymbol{\theta}^{(t,0)}) \Bigr\|^2}_{\text{Descent Condition}} \notag \\
    &\quad + \underbrace{\frac{3\eta \tau}{2} \mathbb{E} \Bigl\| \sum_{k=1}^N \left( \beta_k - \alpha_k^t \right) \nabla f_k( \boldsymbol{\theta}^{(t,0)}) \Bigr\|^2}_{\text{Optimization Structural Bias}} \notag \\
    &\quad + \underbrace{\frac{9\eta \tau}{2} L^2 \eta_c^2 \sigma_{total}^2 \tau(\tau-1) \Big( 1 + (N-S) \Big)}_{\text{Client Drift (Unified)}} \notag \\
    &\quad + \underbrace{9 L^3 \tau^3 \eta^2 \eta_c^2 \sigma_{total}^2 (\tau-1) \Big( 1 + N(N-S) \Big)}_{\text{Client Drift (Aggregate Norms)}} \notag \\
    &\quad + \underbrace{18 L^2 \eta^3 \tau^3 \nu_{max}^2 N(N-S) \sigma_{total}^2}_{\text{Server Delay Penalty}} \notag \\
    &\quad + \underbrace{3 L\tau^2\eta^2 \Big( \sigma_g^2 \big[ 1 + N(N-S) \big] + N(N-S) G \Big)}_{\text{Heterogeneity Bound}} \notag \\
    &\quad + \underbrace{\eta (N-S) N \sigma^2 + L\tau\eta^2\sigma^2 \left( \sum_{k=1}^N \mathbb{E} \big[ (\alpha_k^t)^2 \big] + N(N-S) \right)}_{\text{Stochastic Noise Floor}} \label{eq:fully_expanded_descent}
\end{align}
}
\normalsize
By using the learning rate condition $-\frac{\eta\tau}{4}(1-12L\eta\tau) \leq -\eta\tau\frac{S}{8K}$ to bound the descent step, and grouping the variance components into stochastic noise ($\Phi_{noise}^t$), system heterogeneity ($\Phi_{het}$), and stale delay penalties ($\Phi_{stale}$), we obtain the following single-round descent bound:


\small{
\begin{align}
\mathbb{E} \left[ f ( \boldsymbol{\theta}^{(t+1,0)}) \right] - \mathbb{E} \left[ f ( \boldsymbol{\theta}^{(t,0)}) \right] &\leq -\eta \tau \frac{S}{8K} \mathbb{E} \Bigl\| \nabla f( \boldsymbol{\theta}^{(t,0)}) \Bigr\|^2  + \frac{3\eta \tau N G}{2} \sum_{k=1}^N \mathbb{E} \left[ (\beta_k - \alpha_k^t)^2 \right] \notag  \\ 
&+ \Phi^t + \Psi + \Omega 
\label{eq:dc_single_stepf}
\end{align}
}
\normalsize

where the system variance components, stochastic noise ($\Phi^t$), system heterogeneity ($\Psi$), and stale delay drift ($\Omega$), are defined as:

\small{
\begin{align}
\Phi^t &= \eta (N-S) N \sigma^2  + L\tau\eta^2\sigma^2 \left( \sum_{k=1}^N \mathbb{E} \big[ (\alpha_k^t)^2 \big] + N(N-S) \right) \notag \\
\Psi &= 3 L\tau^2\eta^2 \Big( \sigma_g^2 \big[ 1 + N(N-S) \big] + N(N-S) G \Big) \notag \\
\Omega &= \frac{9\eta \tau}{2} L^2 \eta_c^2 \sigma_{total}^2 \tau(\tau-1) \Big( 1 + (N-S) \Big) \notag \\
&\quad + 9 L^3 \tau^3 \eta^2 \eta_c^2 \sigma_{total}^2 (\tau-1) \Big( 1 + N(N-S) \Big) \notag \\
&\quad + 18 L^2 \eta^3 \tau^3 \nu_{max}^2 N(N-S) \sigma_{total}^2 
\end{align}
}
\normalsize

Assuming the global objective is bounded below by $f^*$, we define the initial suboptimality gap as $D_0 = f( \boldsymbol{\theta}^{(0,0)}) - f^*$. Telescoping Eq.~\ref{eq:dc_single_stepf} over $T$ rounds and dividing by $T \cdot \eta \tau \frac{S}{8K}$, we obtain the final convergence bound:

\small{
\begin{align}
\frac{1}{T} \sum_{t=0}^{T-1} \mathbb{E} \Bigl\| \nabla f( \boldsymbol{\theta}^{(t,0)}) \Bigr\|^2 &\leq \frac{8K D_0}{\eta \tau S T} + \frac{12 K N G}{S T} \sum_{t=0}^{T-1} \sum_{k=1}^N \mathbb{E} \left[ (\beta_k - \alpha_k^t)^2 \right] \notag \\
    &+ \frac{8K}{\eta \tau S T} \sum_{t=0}^{T-1} \Big( \Phi^t + \Psi + \Omega \Big) \label{eq:dc_final_convergence}
\end{align}
}
\normalsize

By defining the probability weight vectors as $\boldsymbol{\beta} = [\beta_1, \dots, \beta_N]^T$ and $\boldsymbol{\alpha}^t = [\alpha_1^t, \dots, \alpha_N^t]^T$, we can completely decompose the convergence bound of FedUCA into five interpretable driving forces:

\small{
\begin{align}
\frac{1}{T} \sum_{t=0}^{T-1} \mathbb{E} \Bigl\| \nabla f( \boldsymbol{\theta}^{(t,0)}) \Bigr\|^2 &\leq \mathcal{E}_{\text{init}} + \mathcal{E}_{\text{agg}} + \mathcal{E}_{\text{var}} + \mathcal{E}_{\text{drift}} + \mathcal{E}_{\text{sys}} \label{eq:dc_decomposed_convergence}
\end{align}
}
\normalsize

where the discrete error components are defined as follows: the optimization rate is $\mathcal{E}_{init} = \frac{8K D_0}{\eta \tau S T}$, the structural Optimization bias is $\mathcal{E}_{agg} = \frac{12 K N G}{S} \big( \frac{1}{T} \sum_{t=0}^{T-1} \mathbb{E} \| \boldsymbol{\beta} - \boldsymbol{\alpha}^t \|^2 \big)$, the stochastic   variance is $\mathcal{E}_{ var} = \frac{8 K L \eta \sigma^2}{S} \big( \frac{1}{T} \sum_{t=0}^{T-1} \mathbb{E} \| \boldsymbol{\alpha}^t \|^2 \big)$, the temporal drift penalty is $\mathcal{E}_{drift} = \frac{8K}{\eta \tau S} \Omega$, and the irreducible system error floor is $\mathcal{E}_{sys} = \frac{8K}{\eta \tau S} \big( \Psi + \eta (N-S) N \sigma^2 \big) + \frac{8K L \eta \sigma^2 N(N-S)}{S}$.
\end{proof}


\section*{NeurIPS Paper Checklist}

\begin{enumerate}

\item {\bf Claims}
    \item[] Question: Do the main claims made in the abstract and introduction accurately reflect the paper's contributions and scope?
    \item[] Answer: \answerYes{}
    \item[] Justification: All the claims made in the abstract and introduction are supported by experiments and theoretical analysis. 

\item {\bf Limitations}
    \item[] Question: Does the paper discuss the limitations of the work performed by the authors?
    \item[] Answer: \answerYes{} 
    \item[] Justification: We have clearly discussed the limitations in the paper.    

\item {\bf Theory assumptions and proofs}
    \item[] Question: For each theoretical result, does the paper provide the full set of assumptions and a complete (and correct) proof?
    \item[] Answer: \answerYes{} 
    \item[] Justification: All the assumptions required for the proofs are clearly stated, and the known lemmas are clearly referenced.

    \item {\bf Experimental result reproducibility}
    \item[] Question: Does the paper fully disclose all the information needed to reproduce the main experimental results of the paper to the extent that it affects the main claims and/or conclusions of the paper (regardless of whether the code and data are provided or not)?
    \item[] Answer: \answerYes{} 
    \item[] Justification: All the experiment details are clearly documented for reproducibility of the results presented in the main paper.

\item {\bf Open access to data and code}
    \item[] Question: Does the paper provide open access to the data and code, with sufficient instructions to faithfully reproduce the main experimental results, as described in supplemental material?
    \item[] Answer: \answerYes{}
    \item[] Justification: All the experimental details are clearly documented for reproducibility of the results presented in the main paper. The code is released.

\item {\bf Experimental setting/details}
    \item[] Question: Does the paper specify all the training and test details (e.g., data splits, hyperparameters, how they were chosen, type of optimizer) necessary to understand the results?
    \item[] Answer: \answerYes{} 
    \item[] Justification: All the training/hyperparameter details have been included to understand the results.

\item {\bf Experiment statistical significance}
    \item[] Question: Does the paper report error bars suitably and correctly defined or other appropriate information about the statistical significance of the experiments?
    \item[] Answer: \answerYes{} 
    \item[] Justification: We have included the one sigma error bars.

\item {\bf Experiments compute resources}
    \item[] Question: For each experiment, does the paper provide sufficient information on the computer resources (type of compute workers, memory, time of execution) needed to reproduce the experiments?
    \item[] Answer: \answerYes{} 
    \item[] Justification: We have provided this information in the supplementary.
    
\item {\bf Code of ethics}
    \item[] Question: Does the research conducted in the paper conform, in every respect, with the NeurIPS Code of Ethics \url{https://neurips.cc/public/EthicsGuidelines}?
    \item[] Answer: \answerYes{} 
    \item[] Justification: We have read and followed the mentioned ethics guidelines.

\item {\bf Broader impacts}
    \item[] Question: Does the paper discuss both potential positive societal impacts and negative societal impacts of the work performed?
    \item[] Answer: \answerNA{} 
    \item[] Justification: Our work addresses FL with rational clients. It does not have any negative societal impact.
    
\item {\bf Safeguards}
    \item[] Question: Does the paper describe safeguards that have been put in place for responsible release of data or models that have a high risk for misuse (e.g., pre-trained language models, image generators, or scraped datasets)?
    \item[] Answer: \answerNA{} 
    \item[] Justification: \answerNA{}

\item {\bf Licenses for existing assets}
    \item[] Question: Are the creators or original owners of assets (e.g., code, data, models), used in the paper, properly credited and are the license and terms of use explicitly mentioned and properly respected?
    \item[] Answer: \answerYes{} 
    \item[] Justification:  All the resources used to develop the work are mentioned and acknowledged.

\item {\bf New assets}
    \item[] Question: Are new assets introduced in the paper well documented and is the documentation provided alongside the assets?
    \item[] Answer: \answerYes{} 
    \item[] Justification: The code is released.

\item {\bf Crowdsourcing and research with human subjects}
    \item[] Question: For crowdsourcing experiments and research with human subjects, does the paper include the full text of instructions given to participants and screenshots, if applicable, as well as details about compensation (if any)? 
    \item[] Answer: \answerNA{} 
    \item[] Justification: \answerNA{}

\item {\bf Institutional review board (IRB) approvals or equivalent for research with human subjects}
    \item[] Question: Does the paper describe potential risks incurred by study participants, whether such risks were disclosed to the subjects, and whether Institutional Review Board (IRB) approvals (or an equivalent approval/review based on the requirements of your country or institution) were obtained?
    \item[] Answer: \answerNA{} 
    \item[] Justification: \answerNA{}

\item {\bf Declaration of LLM usage}
    \item[] Question: Does the paper describe the usage of LLMs if it is an important, original, or non-standard component of the core methods in this research? Note that if the LLM is used only for writing, editing, or formatting purposes and does \emph{not} impact the core methodology, scientific rigor, or originality of the research, declaration is not required.
    \item[] Answer: \answerNA{} 
    \item[] Justification: \answerNA{}

\end{enumerate}

\end{document}